\documentclass{article}
\pdfoutput=1
\PassOptionsToPackage{numbers, compress}{natbib}
\usepackage[final]{neurips_2021}

\usepackage[utf8]{inputenc} % allow utf-8 input
\usepackage[T1]{fontenc}    % use 8-bit T1 fonts
\usepackage[colorlinks,
            linkcolor=red,
            anchorcolor=blue,
            citecolor=blue
            ]{hyperref}
\usepackage{url}            % simple URL typesetting
\usepackage{booktabs}       % professional-quality tables
\usepackage{amsfonts}       % blackboard math symbols
\usepackage{nicefrac}       % compact symbols for 1/2, etc.
\usepackage{microtype}      % microtypography
\usepackage[dvipsnames]{xcolor}         % colors
\usepackage{smile}

\usepackage{graphicx}
\usepackage{caption}
\newcommand{\qvalue}{Q}
\newcommand{\vvalue}{V}

\newcommand{\reward}{r}

\usepackage{enumitem}

\newcommand{\alglinelabel}{%
  \addtocounter{ALC@line}{-1}% Reduce line counter by 1
  \refstepcounter{ALC@line}% Increment line counter with reference capability
  \label% Regular \label
}

\usepackage{thmtools}
\usepackage{mathtools}
\newcommand{\hatw}{\hat{\mathbf{w}}}
\newcommand{\hatQ}{\hat{Q}}
\newcommand{\hatV}{\hat{V}}
\newcommand{\hatbeta}{\hat{\bbeta}}
\newcommand{\hattheta}{\hat{\btheta}}
\newcommand{\hatsigma}{\hat{\sigma}}

\newcommand{\diff}{\textnormal{d}}
\newcommand{\sigmanoise}{\sigma_{r}}

\title{Variance-Aware Off-Policy Evaluation with Linear Function Approximation}

\author{
  Yifei Min\thanks{Equal contribution.} \\
  Department of Statistics and Data Science\\
  Yale University\\
  CT 06511 \\ 
  \texttt{yifei.min@yale.edu} \\
   \And
   Tianhao Wang${}^*$ \\
   Department of Statistics and Data Science \\
   Yale University\\
   CT 06511 \\ 
   \texttt{tianhao.wang@yale.edu} \\
   \And
   Dongruo Zhou \\
   Department of Computer Science \\
   University of California, Los Angeles\\
   CA 90095 \\ 
   \texttt{drzhou@cs.ucla.edu}  \\
   \And
   Quanquan Gu \\
   Department of Computer Science \\
   University of California, Los Angeles\\
   CA 90095 \\ 
   \texttt{qgu@cs.ucla.edu} \\
}

\begin{document}

\maketitle

\begin{abstract}
We study the off-policy evaluation (OPE) problem in reinforcement learning with linear function approximation, which aims to estimate the value function of a target policy based on the offline data collected by a behavior policy. We propose to incorporate the variance information of the value function to improve the sample efficiency of OPE. More specifically, for time-inhomogeneous episodic linear Markov decision processes (MDPs), we propose an algorithm, \texttt{VA-OPE}, which uses the estimated variance of the value function to reweight the Bellman residual in Fitted Q-Iteration. We show that our algorithm achieves a tighter error bound than the best-known result. We also provide a fine-grained characterization of the distribution shift between the behavior policy and the target policy. Extensive numerical experiments corroborate our theory.
\end{abstract}

\section{Introduction}
Reinforcement learning (RL) has been a hot spot in both theory and practice in the past decade. Many efficient algorithms have been proposed and theoretically analyzed for finding the optimal policy adopted by an agent to maximize the long-term cumulative rewards. In contrast to online RL where the agent actively interacts with the environment, offline RL (a.k.a., batch RL) \citep{levine2020offline,lange2012batch} aims to extract information from past data and use this information to learn the optimal policy. There has been much empirical success of offline RL in various application domains \citep{ bertoluzzo2012testing, charles2013counterfactual, tang2013automatic, thomas2017predictive, quillen2018deep}.

Among various tasks of offline RL, an important task is called \emph{off-policy evaluation} (OPE), which evaluates the performance of a target policy $\pi$ given offline data generated by a behavior policy $\bar\pi$. Most existing theoretical works on OPE are in the setting of tabular MDPs \citep{precup2000eligibility, li2011unbiased,dudik2011doubly,jiang2016doubly,xie2019towards, yin2020asymptotically,yin2020near,yin2021near}, where the state space $\cS$ and the action space $\cA$ are both finite. 
However, real-world applications often have high-dimensional or even infinite-dimensional state and action spaces, where function approximation is required for computational tractability and generalization. While provably efficient online RL with linear function approximation has been widely studied recently \citep{yang2019sample,jin2020provably,zanette2020learning,jia2020model,ayoub2020model,zhou2020provably}, little work has been done for analyzing OPE with linear function approximation, with one notable exception by \citet{duan2020minimax}. More specifically, \citet{duan2020minimax} analyzed a regression-based Fitted Q-Iteration method (\texttt{FQI-OPE}) that achieves an $\tilde \cO(H^2\sqrt{(1+d(\pi,\bar\pi))/N})$ error for linear MDPs \cite{yang2019sample,jin2020provably}, where $H$ is the planning horizon, $N$ is the sample size, and $d(\pi,\bar\pi)$ represents the distribution shift between the behavior policy and the target policy. They also proved a sample complexity lower bound for a subclass of linear MDPs, for which their algorithm is nearly minimax optimal. 
However, as we will show later, the $H^2$ dependence is not tight since they discard the useful variance information contained in the offline data. Consequently, their result is only optimal for a small class of MDPs of which the value functions have large variance.
The $H^2$ dependence in the sample complexity also makes their algorithm less sample-efficient for long-horizon problems, which is one of the major challenges in RL. 

Extracting useful information from the data is particularly important for offline RL since the agent cannot sample additional data by interacting with the environment, as compared to online RL. 
In this paper, we propose a new algorithm that incorporates the variance information of the value functions to improve the sample efficiency of OPE. 
This allows us to achieve a deeper understanding and tighter error bounds of OPE with linear function approximation.
In detail, we consider time-inhomogeneous linear MDPs \citep{yang2019sample,jin2020provably} where the transition probability and reward function are assumed to be linear functions of a known feature mapping and may vary from stage to stage. 

The main contributions of this paper are summarized as follows:
\begin{itemize}[leftmargin = *]
    \item We develop \texttt{VA-OPE} (Variance-Aware Off-Policy Evaluation), an algorithm for OPE that effectively utilizes the variance information from the offline data. The core idea behind the proposed algorithm is to calibrate the Bellman residual in the regression by an estimator of the conditional variance of the value functions, such that data points of higher quality can receive larger important weights.   
    \item We show that our algorithm achieves $\tilde \cO(\sum_h({\mathbf{v}}_h^\top{\mathbf{\Lambda}}_h^{-1}{\mathbf{v}}_h)^{1/2}/\sqrt{K})$ policy evaluation error, where ${\mathbf{v}}_h$ is the expectation of the feature vectors under target policy and ${\mathbf{\Lambda}}_h$ is the uncentered covariance matrix under behavior policy weighted by the conditional variance of the value function. Our algorithm achieves a tighter error bound and milder dependence on $H$ than \texttt{FQI-OPE} \citep{duan2020minimax}, and provides a tighter characterization of the distribution shift between the behavior policy and the target policy, which is also verified by extensive numerical experiments.  
    \item Our analysis is based on a novel two-step proof technique. In the first step, we use backward induction to establish worst-case uniform convergence\footnote{By uniform convergence we mean the convergence of the estimated value functions in $\ell_{\infty}$-norm to their true values, which is different from the uniform convergence over all policies in \citet{yin2020near}.} results for the estimators of the value functions. In the second step, the convergence of OPE estimator is proved by tightening the uniform convergence result based on an average-case analysis. Our proof strategy provides a generic way for analyzing (weighted) ridge regression methods that are carried out in a backward and iterated fashion. The analyses in both steps might be of independent interest.
\end{itemize}

\noindent\textbf{Notation} 
We use lower case letters to denote scalars and use lower and upper case boldface letters to denote vectors and matrices respectively. For any vector $\xb\in \RR^d$ and any positive semi-definite matrix $\bSigma\in \RR^{d\times d}$, we denote by $\|\xb\|_2$ the Euclidean norm and $\|\bSigma\|$ the operator norm, and define $\|\xb\|_{\bSigma}=\sqrt{\xb^\top\bSigma\xb}$. For any positive integer $n$, we denote by $[n]$ the set $\{1,\dots,n\}$. For any finite set $A$, we denote by $|A|$ the cardinality of $A$. For two sequences $\{a_n\}$ and $\{b_n\}$, we write $a_n=\cO(b_n)$ if there exists an absolute constant $C$ such that $a_n\leq Cb_n$, and we write $a_n=\Omega(b_n)$ if there exists an absolute constant $C$ such that $a_n\geq Cb_n$. We use $\tilde \cO(\cdot)$ to further hide the logarithmic factors.

\section{Preliminaries}\label{sec:problem_setup}
\subsection{Markov Decision Processes}
 We consider the time-inhomogeneous episodic Markov Decision Process (MDP), which is represented by a tuple $M(\cS, \cA, H, \{\reward_h\}_{h=1}^H, \{\PP_h\}_{h=1}^H)$. In specific, we denote the state space by $\cS$ and the action space by $\cA$, and $H>0$ is the horizon length of each episode. At each stage $h\in[H]$, $\reward_h: \cS \times \cA \rightarrow [0,1]$ is the reward function, and $\PP_h(s'|s,a) $ is the transition probability function which represents the probability for state $s$ to transit to state $s'$ given action $a$. A policy $\pi$ consists of $H$ mappings $\{\pi_h\}_{h=1}^{H}$ from $\cS$ to the simplex on $\cA$, such that for any $(h,s)\in[H]\times\cS$, $\pi_h(\cdot|s)$ is a probability distribution over $\cA$. Here a policy can be either deterministic (point mass) or stochastic. For any policy $\pi$, we define the associated action-value function $\qvalue^\pi_h(s,a)$ and value function $\vvalue^\pi_h(s)$ at each stage $h\in[H]$ as follows:
\begin{align}\label{eq:relation Q and V}
&\qvalue^\pi_h(s,a) =  \EE_{\pi}\bigg[\sum_{i = h}^{H} \reward_i(s_{i}, a_{i})\bigg|s_h = s, a_h = a\bigg], \qquad \vvalue^\pi_h(s) = \int_{\cA} \qvalue^\pi_h(s,a)\diff\pi_h(a|s), 
\end{align}
where $a_i\sim\pi_i(\cdot|s_i)$ and $s_{i+1}\sim\PP_i(\cdot|s_i,a_i)$.
For any function $V:\cS\to\RR$, we introduce the following shorthand notation for the conditional expectation and variance of $V$: 
\begin{align}\label{def: value function}
    [\PP_hV](s,a)&=\EE_{s'\sim\PP_h(\cdot|s,a)}[V(s')], \qquad [\VV_hV](s,a)=[\PP_hV^2](s,a)-([\PP_hV](s,a))^2.
\end{align}

\noindent\textbf{Time-inhomogeneous linear MDPs.} We consider a special class of MDPs called \emph{linear MDPs}~\citep{yang2019sample, jin2020provably}. Note that most of the existing works on RL with linear function approximation rely on this assumption.

\begin{assumption}\label{assump:linear_MDP} $M(\cS, \cA, H, \{r_h\}_{h=1}^H, \{\PP_h\}_{h=1}^H)$ is called a linear MDP with a \emph{known} feature mapping $\bphi: \cS \times \cA \to \RR^d$, if for any $h \in [H]$, there exist $\bgamma_h$ and $\bmu_h \in \RR^d$, such that for any state-action pair $(s,a)\in \cS \times \cA$, it holds that
\begin{align}\label{eq:linear_MDP}
    \PP_h(\cdot \mid s,a) & = \langle \bphi(s,a), \bmu_h(\cdot) \rangle, \qquad r_h(s,a) = \langle \bphi (s,a), \bgamma_h \rangle.
\end{align}
We assume that at any stage $h$, for any state-action pair $(s,a)\in\cS\times\cA$, the reward received by the agent is given by $r = r_h(s,a) + \epsilon_h(s,a)$, where $r_h(s,a)\in[0,1]$ is the expected reward and $\epsilon_h(s,a)$ is the random noise. We assume that the noise is zero-mean and independent of anything else.

Without loss of generality, we assume that $\|\bgamma_h\|_2\leq 1$ and $\| \bphi(s,a)\|_2\leq 1$ for all $(s,a)\in \cS\times \cA$. We also assume that $r_h(s,a) + \epsilon_h(s,a)\leq1$, $|\epsilon_h(s,a)|\leq 1$ almost surely and thus $\Var(\epsilon_h(s,a)) \leq 1$ for all $h\in[H]$ and $(s,a)\in\cS\times\cA$. Moreover, we assume that $\max_{h\in[H]}\left\|\int_\cS f(s) \diff \bmu_h(s)\right\|_2\leq \sqrt{d}$ for all bounded function $f: \cS \to \RR$ such that $\sup_{s\in\cS}|f(s)|\leq 1$. 
\end{assumption}

The above assumption on linear MDPs implies the following proposition for the action-value functions.
\begin{proposition}[Proposition 2.3, \cite{jin2020provably}]\label{proposition: linear Q}
For a linear MDP, for any policy $\pi$, there exist weights $\{\wb_h^\pi,h\in[H]\}$ such that for any $(s,a,h)\in\cS\times\cA\times[H]$, we have $Q_h^\pi(s,a)=\langle\bphi(s,a),\wb_h^\pi\rangle$. Moreover, we have $\|\wb_h^\pi\|_2\leq 2H\sqrt{d}$ for all $h\in[H]$. 
\end{proposition}

Following this proposition, we may further show that the value functions are also linear functions, but of different features. We define $\bphi_h^\pi(s)=\int_{\cA} \bphi(s,a)\diff\pi_h(a|s)$ for all $s\in[S]$ and $h\in[H]$. Then by \eqref{eq:relation Q and V} we have 
\begin{align*}
        V_h^\pi(s) =\int_{\cA}\bphi(s,a)^\top\wb_h^\pi \diff\pi(a|s)=\langle\bphi_h^\pi(s),\wb_h^\pi\rangle .     
\end{align*}

\subsection{Off-policy Evaluation}
The purpose of OPE is to evaluate a (known) target policy $\pi$ given an offline dataset generated by a different (unknown) behavior policy $\bar\pi$. In this paper, our goal is to estimate the expectation of the value function induced by $\pi$ over a fixed initial distribution $\xi_1$, i.e.,
\begin{align*}
    v_1^\pi = \EE_{s\sim\xi_1}[V_1^\pi(s)].
\end{align*}
To faciliate the presentation, we further introduce some important notations. For all $h\in[H]$, let $\nu_h$ be the occupancy measure over $\cS\times\cA$ at stage $h$ induced by the transition $\PP$ and the behavior policy $\bar\pi$, that is, for any $E\subseteq\cS\times\cA$,
\begin{align}\label{eq:def of occupancy measure}
    \nu_h(E)=\EE\left[(s_h,a_h)\in E\mid s_1\sim\xi_1, \ a_i\sim\bar{\pi}(\cdot|s_i), \ s_{i+1}\sim\PP_i(\cdot|s_i,a_i), \ 1\leq i\leq h\right] .
\end{align}
For simplicity, we write $\EE_{h}[f(s,a)]=\EE_{\bar \pi,h}[f(s,a)]=\int_{\cS\times\cA}f(s,a)\text{d}\nu_h(s,a)$ for any function $f$ on $\cS\times\cA$. Similarly, we use $\EE_{\pi, h} [f(s,a)]$ to denote the expectation of $f$ with respect to the occupancy measure at stage $h$ induced by the transition $\PP$ and the target policy $\pi$.

We define the following uncentered covariance matrix under behavior policy for all $h\in[H]$:
\begin{align}\label{def: Sigma}
    \bSigma_h = \EE_{\bar\pi,h}\left[\bphi(s,a)\bphi(s,a)^\top\right].
\end{align}
Intuitively, these matrices measure the coverage of the offline data in the state-action space. It is known that the success of OPE necessitates a good coverage \citep{duan2020minimax,wang2021instabilities}.
Therefore here we make the same coverage assumption on the offline data. 

\begin{assumption}[Coverage]\label{assump: smallest eigenvalue}
For all $h\in[H]$, $\kappa_h\coloneqq\lambda_{\min}(\bSigma_h)>0$. Denote $\kappa=\min_{h\in[H]}\kappa_h$.
\end{assumption}

A key difference in our result is that, instead of depending on $\bSigma_h$ directly, the error bound depends on the following weighted version of the covariance matrices defined as
\begin{align}\label{def: Lambda}
    \bLambda_h \coloneqq \EE_{\bar\pi,h}\left[\sigma_h(s,a)^{-2}\bphi(s,a)\bphi(s,a)^\top\right] , 
\end{align}
for all $h\in[H]$, where each $\sigma_h: \cS \times \cA \to \RR$ is defined as 
\begin{align}\label{def: sigmah}
    \sigma_h(s,a) \coloneqq \sqrt{\max\{1,\VV_hV_{h+1}^\pi(s,a)\}+1} . 
\end{align}
Note that in the definition of $\sigma_h(\cdot,\cdot)$, taking the maximum and adding an extra $1$ is purely for technical reason and is related to its estimator $\hatsigma_h(\cdot,\cdot)$, which we will introduce and explain later in Section \ref{sec: proposed alg}. In general, one can think of $\sigma_h^2(s,a) \approx \VV_hV_{h+1}^\pi(s,a)$. Therefore, compared with the raw covariance matrix $\bSigma_h$, $\bLambda_h$ further incorporates the variance of the value functions under the target policy.
This is the key to obtaining a tighter instance-dependent error bound.

\begin{definition}[Variance-aware coverage]
We define $\iota_h\coloneqq\lambda_{\min}(\bLambda_h)$ and $\iota=\min_{h\in[H]}\iota_h$.
\end{definition}
Since $\sup_{(s,a)\in\cS\times\cA}\sigma_h(s,a)^2$ is bounded from above, by \eqref{def: Lambda} and Assumption \ref{assump: smallest eigenvalue}, we immediately have $\iota_h\geq\kappa_h/[\sup_{(s,a)\in\cS\times\cA} \sigma_h(s,a)^2] >0$ for all $h\in[H]$, and thus $\iota>0$.
Even if Assumption~\ref{assump: smallest eigenvalue} does not hold, we can always restrict to the subspace $\text{span}\{\bphi(s_h,a_h)\}$. For convenience of presentation, we make Assumption \ref{assump: smallest eigenvalue} in this paper.

Next, we introduce the assumption on the sampling process of the offline data.

\begin{assumption}[Stage-sampling Data]\label{assump:data_generation stage sampling}
We have two offline datasets $\cD$ and $\check{\cD}$ where each dataset consists of data from $H$ stages: $\cD = \{ \cD_h\}_{h \in [H]} $ and $\check \cD = \{ \check\cD_h\}_{h \in [H]} $. For the dataset $\cD$, we assume $\cD_{h_1}$ is independent of $\cD_{h_2}$ for $h_1\neq h_2$. For each stage $h$, we have $\cD_h = \{(s_{k,h},a_{k,h},r_{k,h},s'_{k,h})\}_{k\in[K]}$, where we assume for each $k\in[K]$, the data point $(s_{k,h},a_{k,h},r_{k,h},s'_{k,h})$ is sampled identically and independently in the following way: $(s_{k,h},a_{k,h})\sim \nu_h(\cdot,\cdot)$ where $\nu_h(\cdot,\cdot)$ is the occupancy measure defined in \eqref{eq:def of occupancy measure}, and $s'_{k,h}\sim\PP_h(\cdot | s_{k,h},a_{k,h})$. 
The same holds for $\check\cD$, and we write $\check \cD_h = \{(\check s_{k,h},\check a_{k,h},\check r_{k,h},\check s'_{k,h})\}_{k\in[K]}$. Note that here $s'_{k,h} \neq s_{k,h+1}$.
\end{assumption}

Assumptions \ref{assump:data_generation stage sampling} is standard in the offline RL literature \citep{yin2020near, duan2020minimax}. 
Note that in the assumption, there is a data splitting, i.e., one can view it as the whole dataset $\cD \cup \check{\cD}$ being split into two halves. The datasets $\cD$ and $\check \cD$ will then be used for two different purposes in Algorithm \ref{alg:weighted FQI} as will be made clear in the next section. We would like to remark that the only purpose of the splitting is to avoid a lengthy analysis. There is no need to perform the data splitting in practice. Also, in our implementation and experiments, we do not split the data.

\section{Algorithm}
To ease the notation, we denote $\bphi_{k,h}=\bphi(s_{k,h},a_{k,h})$, $\check \bphi_{k,h}=\bphi(\check s_{k,h},\check a_{k,h})$, $\hat\sigma_{k,h}=\hat\sigma_h(s_{k,h},a_{k,h})$ and $r_{k,h}=r_h(s_{k,h},a_{k,h}) + \epsilon_{k,h}$  for all $(h,k)\in[H]\times[K]$. Recall that we use the check mark to denote the other half of the splitted dataset. How the splitted data is utilized will be clear in Section \ref{sec: proposed alg} when we introduce the proposed algorithm. 

\subsection{Regression-Based Value Function Estimation}\label{sec: reg value est}

By Proposition \ref{proposition: linear Q}, it suffices to estimate the vectors $\{\wb_h^\pi,h\in[H]\}$. A popular approach is to apply the Least-Square Value Iteration (LSVI) \citep{jin2020provably} which relies on the Bellman equation, $Q_h^\pi(s,a)=r_h(s,a)+[\PP_hV_{h+1}^\pi](s,a)$, that holds for all $h\in[H]$ and $(s,a)\in\cS\times\cA$.
By viewing $V_{h+1}^\pi(s'_{k,h})$ as an unbiased estimate of $[\PP_hV_{h+1}^\pi](s_{k,h},a_{k,h})$, the idea of the LSVI-type method is to solve the following ridge regression problem:
\begin{align}\label{eq:FQI}
        \hat{\wb}_h^\pi& := \argmin_{\wb \in \RR^d} \lambda \norm{\wb}_2^2 +\sum_{k=1}^K \left[ \langle \bphi_{k,h},\wb \rangle - r_{k,h} - V_{h+1}^\pi(s_{k,h}') \right]^2, 
\end{align}
for some regularization parameter $\lambda >0$. 
Since we do not know the exact values of $V_{h+1}^\pi$ in \eqref{eq:FQI}, we replace it by an estimator $\hat V_{h+1}^\pi$, and then recursively solve the lease-square problem in a backward manner, which enjoys a closed-form solution as follows
\begin{align*}
    \hat\wb_h^\pi = \left[\sum_{k=1}^K\bphi_{k,h}\bphi_{k,h}^\top + \lambda \Ib_d\right]^{-1}\sum_{k=1}^K\bphi_{k,h}\left[r_{k,h}+\hat V_{h+1}^\pi(s_{k,h}')\right] \, . 
\end{align*}
This has been used in the \texttt{LSVI-UCB} algorithm proposed by \citet{jin2020provably} and the \texttt{FQI-OPE} algorithm studied by \citet{duan2020minimax}, for online learning and OPE of linear MDPs respectively. For this kind of algorithms, the key difficulty in the analysis lies in bounding the Bellman error:
\begin{align*}
    \left[\sum_{k=1}^K\bphi_{k,h}\bphi_{k,h}^\top + \lambda \Ib_d\right]^{-1}  \sum_{k=1}^K\bphi_{k,h}\big([\PP_h\hat V_{h+1}^\pi](s_{k,h},a_{k,h})- \hat V_{h+1}^\pi(s_{k,h}')\big).
\end{align*}
\citet{jin2020provably} applied a Hoeffding-type inequality to bound the Bellman error. Although \citet{duan2020minimax} applied Freedman's inequality in their analysis, their algorithm design overlooks the variance information in the data and consequently they can only adopt a crude upper bound on the conditional variance of the value function, i.e., $\VV_hV_{h+1}^\pi\leq (H-h)^2$, which simply comes from $\sup_{s\in\cS}V_{h+1}^\pi(s)\leq H-h$. Therefore, it prevents \citep{duan2020minimax} from getting a tight instance-dependent error bound for OPE. This is further verified by our numerical experiments in Appendix \ref{sec: simulation detail} which show that the performance of \texttt{FQI-OPE} degrades for large $H$.  
This motivates us to utilize the variance information in the data for OPE.

\subsection{The Proposed Algorithm}\label{sec: proposed alg}
In particular, we present our main algorithm as displayed in Algorithm \ref{alg:weighted FQI}. Due to the greedy nature of the value functions, we adopt a backward estimation scheme. 

\paragraph{Weighted ridge regression.} For any $h\in[H]$, let $\hat\wb_{h+1}^\pi$ be the estimate of $\wb_{h+1}^\pi$ computed at the previous step, and correspondingly $\hat V_{h+1}^\pi(\cdot)=\langle\bphi_{h+1}^\pi(\cdot),\hat\wb_{h+1}^\pi\rangle$. Instead of the ordinary ridge regression \eqref{eq:FQI}, we consider the following weighted ridge regression:
\begin{align}\label{eq:weighted_FQI}
    \hat{\wb}_h^\pi &:= \argmin_{\wb \in \RR^d} \lambda \norm{\wb}_2^2 + \sum_{k=1}^K \left[ \langle \bphi_{k,h},\wb \rangle - r_{k,h} - \hatV_{h+1}^\pi (s_{k,h}') \right]^2 \big/ \  \hat\sigma_{k,h}^2,    
\end{align} 
where $\hatsigma_{k,h}=\hatsigma_h(s_{k,h},a_{k,h})$ for all $(h,k)\in[H]\times[K]$ with $\hat\sigma_h(\cdot,\cdot)$ being a proper estimate of $\sigma_h(\cdot,\cdot)$ defined in \eqref{def: sigmah}. We then have the following closed-form solution (Line \ref{alg: w estimate} and \ref{alg: Lambda} of Alg. \ref{alg:weighted FQI}):
\begin{align}\label{def: hatLambda}
    \hat\wb_h^\pi = \hat\bLambda_h^{-1} \sum_{k=1}^K \bphi_{k,h} \left( r_{k,h}+\hatV_{h+1}^\pi (s_{k,h}') \right) \big/ \  \hat\sigma_{k,h}^2, \ \textnormal{with} \ \hat\bLambda_h = \sum_{k=1}^K  \hat\sigma_{k,h}^{-2}\bphi_{k,h}\bphi_{k,h}^\top+\lambda\Ib_d.  
\end{align}
In the above estimator, we use the dataset $\cD$ to estimate the value functions. Next, we apply an LSVI-type method to estimate $\sigma_h$ using the dataset $\check \cD$.

\paragraph{Variance estimator.} By \eqref{def: value function}, we can write
\begin{align}\label{eq:var_hatV_decompose}
    [\VV_h V_{h+1}^\pi ](s , a) = & [\PP_h (V_{h+1}^{\pi})^2 ](s, a) - \left( [\PP_h V_{h+1}^\pi ](s , a)  \right)^2.    
\end{align}
For the first term in \eqref{eq:var_hatV_decompose}, by Assumption \ref{assump:linear_MDP} we have
\[
     [\PP_h (V_{h+1}^{\pi})^2 ](s, a) = \int_{\cS}V_{h+1}^\pi(s')^2\diff\PP_h(s'|s,a) = \bphi(s,a)^\top\int_{\cS} V_{h+1}^\pi(s')^2 \ \diff\bmu_h(s'),
\]
which suggests that $\PP_h(V_{h+1}^\pi)^2$ also has a linear representation. Thus we adopt a linear estimator $\langle\bphi(s,a),\hatbeta_h^\pi\rangle$ where $\hatbeta_h^\pi$ (Line \ref{alg: beta}) is the solution to the following ridge regression problem:
\begin{align}\label{eq: hat beta}
   \hatbeta_h^\pi &= \argmin_{\bbeta \in \RR^d}  \sum_{k=1}^K \left[ \left\langle \check\bphi_{k,h} , \bbeta \right\rangle - [ \hatV_{h+1}^\pi]^2 (\check{s}_{k,h}') \right]^2 + \lambda \norm{\bbeta}_2^2 = \hat{\bSigma}_h^{-1}\sum_{k=1}^K \check\bphi_{k,h} \hatV_{h+1}^\pi( \check{s}_{k,h}')^2.
\end{align}
Similarly, we estimate the second term in \eqref{eq:var_hatV_decompose} by $\langle \bphi (s,a), \hattheta_h^\pi \rangle$, where $\hattheta_h^\pi$ (Line \ref{alg: theta}) is given by 
\begin{align}\label{eq: hat theta}
    \hattheta_h &= \argmin_{\btheta \in \RR^d}  \sum_{k=1}^K \left[ \left\langle \check\bphi_{k,h}, \btheta \right\rangle - \hatV_{h+1}^\pi (\check{s}_{k,h}') \right]^2+\lambda \norm{\btheta}_2^2 = \hat\bSigma_h^{-1} \sum_{k=1}^K \check\bphi_{k,h} \hat V_{h+1}^\pi( \check{s}_{k,h}'),
\end{align}with $\hat\bSigma_h=\sum_{k=1}^K \check\bphi_{k,h} \check\bphi_{k,h}^\top + \lambda \Ib_d$.
Combining \eqref{eq: hat beta} and \eqref{eq: hat theta}, we estimate $\VV_h V_{h+1}^\pi $ by 
\begin{align}\label{eq:var_V_estimate}
     [\hat\VV_h \hatV_{h+1}^\pi] (\cdot, \cdot)& = \langle \bphi(\cdot, \cdot) , \hatbeta_h^\pi \rangle_{[0,(H-h+1)^2]} - \left[  \langle \bphi(\cdot,\cdot), \hattheta_h^\pi \rangle_{[0,H-h+1]}   \right]^2 ,    
\end{align}
where the subscript $[0, (H-h+1)^2]$ denotes the clipping into the given range, and similar for the subscript $[0,H-h+1]$. We do such clipping due to the fact that $V_{h+1}^\pi\in[0,H-h]$. 
We add $1$ to deal with the approximation error in $\hatV_{h+1}^\pi$. 

Based on $\hat\VV_h\hat V_{h+1}^\pi$, the final variance estimator $\hat\sigma_h(\cdot,\cdot)$ (Line \ref{alg: sigma}) is defined as
\begin{align*}
    \hat\sigma_h(\cdot,\cdot) = \sqrt{\max\{1, \hat\VV_h \hatV_{h+1}^\pi(\cdot,\cdot)\}+1}.
\end{align*}
In order to deal with the situation where $\hat\VV_h \hatV_{h+1}^\pi < 0$ or is very close to $0$, we take maximum between $\hat\VV_h\hat V_{h+1}^\pi$ and 1. Also, to account for the noise in the observed rewards, we add an extra $1$ which is an upper bound of the noise variance by Assumption \ref{assump:linear_MDP}.

\paragraph{Final estimator.} Recursively repeat the above procedure for $h=H,H-1,\ldots,1$, and we obtain $\hat V_1$. Then the final estimator for $v_1^\pi$ (Line \ref{alg: final}) is defined as $\hat v_1^\pi = \int_\cS \hat V_1^\pi(s) \ \diff\xi_1(s).$

\begin{algorithm}[t]
	\caption{Variance-Aware Off-Policy Evaluation (\texttt{VA-OPE})}
	\label{alg:weighted FQI}
	\begin{algorithmic}[1]
	\STATE {\bfseries Input:} target policy $\pi = \{\pi_h\}_{h \in [H]}$, datasets $\cD=\{ \{ (s_{k,h},a_{k,h},r_{k,h}, s_{k,h}') \}_{h\in[H]} \}_{k\in[K]}$ and $\check{\cD}=\{ \{ (\check{s}_{k,h},\check{a}_{k,h}, \check{r}_{k,h}, \check{s}_{k,h}') \}_{h\in[H]} \}_{k\in[K]}$, initial distribution $\xi_1$, $\hatw_{H+1}^\pi = \mathbf{0}$ 
	\FOR{$h=H,H-1,\dots,1$}
	\STATE $\hat\bSigma_h \leftarrow \sum_{k=1}^K \check\bphi_{k,h}\check\bphi_{k,h}^\top + \lambda\Ib_d$ \alglinelabel{alg: Sigma} 
	\STATE $\hatbeta_h \leftarrow \hat{\bSigma}_h^{-1} \sum_{k=1}^K \check\bphi_{k,h} \hatV_{h+1}^\pi(\check{s}_{k,h}')^2$\alglinelabel{alg: beta}
	\STATE $\hattheta_h \leftarrow \hat{\bSigma}_h^{-1} \sum_{k=1}^K \check\bphi_{k,h} \hatV_{h+1}^\pi(\check{s}_{k,h}')$\alglinelabel{alg: theta}
    \STATE $\hat\sigma_h(\cdot,\cdot) \leftarrow \sqrt{\max\{1,\hat\VV_h\hatV_{h+1}^\pi(\cdot,\cdot) \}+ 1}$ \alglinelabel{alg: sigma}
	\STATE $\hat\bLambda_h\leftarrow\sum_{k=1}^K\bphi_{k,h} \bphi_{k,h}^\top/\hat\sigma_{k,h}^2+ \lambda \Ib_d$\alglinelabel{alg: Lambda}
	\STATE $Y_{k,h} \leftarrow  r_{k,h} + \langle \bphi_h^\pi(s_{k,h}'), \hatw_{h+1}^\pi \rangle$
	\STATE $\hatw_h^\pi \leftarrow \hat\bLambda_h^{-1} \sum_{k=1}^K \bphi_{k,h}Y_{k,h} / \hat\sigma_{k,h}^2$\alglinelabel{alg: w estimate}
	\STATE $\hat Q_h^\pi(\cdot,\cdot)\leftarrow \langle\bphi(\cdot,\cdot),\hat\wb_h^\pi\rangle$,\quad $ \hatV_{h}^\pi (\cdot) \leftarrow \langle \bphi_h^\pi(\cdot), \hatw_{h}^\pi \rangle$ \alglinelabel{alg: Q and V}
	\ENDFOR
	\STATE {\bfseries Output:} $\hat{v}^\pi_1 \leftarrow \int_{\mathcal{S}} \hatV_1^\pi (s) \ \diff  \xi_1(s) $\alglinelabel{alg: final}
	\end{algorithmic}
\end{algorithm}

\paragraph{Intuition behind $\bLambda_h$.} To illustrate the intuition behind the weighted covariance matrix $\Lambda_h$, here we provide some brief heuristics. Let $\{(s_{k,h},a_{k,h},s'_{k,h})\}_{k\in[K]}$ be i.i.d. samples such that $(s_{k,h},a_{k,h})\sim \nu$ for some distribution $\nu$ over $\cS\times\cA$ and $s'_{k,h}\sim\PP_h(\cdot|s_{k,h},a_{k,h})$. Define
\begin{align*}
    \eb_k = \bphi(s_{k,h},a_{k,h})\left([\PP_hV_{h+1}^\pi](s_{k,h},a_{k,h})-V_{h+1}^\pi(s'_{k,h})\right)\big /\ [\VV_hV_{h+1}^\pi](s_{k,h},a_{k,h})^2
\end{align*}
for all $k\in[K]$. Note that $\eb_k$'s are i.i.d zero-mean random vectors and a simple calculation yields
\begin{align*}
    &\Cov(\eb_k) = \EE\left[[\VV_hV_{h+1}^\pi](s_{k,h},a_{k,h})^{-2}\bphi(s_{k,h},a_{k,h})\bphi(s_{k,h},a_{k,h})^\top\right].
\end{align*}
This coincides with \eqref{def: Lambda}. Suppose $\Cov(\eb_k)\succ 0$, then by the central limit theorem, it holds that
\[
    \frac{1}{\sqrt{K}}\sum_{k=1}^K\eb_k \overset{d}{\longrightarrow}\cN(0,\Cov(\eb_k)). 
\]
Therefore, $\Cov(\eb_k)^{-1}$, or equivalently $\bLambda_h^{-1}$, can be seen as the Fisher information matrix associated with the weighted product of the Bellman error and the feature vectors. This is a tighter characterization of the convergence rate than bounding $\VV_hV_{h+1}^\pi$ by its naive upper bound $(H-h)^2$.

\section{Theoretical Results}\label{sec:theory}

In this section, we introduce our main theoretical results and give an overview of the proof technique.

\subsection{OPE Error Bound}
Our main result is a refined average-case OPE analysis that yields a tighter error bound in Theorem~\ref{thm:ope}. The proof is in Appendix \ref{sec: proof of OPE}. To simplify the notation, we define:
\begin{align*}
    C_{h,2} = \sum_{i=h}^H \frac{H-h+1}{\sqrt{2\iota_h}} \,, \quad C_{h,3} = \frac{(H-h+1)^2}{2} \,, \quad C_{h,4} = \left(\|\bLambda_h \|\cdot \|\bLambda_h^{-1}\| \right)^{1/2} \,.    
\end{align*}

\begin{theorem}\label{thm:ope}
Set $\lambda=1$. Under Assumptions \ref{assump:linear_MDP}, \ref{assump: smallest eigenvalue} and \ref{assump:data_generation stage sampling}, if $K$ satisfies 
\begin{align}\label{eq:K lower bound OPE}
    K \geq C \cdot C_3 \cdot d^2 \left[ \log\left(\frac{dH^2K}{\kappa\delta}\right)\right]^2 ,
\end{align} where $C$ is some problem-independent universal constant and  
\begin{align*}
       C_3 \coloneqq \max & \bigg\{  \max_{h\in[H]} \frac{C_{h,3} \cdot C_{h,2}^2 }{ 8 \iota_h^2 } \, \ , \ \ \frac{H^4}{\kappa^2} \, , \  \frac{H^2}{\kappa^2} \cdot  \max_{h\in[H]}\frac{C_{h,3}}{2} \cdot\max_{h\in[H]}\frac{C_{h,3}}{\iota_h}   \bigg\},    
\end{align*}
then with probability at least $1-\delta$, the output of Algorithm \ref{alg:weighted FQI} satisfies
\begin{align*}
        |v_1^\pi - \hat{v}_1^\pi| \leq  &C \cdot \left[ \sum_{h=1}^H \left\| \vb_h^\pi \right\|_{\bLambda_h^{-1}} \right] \cdot \sqrt{\frac{\log(16H/\delta)}{K}} +  C \cdot C_4 \cdot \log\left(\frac{16H}{\delta}\right) \cdot \left( \frac{1}{K^{3/4}} + \frac{1}{K} \right) \,,    
\end{align*} where $\vb_h^\pi \coloneqq \EE_{\pi,h}[\bphi(s_h,a_h)]$ and $C_4 \coloneqq \sum_{h=1}^H  \sqrt{C_{h,4} \cdot C_{h,2} \cdot \frac{(H-h+1)d}{4\iota_h}\cdot \log\left(\frac{dH^2K}{\kappa\delta}\right)} \cdot \left\| \vb_h^\pi\right\|_{\bLambda_h^{-1}} $. 
\end{theorem}

Theorem \ref{thm:ope} suggests that Algorithm \ref{alg:weighted FQI} provably achieves a tighter instance-dependent error bound for OPE than that in \citep{duan2020minimax}. In detail, the dominant term in our bound is $\tilde\cO(\sum_{h=1}^H\|\vb_h^\pi\|_{\bLambda_h^{-1}}/\sqrt{K})$, as compared to the $\tilde\cO(\sum_{h=1}^H(H-h+1)\|\vb_h^\pi\|_{\bSigma_h^{-1}}/\sqrt{K})$ term in \citep{duan2020minimax}. By \eqref{def: Sigma} and \eqref{def: Lambda}, our bound is at least as good as the latter since $\bSigma_h\preceq [(H-h+1)^2+1]\bLambda_h$. More importantly, it is instance-dependent and tight for the general class of linear MDPs: for those where $\VV_hV_{h+1}^\pi$ is close to its crude upper bound $(H-h+1)^2$, our bound recovers the prior result. When $\VV_hV_{h+1}^\pi$ is small, \texttt{VA-OPE} benefits from incorporating the variance information and our bound gets tightened accordingly.

\begin{remark}
Note that we do not require $\VV_hV_{h+1}^\pi(s,a)$ to be uniformly small for all $s,a$, and $h$. From the bound and \eqref{def: Lambda}, as long as the variances are smaller than $(H-h+1)^2$  on average of $(s,a) \in \cS \times \cA$ and in sum of $h$, the bound is improved. 
It is also worth noting that the lower bound proved in \citep{duan2020minimax} only holds for a subclass of linear MDPs with $\VV_h V_{h+1}^\pi=\Omega((H-h+1)^2)$, and thus their minimax-optimality does not hold for general linear MDPs.
For more detailed comparison we refer the reader to Appendix \ref{sec:comparison with Duan}. 
\end{remark}

\begin{remark}
Conceptually, the term $\|\vb_h^\pi\|_{\bLambda_h^{-1}}$ serves as a more precise characterization of the distribution shift between the behavior policy $\bar\pi$ and the target policy $\pi$ in a variance-aware manner. This enables our algorithm to utilize the data more effectively. Compared with online RL where one can sample new data, OPE is more `data-hungry': one cannot decide the overall quality of the data. Thus it is especially beneficial to put more focus  on targeted values with less uncertainty. This is also the intuitive reason why our algorithm can achieve a tighter error bound. 
\end{remark}

\subsection{Overview of the Proof Technique}

Here we provide an overview of the proof for Theorem \ref{thm:ope}.
Due to the parallel estimation of the the value functions and their variances, the analysis of \texttt{VA-OPE} is much more challenging compared with that of \texttt{FQI-OPE}. As a result, we need to develop a novel proof technique. 
First, we have the following error decomposition. 

\begin{lemma}\label{lemma: error decomposition}
    For any $h\in[H]$, let $\hat V_h^\pi$ be the output of Algorithm \ref{alg:weighted FQI}. Then it holds that
\begin{align}
    V_h^\pi(s)-\hat V_h^\pi(s) &=\int_{\cA}[\PP_h(V_{h+1}^\pi-\hat V_{h+1}^\pi)](s,a)\diff\pi_h(a|s) + \lambda\bphi_h^\pi(s)^\top\hat\bLambda_h^{-1}\wb_h^\pi \label{eq: error decomp V}\\
    &\ +\bphi_h^\pi(s)^\top\hat\bLambda_h^{-1}\left[-\lambda \int_{\cS}\rbr{V_{h+1}^\pi(s')-\hat V_{h+1}^\pi(s')}\diff\bmu_h(s') + \sum_{k=1}^K\bphi_{k,h}\hat\sigma_{k,h}^{-2}\Delta_{k,h}\right],   \notag
\end{align}
where $\Delta_{k,h}=[\PP_h\hat V_{h+1}^\pi](s_{k,h},a_{k,h})-\hat V_{h+1}^\pi(s_{k,h}')-\epsilon_{k,h}$.
In particular, recall that $\hat{v}_1^\pi = \EE[\hatV_1^\pi(s_1)\mid s_1\sim\xi_1]$ and the OPE error can be decomposed as
\begin{align}\label{eq:ope decomp}
    v_1^\pi - \hat{v}_1^\pi = &-\lambda\sum_{h=1}^H (\vb_h^\pi)^\top \hat\bLambda_h^{-1}\int_{\cS}\rbr{V_{h+1}^\pi(s)-\hat V_{h+1}^\pi(s)}\bmu_h(s)\text{d}s \notag \\
    & + \sum_{h=1}^ H(\vb_h^\pi)^\top\hat\bLambda_h^{-1}\sum_{k=1}^K\bphi_{k,h}\hat\sigma_{k,h}^{-2}\Delta_{k,h} + \lambda \sum_{h=1}^H (\vb_h^\pi)^\top \hat\bLambda_h^{-1}\wb_h^\pi.
\end{align}
\end{lemma}
The OPE error bound (Theorem \ref{thm:ope}) is proved by bounding the three terms separately in \eqref{eq:ope decomp}. This decomposition is different from \citep{duan2020minimax} in that $\hat\bSigma_h$ is replaced by $\hat\bLambda_h$. This prevents us from adopting a matrix embedding-type proof as used in the prior work.

The key is to show the convergence of $\hat\bLambda_h$ to its population counterpart. However, by definition of $\hat\bLambda_h$, to establish such a result, it first requires the convergence of $\hat V_{h+1}^\pi$ to $V_{h+1}^\pi$ in a uniform manner, i.e., a high probability bound for $\sup_{s\in\cS}|\hat V_h^\pi(s)-V_h^\pi(s)|$. To show this, we leverage the decomposition in~\eqref{eq: error decomp V} and a backward induction technique, and prove a uniform convergence result which states that with high probability, for all $h \in[H]$, Algorithm \ref{alg:weighted FQI} can guarantee 
\[
\sup_{s\in\cS}\left|\hat V_h^\pi(s)-V_h^\pi(s)\right| \leq \tilde\cO\left( \frac{1}{\sqrt{K}} \right).
\]
This result is formalized as Theorem \ref{thm:uniform convergence} and proved in Appendix \ref{sec: proof of uniform convergence}. To the best of our knowledge, Theorem \ref{thm:uniform convergence} is the first to establish the uniform convergence of the estimation error for the value functions in offline RL with linear function approximation. We believe this result is of independent interest and may be broadly useful in OPE.

\section{Numerical Experiments}\label{sec:experiments}

In this section, we provide numerical experiments to evaluate our algorithm \texttt{VA-OPE}, and compare it with \texttt{FQI-OPE}.

We construct a linear MDP instance as follows. The MDP has $|\cS|=2$ states and $|\cA| = 100$ actions, with the feature dimension $d = 10$. The behavior policy then chooses action $a=0$ with probability $p$ and $ a\in\{1,\cdots,99\}$ with probability $1-p$ and uniformly over $\{1,\cdots,99\}$. The target policy $\pi$ always chooses $a = 0$ no matter which state it is, making state $0$ and $1$ absorbing. The parameter $p$ can be used to control the distribution shift between the behavior and target policies. Here $p \to 0$ leads to small distribution shift, and $p\to 1$ leads to large distribution shift. The initial distribution $\xi_1$ is uniform over $|\cS|$. For more details about the construction of the linear MDP and parameter configuration, please refer to Appendix \ref{sec: simulation detail}.

\begin{figure}[t]
	\begin{subfigure}[b]{0.32\textwidth}
		\includegraphics[width=\linewidth]{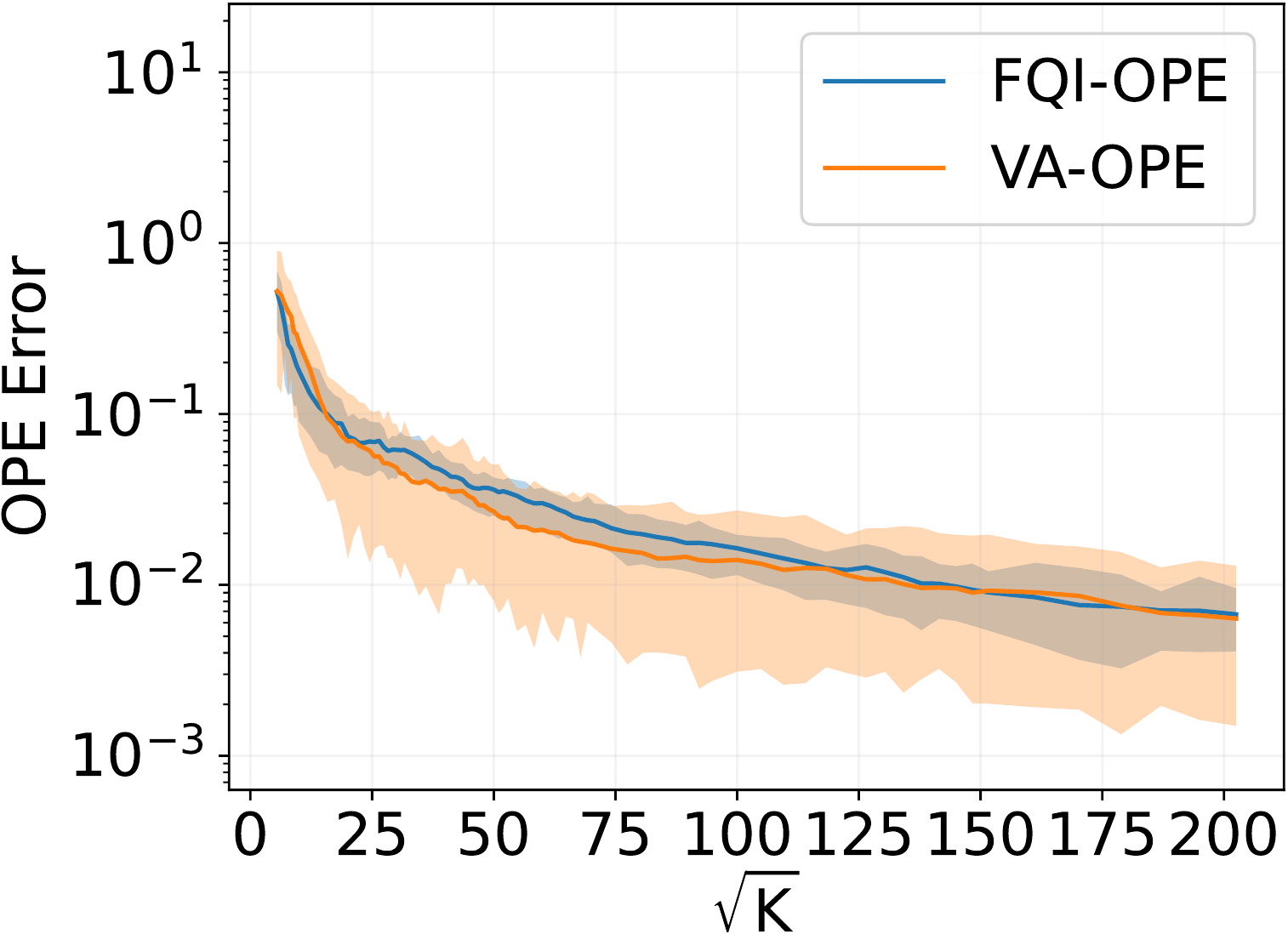}
		\caption{$H = 5$.}
		\label{fig: small h}
	\end{subfigure}
	\begin{subfigure}[b]{0.32\textwidth}
		\includegraphics[width=\linewidth]{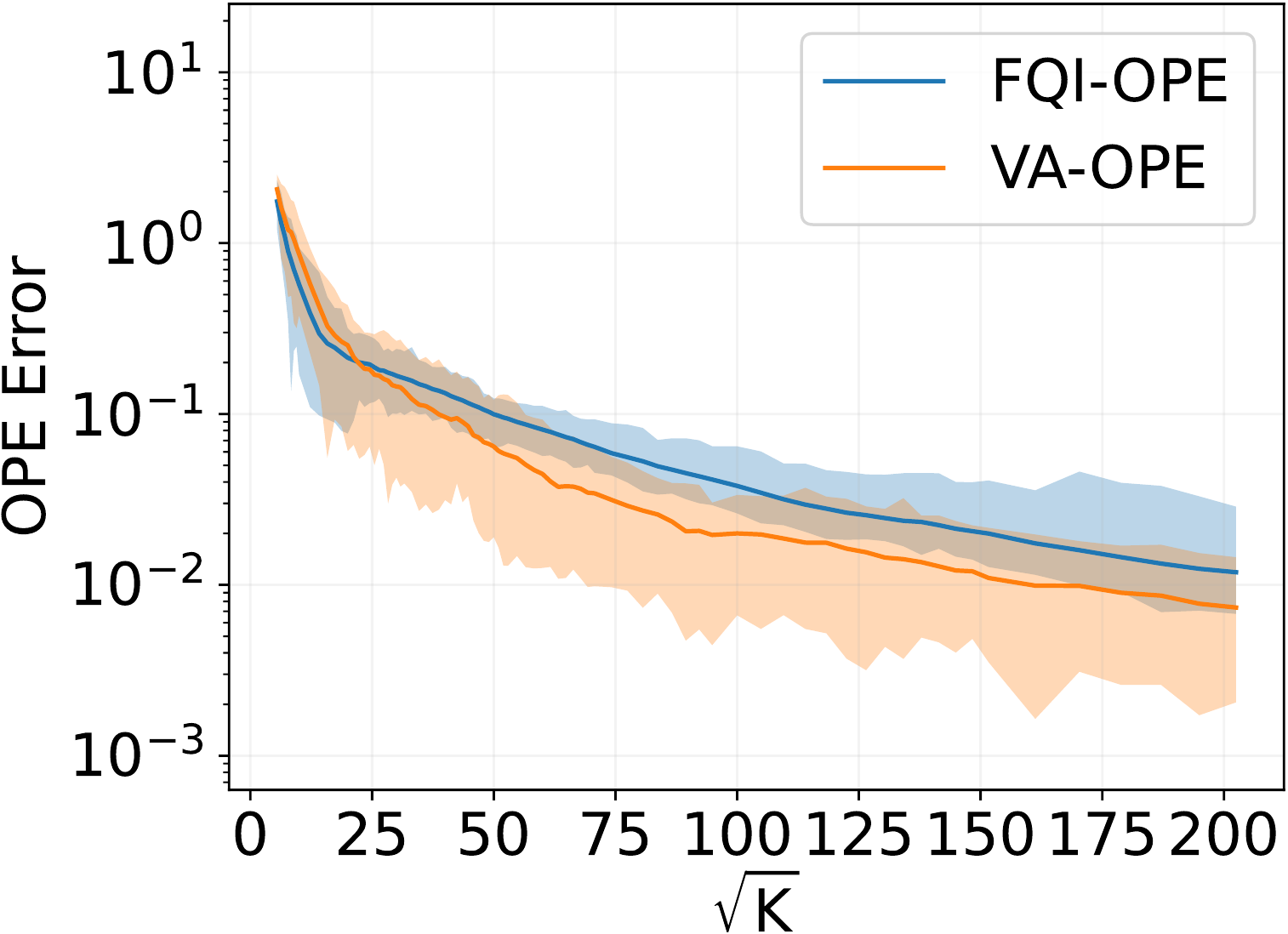}
		\caption{$H = 10$.}
		\label{fig: mid h}
	\end{subfigure}
	\begin{subfigure}[b]{0.32\textwidth}
		\includegraphics[width=\linewidth]{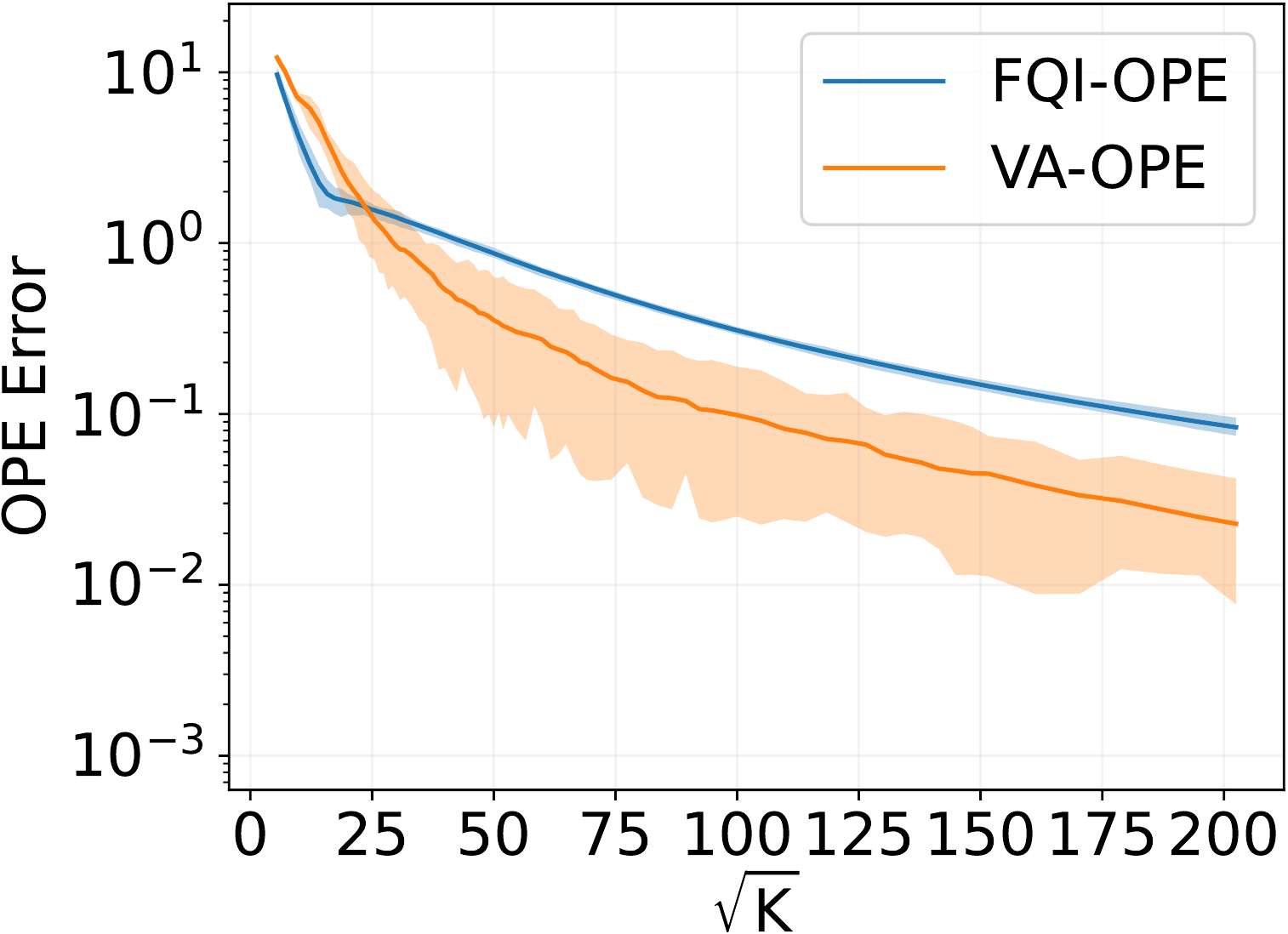}
		\caption{$H = 30$.}
		\label{fig: large h}
	\end{subfigure}
	\caption{Comparison of \texttt{VA-OPE} and \texttt{FQI-OPE} under different settings of horizon length $H$. \texttt{VA-OPE}'s advantage becomes more significant as $H$ increases, matching the theoretical prediction. The results are averaged over 50 trials and the error bars denote an empirical [10\%,90\%] confidence interval. The y-axis is log-scaled OPE error and x-axis is $\sqrt{K}$. For more details please see Appendix \ref{sec: simulation detail}.}
	\label{fig: h main}
\end{figure} 

We compare the performance of the two algorithms on the synthetic MDP described above under different choices of horizon length $H$. We plot the log-scaled OPE error versus $\sqrt{K}$ in Figure \ref{fig: h main}. It is clear that \texttt{VA-OPE} is at least as good as \texttt{FQI-OPE} in all the cases. Specifically, for small $H$ (Figure \ref{fig: small h}), their performance is very comparable, which is as expected. As $H$ increases, we can see from Figure \ref{fig: small h}, \ref{fig: mid h} and \ref{fig: large h} that \texttt{VA-OPE} starts to dominate \texttt{FQI-OPE}, and the advantage is more significant for larger $H$, as suggested by Theorem \ref{thm:ope}. Due to space limit, a comprehensive comparison under different parameter settings is deferred to Appendix \ref{sec: simulation detail}.

\section{Related Work}

\textbf{Off-policy evaluation.} 
There is a large body of literature on OPE for tabular MDPs. Since the seminal work by \citet{precup2000eligibility}, various importance sampling-based estimators have been studied in the literature \citep{li2011unbiased, li2015toward,thomas2016data}. 
By using marginalized importance sampling methods \citep{liu2018breaking, xie2019towards,kallus2019efficiently,yin2020asymptotically}, one is able to further break the ``curse-of-horizon''. 
Moreover, various doubly robust estimators \citep{dudik2011doubly,jiang2016doubly,farajtabar2018more, tang2019doubly, yin2021near} have been developed to achieve variance reduction. Most recently, it is shown by \citet{yin2020near} that uniform convergence over all possible policy is also achievable. However, all the aforementioned works are limited to tabular MDPs.
There is also a notable line of work on the estimation of the stationary distribution ratio between the target policy and the behavior policy using a primal-dual formulation \citep{nachum2019dualdice,zhang2019gendice,dai2020coindice}. However, a theoretical guarantee for the OPE error is not given in the work. More recently, \citet{chen2021infinite} studied OPE in the \emph{infinite-horizon} setting with linear function approximation. 

There are many others topics related to OPE, for example, policy gradient \citep{kakade2001natural, nachum2019algaedice,agarwal2019theory}, conservative policy iteration \citep{kakade2002approximately}, off-policy
temporal-difference learning \citep{precup2001off}, off-policy Q-learning \citep{kumar2019stabilizing},  safe policy iteration \citep{pirotta2013safe} and pessimism in RL \citep{kidambi2020morel, jin2020pessimism}, to mention a few. We refer the reader to the excellent survey by \citet{levine2020offline} for a more detailed introduction. 

\paragraph{Online RL with linear function approximation.}
RL with function approximation has been actively studied as an extension of the tabular setting. \citet{yang2019sample} studied discounted linear MDPs with a generative model, and \citet{jin2020provably} proposed an efficient \texttt{LSVI-UCB} algorithm for linear MDPs without a generative model. It has been shown by \citet{du2019good} that MDP with misspecified linear function approximation could be exponentially hard to learn. Linear MDPs under various settings have also been studied by \citep{zanette2020learning,neu2020unifying, he2020logarithmic,wang2021provably}. 

A parallel line of work studies linear mixture MDPs \citep{jia2020model,ayoub2020model,cai2020provably,zhou2020nearly, min2021learning} (a.k.a., linear kernel MDPs \citep{zhou2020provably}) where the transition kernel is a linear function of a ternary feature mapping $\psi:\cS\times\cA\times\cS\to\RR^d$. In particular, \citet{zhou2020nearly} achieved a nearly minimax regret bound by carefully utilizing the variance information of the value functions. \citet{zhang2021variance} constructed a variance-aware confidence set for time-homogeneous linear mixture MDPs. However, both works are focused on online RL rather than offline RL. It requires novel algorithm designs to exploit the variance information for offline tasks like OPE. What's more, the analysis in the offline setting deviates a lot from that for online RL where one can easily apply the law of total variance to obtain tighter bounds.

\section{Conclusion and Future Work}\label{sec: conclusion}

In this paper, we incorporate the variance information into OPE and propose \texttt{VA-OPE}, an algorithm that provably achieves tighter error bound. Our $\tilde O(\sum_h({\mathbf{v}}_h^\top{\mathbf{\Lambda}}_h^{-1}{\mathbf{v}}_h)^{1/2}/\sqrt{K})$ error bound has a sharper dependence on the distribution shift between the behavior policy and the target policy. 

Our work suggests several promising future directions. Theoretically, it remains open to provide an instance-dependent lower bound for the OPE error. Also, beyond the linear function approximation, it is interesting to establish similar results under more general function approximation schemes. Empirically, can we exploit the algorithmic insight of our algorithm to develop practically more data-effective OPE algorithms for complex real-world RL tasks? We wish to explore these directions in the future.

%\section*{Funding Transparency Statement}
\section*{Acknowledgments and Disclosure of Funding}
We thank Mengdi Wang and Yaqi Duan for helpful discussions during the preparation of the paper. 
We also thank the anonymous reviewers for their helpful comments. 
DZ and QG are partially supported by the National Science Foundation CAREER Award 1906169, IIS-1904183 and AWS Machine Learning Research Award. The views and conclusions contained in this paper are those of the authors and should not be interpreted as representing any funding agencies.

%\section*{Acknowledgment} 

% 1. Funding (financial activities supporting the submitted work)

% 2. Competing Interests (financial activities outside the submitted work)

\bibliographystyle{abbrvnat}
\bibliography{reference}

%%%%%%%%%%%%%%%%%%%%%%%%%%%%%%%%%%%%%%%%%%%%%%%%%%%%%%%%%%%%

\section*{Checklist}

%%% BEGIN INSTRUCTIONS %%%
% The checklist follows the references.  Please
% read the checklist guidelines carefully for information on how to answer these
% questions.  For each question, change the default \answerTODO{} to \answerYes{},
% \answerNo{}, or \answerNA{}.  You are strongly encouraged to include a {\bf
% justification to your answer}, either by referencing the appropriate section of
% your paper or providing a brief inline description.  For example:
% \begin{itemize}
%   \item Did you include the license to the code and datasets? \answerYes{See Section 2.}
%   \item Did you include the license to the code and datasets? \answerNo{The code and the data are proprietary.}
%   \item Did you include the license to the code and datasets? \answerNA{}
% \end{itemize}
% Please do not modify the questions and only use the provided macros for your
% answers.  Note that the Checklist section does not count towards the page
% limit.  In your paper, please delete this instructions block and only keep the
% Checklist section heading above along with the questions/answers below.
%%% END INSTRUCTIONS %%%

\begin{enumerate}

\item For all authors...
\begin{enumerate}
  \item Do the main claims made in the abstract and introduction accurately reflect the paper's contributions and scope?
    \answerYes
  \item Did you describe the limitations of your work?
    \answerYes
  \item Did you discuss any potential negative societal impacts of your work?
    \answerNA{} Explain: our work is seeking to develop a mathematical understanding of the off-policy evaluation in reinforcement learning. 
  \item Have you read the ethics review guidelines and ensured that your paper conforms to them?
    \answerYes
\end{enumerate}

\item If you are including theoretical results...
\begin{enumerate}
  \item Did you state the full set of assumptions of all theoretical results?
    \answerYes
	\item Did you include complete proofs of all theoretical results?
    \answerYes
\end{enumerate}

\item If you ran experiments...
\begin{enumerate}
  \item Did you include the code, data, and instructions needed to reproduce the main experimental results (either in the supplemental material or as a URL)?
    \answerYes
  \item Did you specify all the training details (e.g., data splits, hyperparameters, how they were chosen)?
    \answerYes
	\item Did you report error bars (e.g., with respect to the random seed after running experiments multiple times)?
    \answerYes
	\item Did you include the total amount of compute and the type of resources used (e.g., type of GPUs, internal cluster, or cloud provider)?
    \answerYes
\end{enumerate}

\item If you are using existing assets (e.g., code, data, models) or curating/releasing new assets...
\begin{enumerate}
  \item If your work uses existing assets, did you cite the creators?
    \answerNA{}
  \item Did you mention the license of the assets?
    \answerNA{}
  \item Did you include any new assets either in the supplemental material or as a URL?
    \answerNA{}
  \item Did you discuss whether and how consent was obtained from people whose data you're using/curating?
    \answerNA{}
  \item Did you discuss whether the data you are using/curating contains personally identifiable information or offensive content?
    \answerNA{}
\end{enumerate}

\item If you used crowdsourcing or conducted research with human subjects...
\begin{enumerate}
  \item Did you include the full text of instructions given to participants and screenshots, if applicable?
    \answerNA{}
  \item Did you describe any potential participant risks, with links to Institutional Review Board (IRB) approvals, if applicable?
    \answerNA{}
  \item Did you include the estimated hourly wage paid to participants and the total amount spent on participant compensation?
    \answerNA{}
\end{enumerate}

\end{enumerate}

%%%%%%%%%%%%%%%%%%%%%%%%%%%%%%%%%%%%%%%%%%%%%%%%%%%%%%%%%%%%

\clearpage

\appendix

\section{Details of the Experiments}\label{sec: simulation detail}

\subsection{A Synthetic Linear MDP Example}
We construct a synthetic linear MDP example based on a hard example proposed in Section \ref{sec:experiments} in \citep{duan2020minimax}, which was used to illustrate their lower bound. However, the feature dimension $d=2$ in their illustrative example is too small to show discrepancy between our algorithm and theirs. Therefore, we construct an example sharing a similar structure but of much larger feature dimension and size of action space. 

\paragraph{MDP instance.} In specific, our MDP instance contains $|\cS|=2$ states and $|\cA|=100$ actions, and the feature dimension is $d=10$. We denote $\cS=\{0,1\}$ and $\cA=\{0,1,\ldots,99\}$ respectively. For each action $a\in[99]$, we represent it by a binary encoding vector $\ab\in\RR^8$ with each entry being either $1$ or $-1$. With a slight abuse of notation, we interchangebly use $a$ and and its vector representation $\ab$. 

We define
\begin{align*}
    \delta(s,a) = \begin{cases}
    1 & \textnormal{if } \ind\{s=0\}=\ind\{a=0\},\\
    0 & \textnormal{otherwise.}
    \end{cases}
\end{align*}
Then the feature mapping is given by
\begin{align*}
    \bphi(s,a) = (\ab^\top, \delta(s,a), 1-\delta(s,a))^\top\in\RR^{10}.
\end{align*}
Let $\{\alpha_h\}_{h\in[H]}$ be a sequence of integers taking values in $\{0,1\}$. For each $s\in\cS$, the vector-valued measures are defined as
\begin{align*}
    \bmu_h(s)=(0,\ldots,0,(1-s)\oplus\alpha_h,s\oplus\alpha_h)
\end{align*}
for all $h\in[H]$, where $\oplus$ denotes the 'XOR' sign. Finally, we define $\bgamma_h\equiv\bgamma=(0,\ldots,0,1,0)\in\RR^{10}$. Thus the transition is $\PP_h(s'\mid s,a)=\langle\bphi(s,a),\bmu_h(s')\rangle$ and the expected reward is $r_h(s,a)=\langle\bphi(s,a),\bgamma\rangle$. It is straightforward to verify that this is a valid time-inhomogeneous linear MDP.

\paragraph{Behavior and target policy.} 
The target policy is given by $\pi(s) = 0$ for both $s=0,1$. The behavior policy is determined by a parameter $p\in (0,1)$: with probability $1-p$, the behavior policy chooses $a=0$, and with probability $(1-p)/99$ it chooses $a=i$ for each $i\in[99]$. This $p$ can be used to control the distribution shift between the behavior and target policies. 
Note that $p$ close to $0$ induces small distribution shift, while larger $p$ leads to large distribution shift. Moreover, we set the initial distribution $\xi_1$ to be uniform over $\cS$.

\begin{figure}[t]
	%\centering
	\begin{subfigure}{0.32\textwidth}
		\centering
		\includegraphics[width=\linewidth]{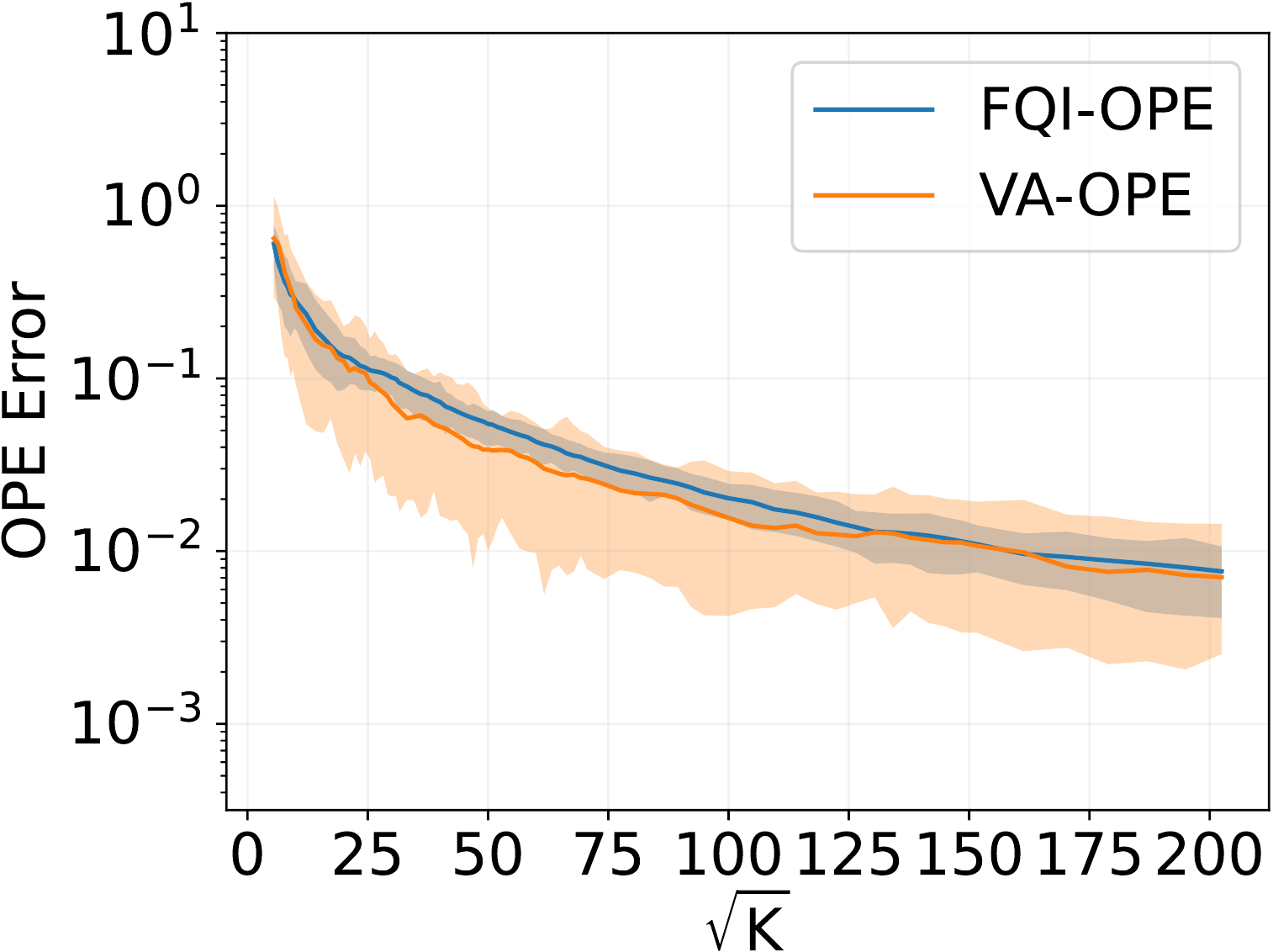}
		\caption{$H = 5$.}
		\label{fig: different h 1}
	\end{subfigure}
	\hfill
	\begin{subfigure}{0.32\textwidth}
		\centering
		\includegraphics[width=\linewidth]{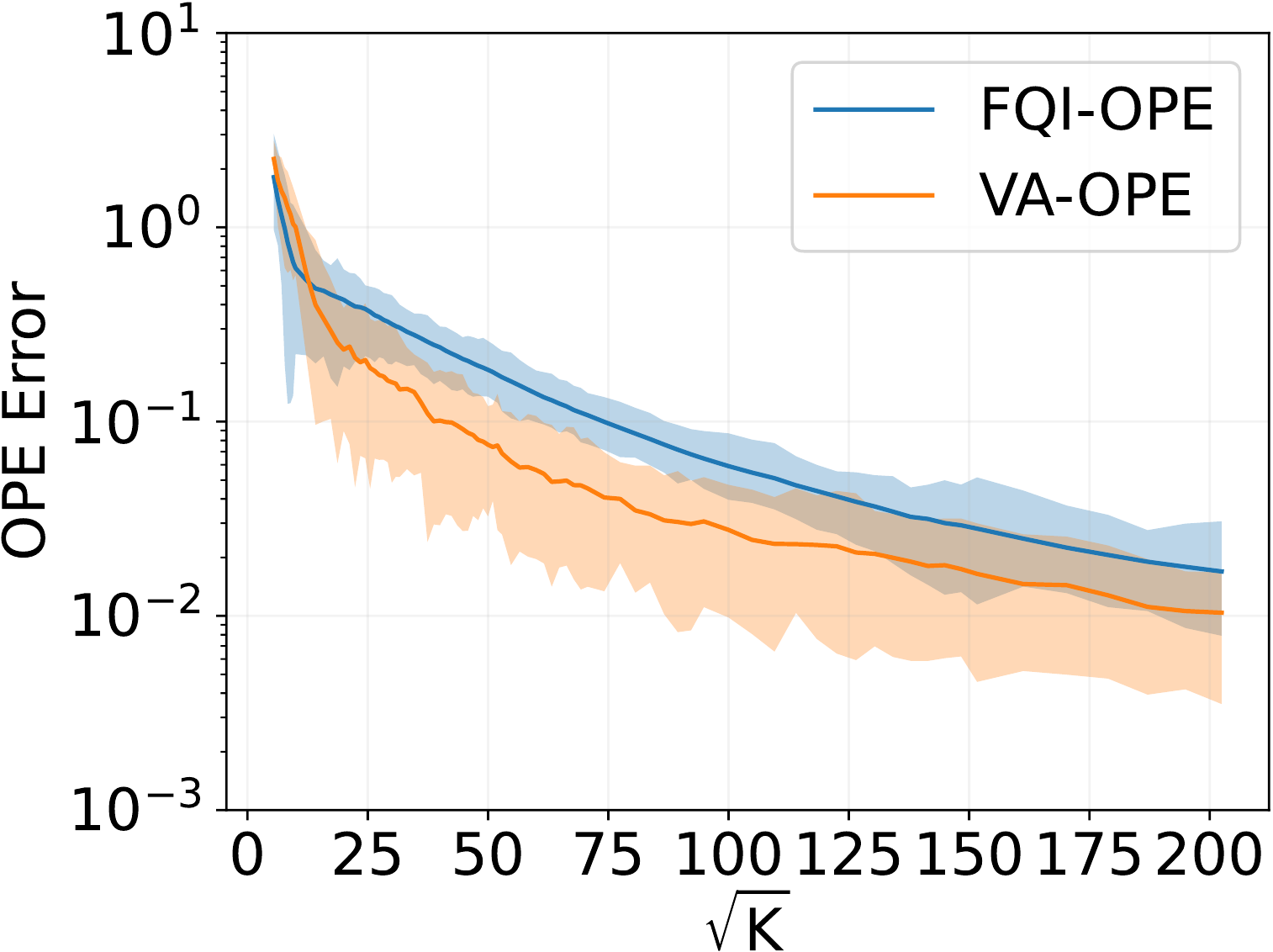}
		\caption{$H=10$.}
		\label{fig: different h 2}
	\end{subfigure}
	\hfill
	\begin{subfigure}{0.32\textwidth}
		\centering
		\includegraphics[width=\linewidth]{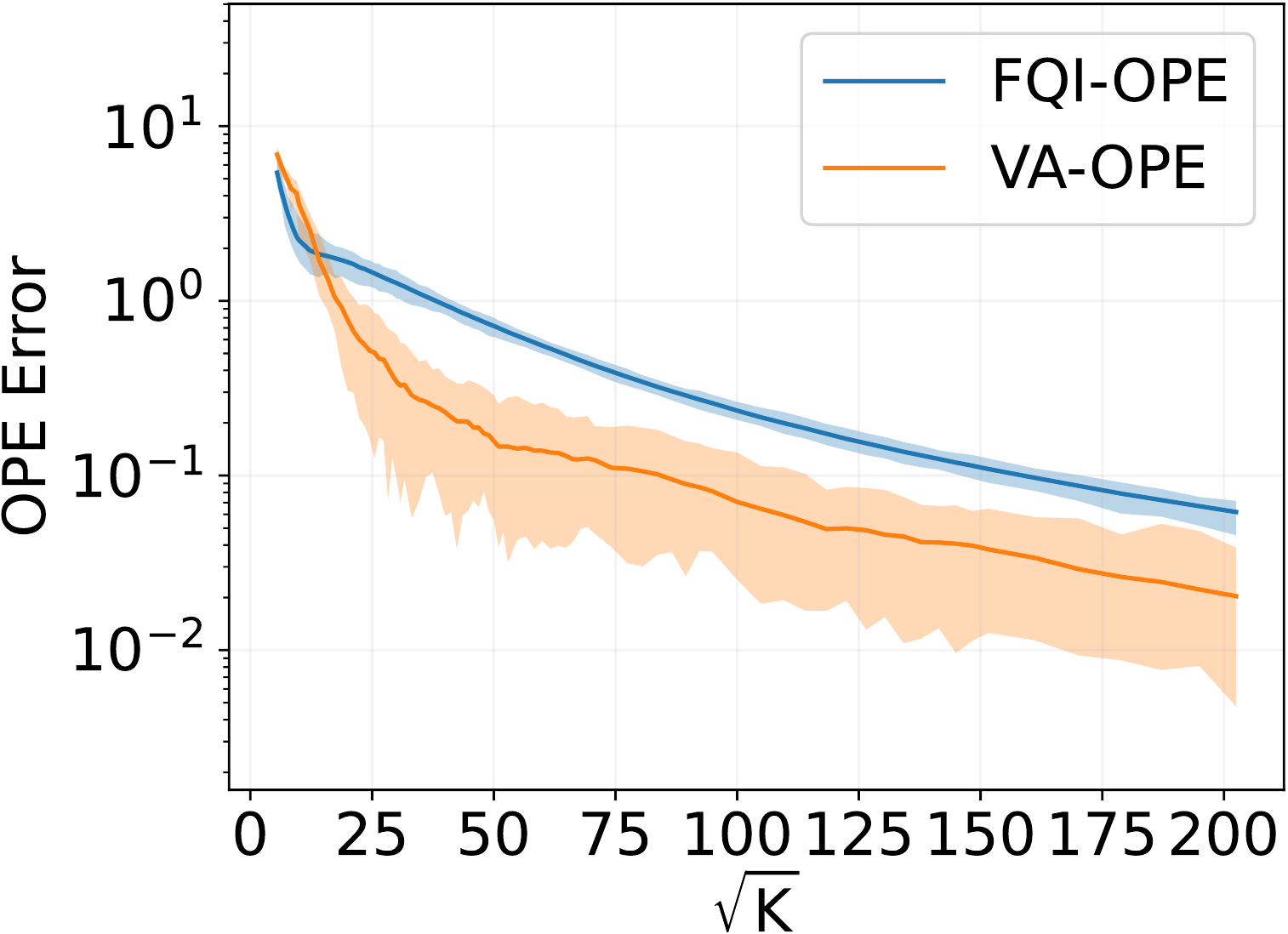}
		\caption{$H=20$.}
		\label{fig: different h 3}
	\end{subfigure}
	\medskip
	\begin{subfigure}{0.32\textwidth}
		\centering
		\includegraphics[width=\linewidth]{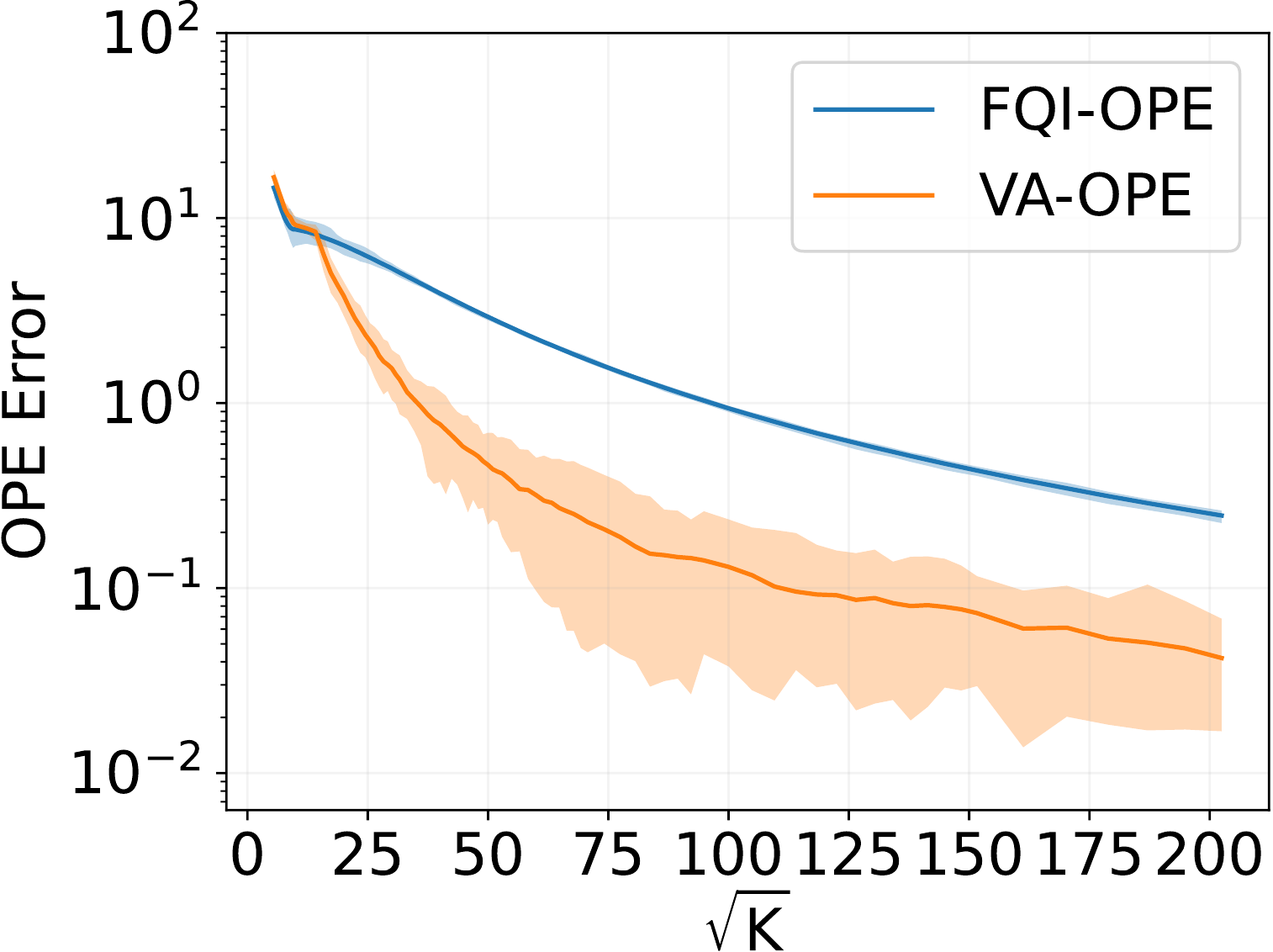}
		\caption{$H = 40$.}
		\label{fig: different h 4}
	\end{subfigure}
	\hfill
	\begin{subfigure}{0.32\textwidth}
		\centering
		\includegraphics[width=\linewidth]{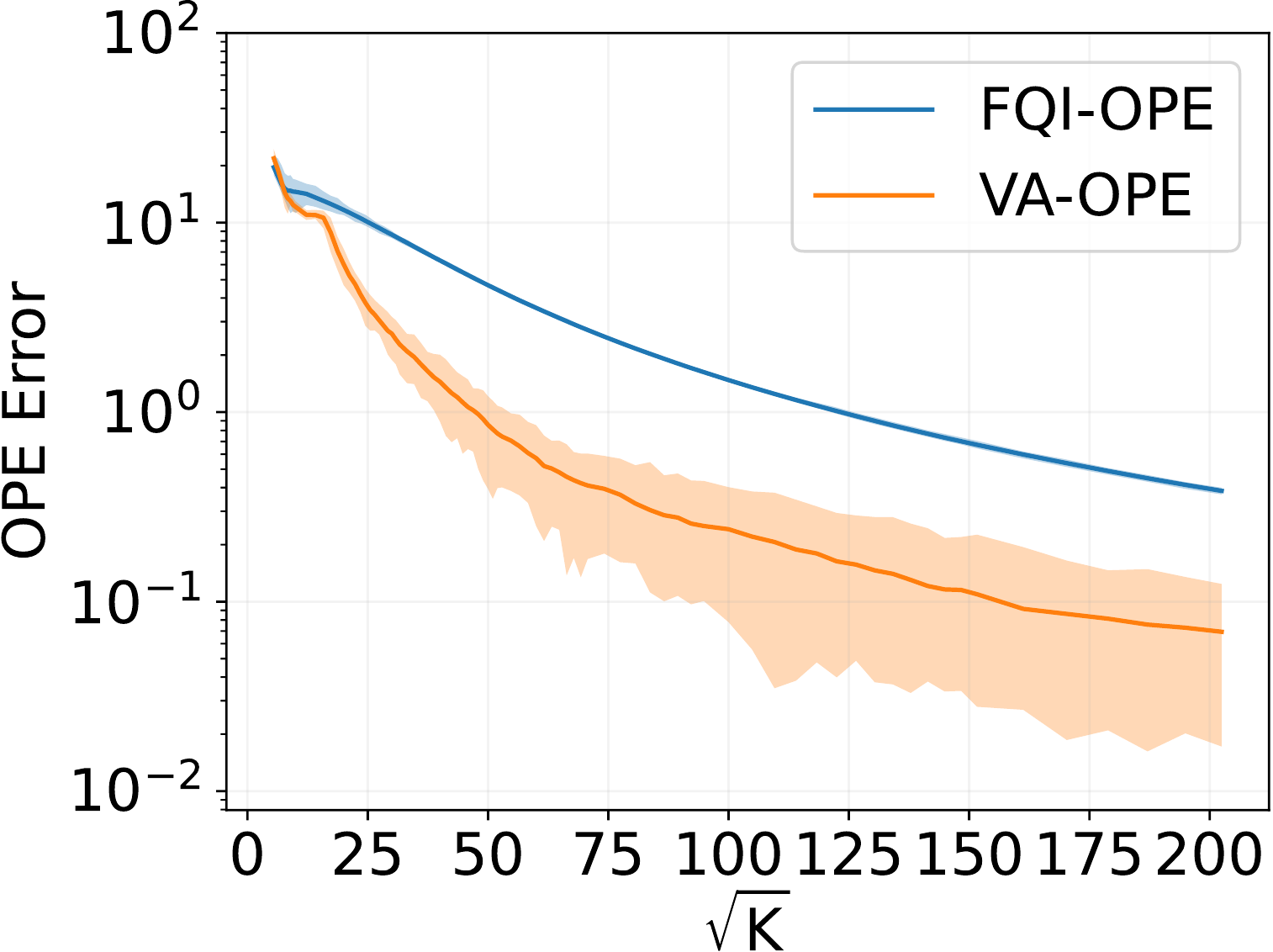}
		\caption{$H=50$.}
		\label{fig: different h 5}
	\end{subfigure}
	\hfill
	\begin{subfigure}{0.32\textwidth}
		\centering
		\includegraphics[width=\linewidth]{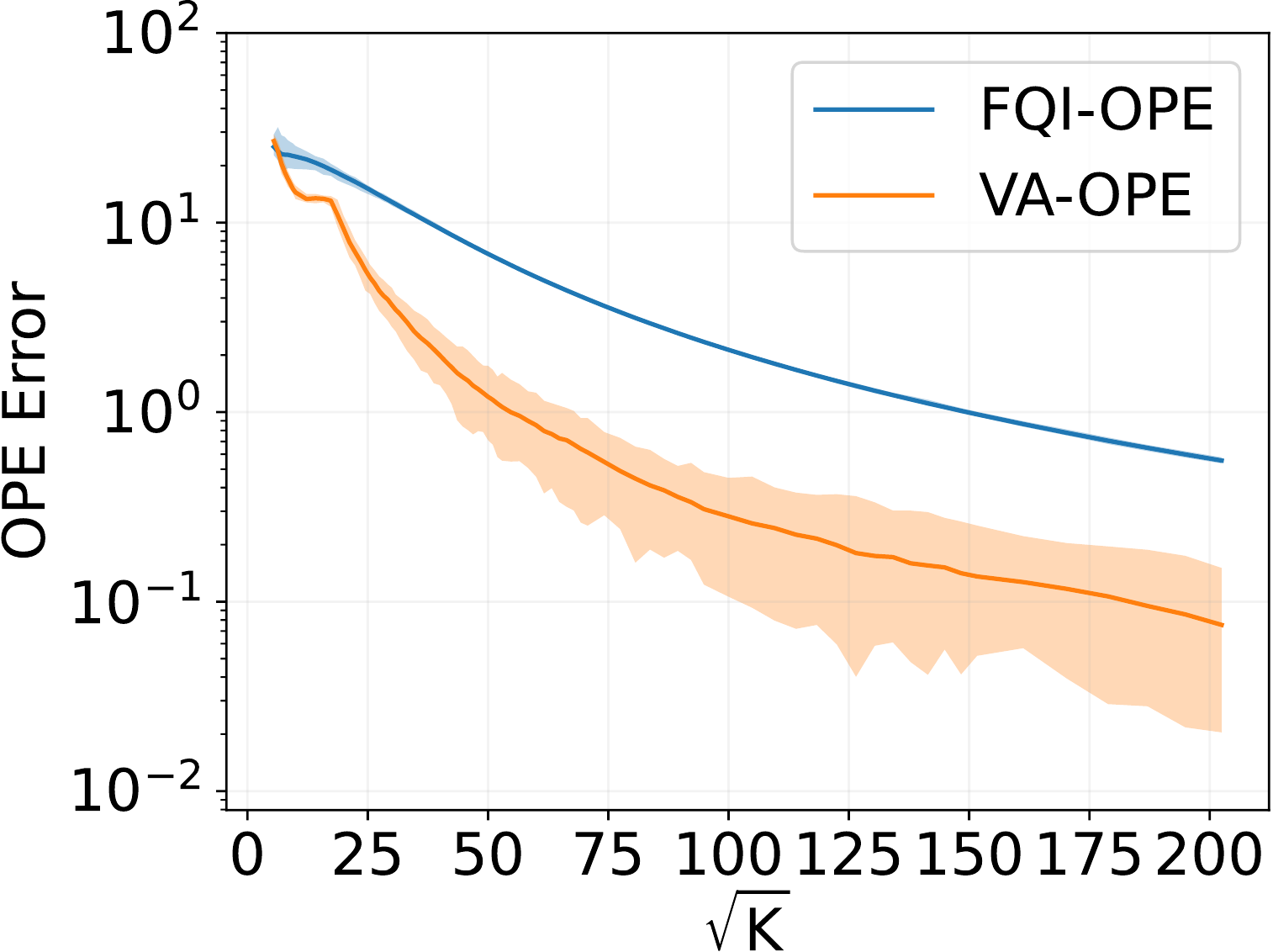}
		\caption{$H=60$.}
		\label{fig: different h 6}
	\end{subfigure}
	\caption{OPE error vs. $\sqrt{K}$. The results are averaged over 50 trials and the error bars are chosen to be the empirical [10\%, 90\%] confidence intervals. For a proper comparison, each sub-plot corresponds to a different setting of $H$ by keeping everything else the same: $|\cS|=2$, $|\cA|=100$, $p=0.6$. }\label{fig: different h}
\end{figure}

We remark that in our implementation of \texttt{VA-OPE} we do \emph{not} apply data splitting, i.e., $\cD = \check\cD$ and therefore no data is wasted. As is mentioned in the main text, the only purpose of the data splitting is to avoid an otherwise lengthy theoretical analysis. Therefore, for each fixed $K$, both algorithms use a dataset of size $K$ sampled under the behavior policy.

\subsection{Impact of the Planning Horizon}
We first study the impact of the planning horizon $H$ on the performance. We run our algorithm \texttt{VA-OPE} and the baseline method \texttt{FQI-OPE} with $\lambda = 1$ on the linear MDP instance constructed in the previous subsection under different values of $H$. We fix the initial distribution to be $\xi_1=[1/2, 1/2]$ and $p$ to be $0.6$. 
The results are reported in Figure \ref{fig: different h}. 

To explain the results, let us first recall the dominant term in our error bound and that in \citep{duan2020minimax} (ignoring the logarithmic and constant factors):
\begin{align}\label{eq: dominant terms}
    D_{\texttt{VA}} = \frac{\sum_{h=1}^H\|\vb_h\|_{\bLambda_h^{-1}}}{\sqrt{K}}\qquad \text{vs}\qquad D_{\texttt{FQI}} = \frac{\sum_{h=1}^H (H-h+1)\|\vb_h^\pi\|_{\bSigma_h^{-1}}}{\sqrt{K}}.
\end{align}
As mentioned in the discussion following Theorem \ref{thm:ope}, it holds that $D_{\texttt{VA}}\leq D_{\texttt{FQI}}$. Indeed, this is reflected by the error plots where the error of \texttt{VA-OPE} is smaller than that of \texttt{FQI-OPE} except for very small $K$.

Moreover, as careful readers may have already observed, the discrepancy between $D_{\texttt{VA}}$ and $D_{\texttt{FQI}}$ would be amplified as the value of $H$ increases. Again, our simulation results confirm this theoretical observation as we can see by comparing the subplots of Figure \ref{fig: different h}. For larger values of $H$, \texttt{VA-OPE} tends to enjoy a much faster convergence rate. We would like to emphasize that this performance gain is especially beneficial for long-horizon tasks.

These findings also shed light on the minimax optimality of the OPE problem. The previous \texttt{FQI-OPE} algorithm is nearly minimax optimal only for a subclass of linear MDPs where $\VV_h V_{h+1}^\pi = \Theta((H-h)^2)$. As suggested by our theory and confirmed by the numerical experiments, our algorithm \texttt{VA-OPE} achieves a tighter instance-dependent error for general linear MDPs. We would like to establish the universal minimax lower bound in the future work, and we believe that \texttt{VA-OPE} is a promising candidate for achieving minimax optimality.

We would also like to remark that the width of the error bars of \texttt{VA-OPE} is similar to that of \texttt{FQI-OPE}. It only appears wider on the plots since the y-axis is $\log_{10}$-scaled.  

\begin{figure}
	%\centering
	\begin{subfigure}{0.32\textwidth}
		\centering
		\includegraphics[width=\linewidth]{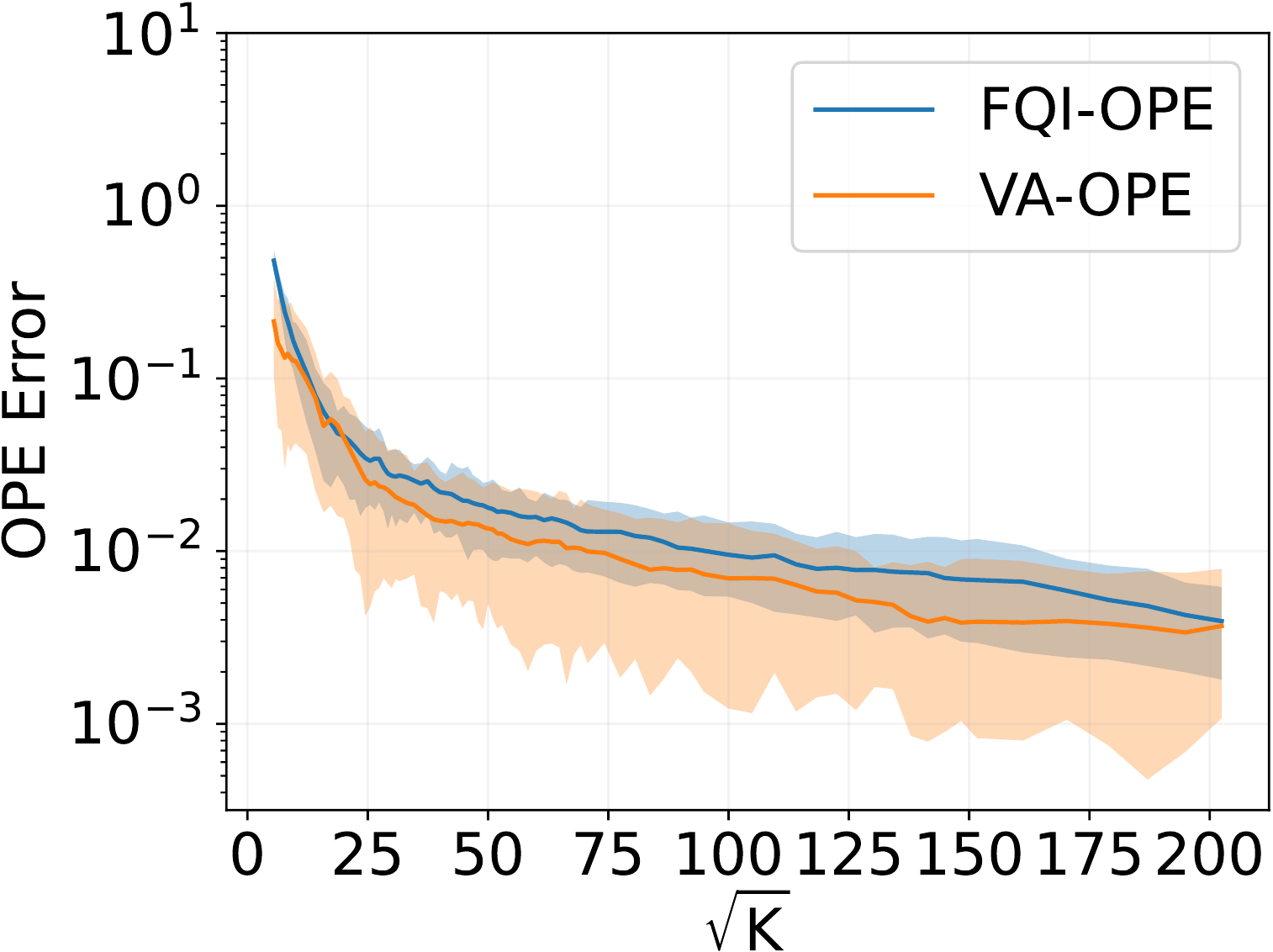}
		\caption{$H = 5, \ p = 0.2$.}
		\label{fig: ds 1}
	\end{subfigure}
	\hfill
	\begin{subfigure}{0.32\textwidth}
		\centering
		\includegraphics[width=\linewidth]{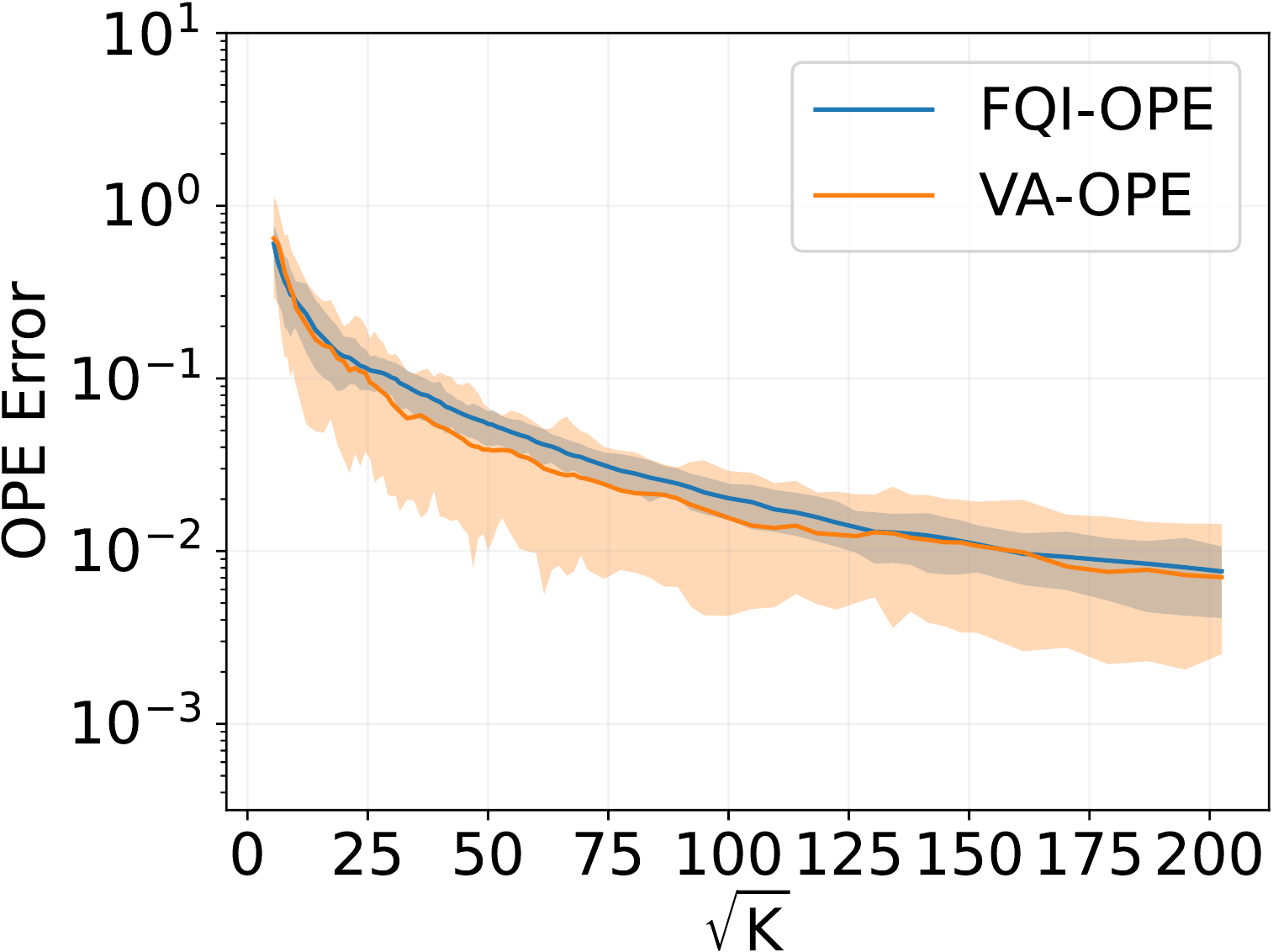}
		\caption{$H=5, \ p = 0.7$.}
		\label{fig: ds 2}
	\end{subfigure}
	\hfill
	\begin{subfigure}{0.32\textwidth}
		\centering
		\includegraphics[width=\linewidth]{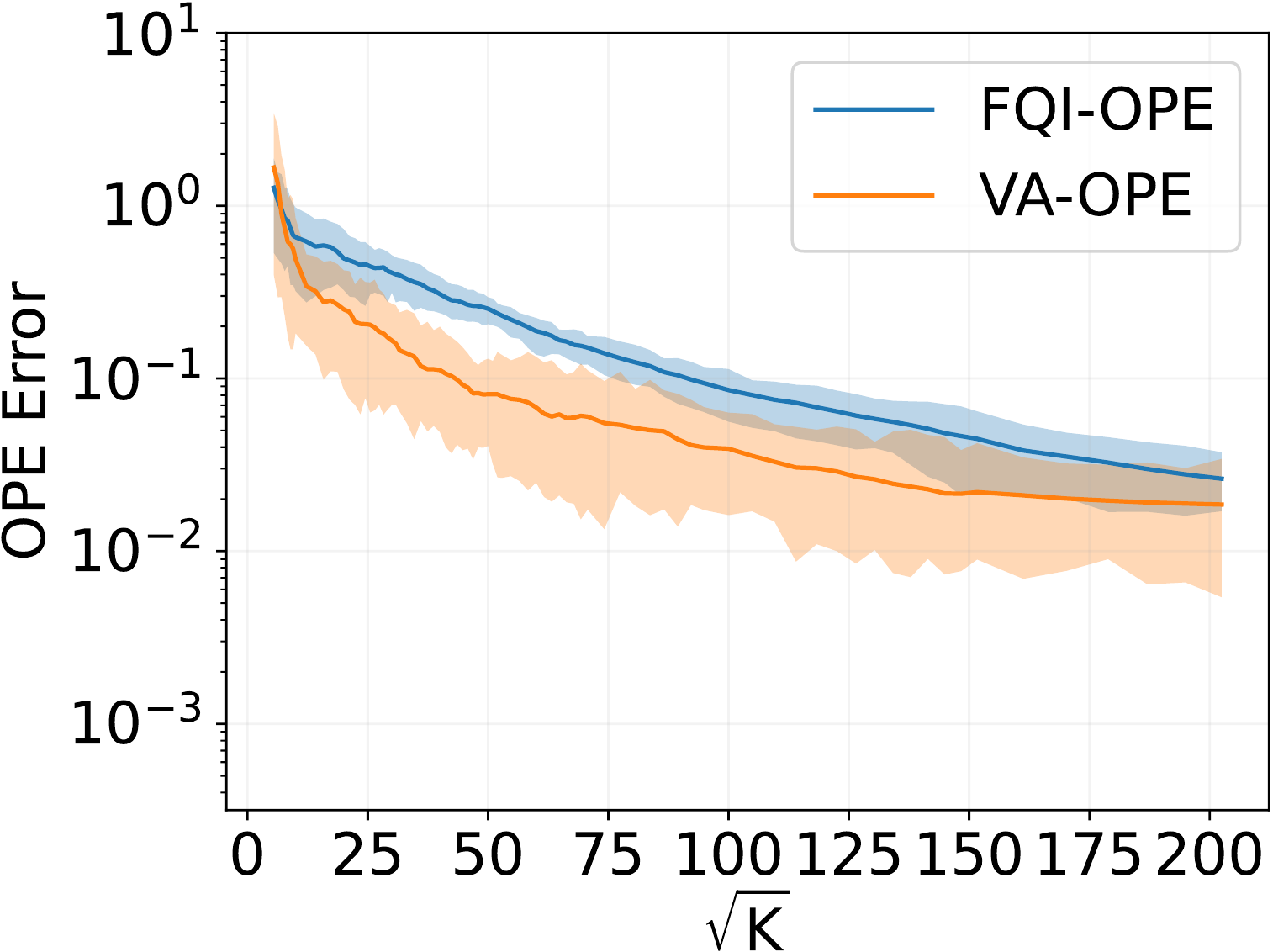}
		\caption{$H=5, \ p = 0.9$.}
		\label{fig: ds 3}
	\end{subfigure}
	\medskip
	\begin{subfigure}{0.32\textwidth}
		\centering
		\includegraphics[width=\linewidth]{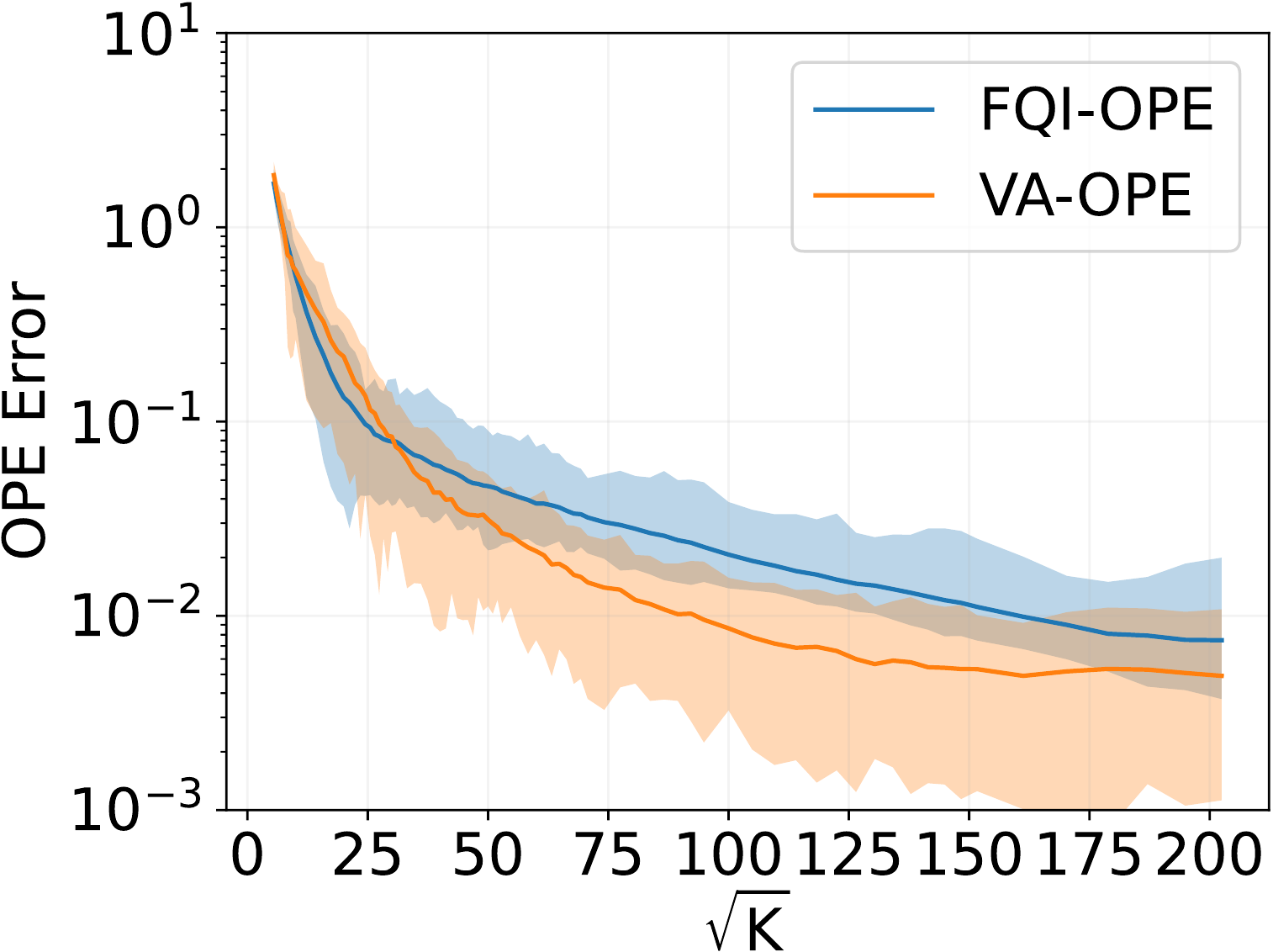}
		\caption{$H = 10, \ p = 0.2$.}
		\label{fig: ds 4}
	\end{subfigure}
	\hfill
	\begin{subfigure}{0.32\textwidth}
		\centering
		\includegraphics[width=\linewidth]{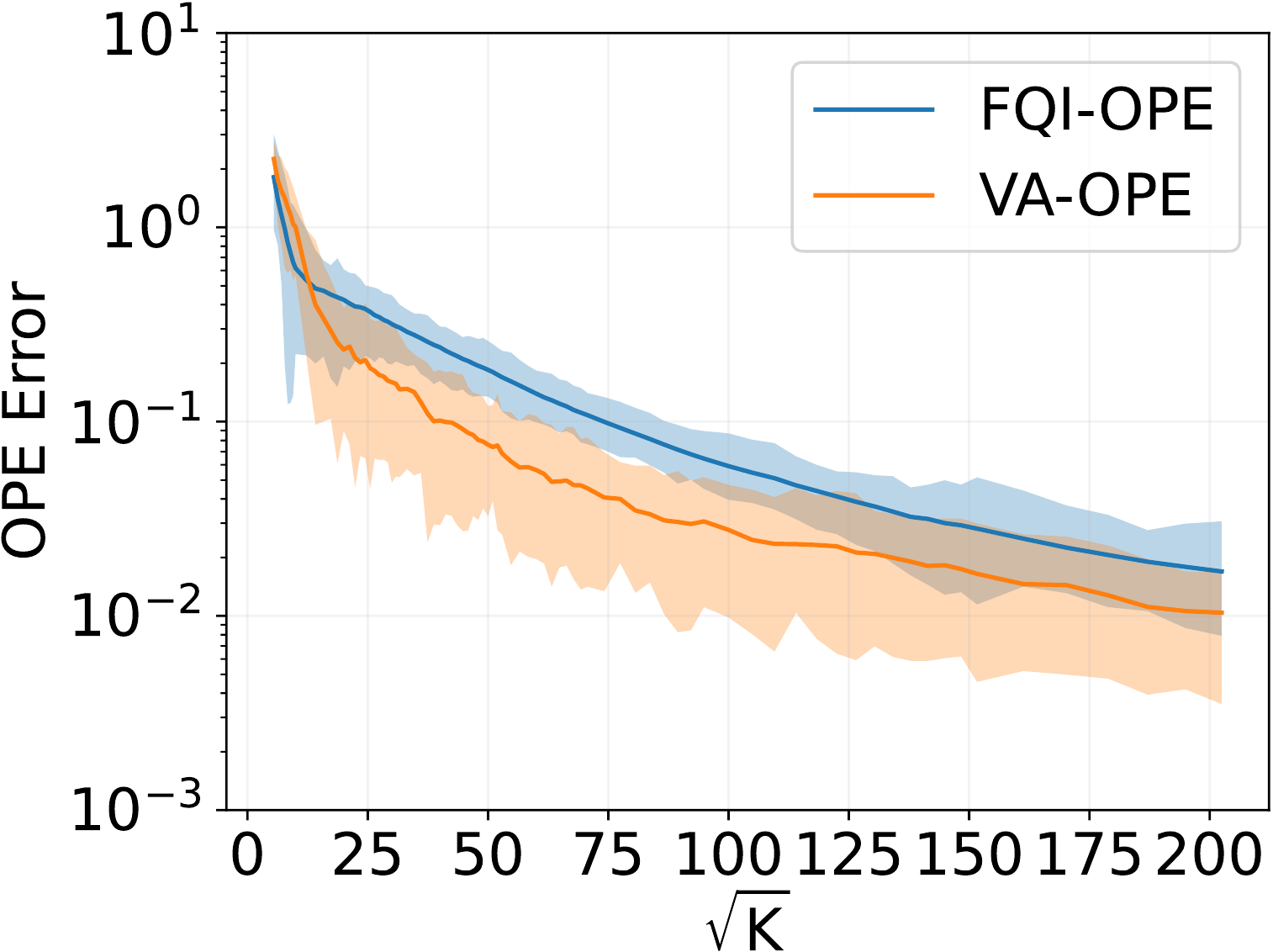}
		\caption{$H=10, \ p = 0.7$.}
		\label{fig: ds 5}
	\end{subfigure}
	\hfill
	\begin{subfigure}{0.32\textwidth}
		\centering
		\includegraphics[width=\linewidth]{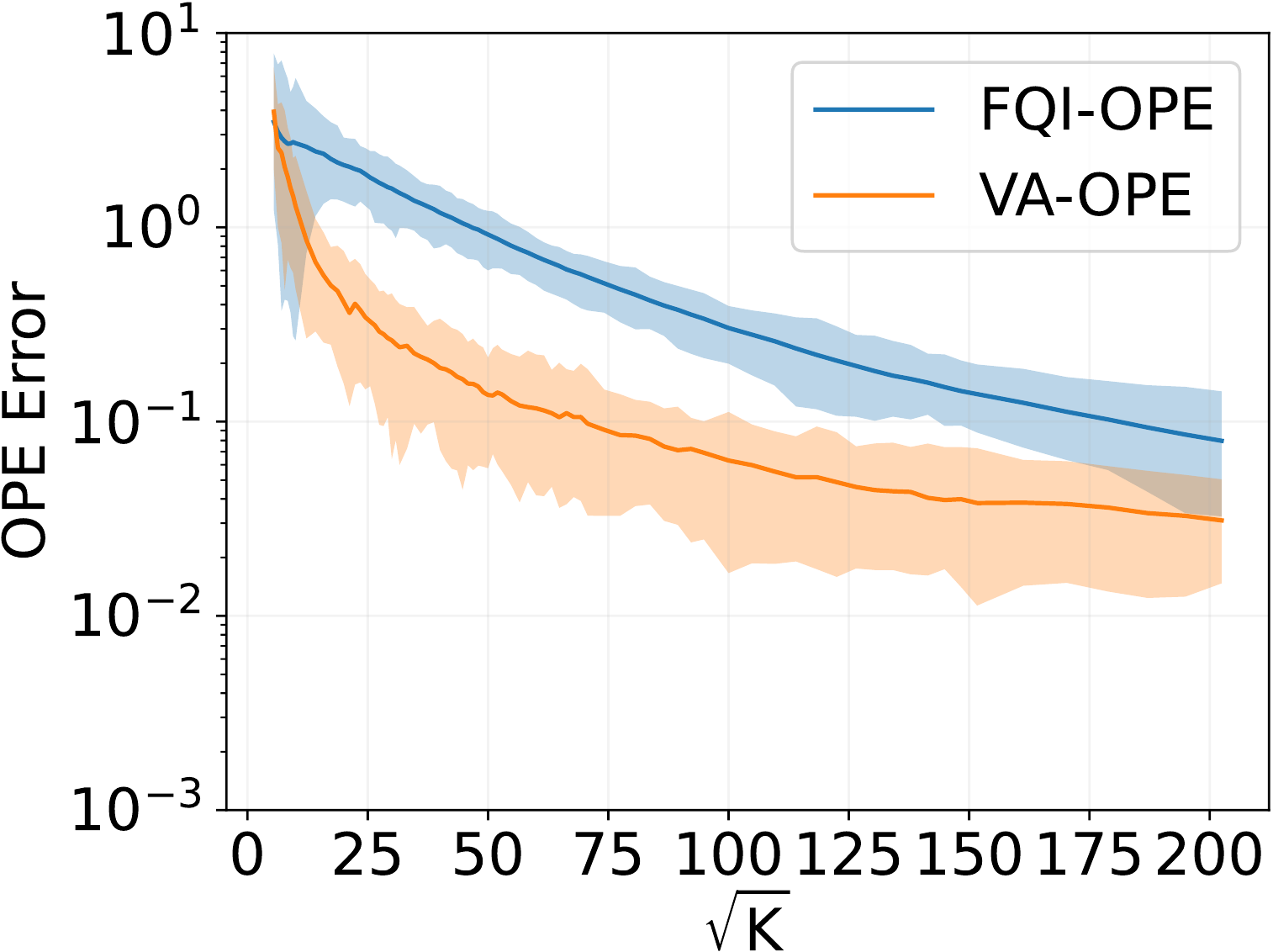}
		\caption{$H=10, \ p = 0.9$.}
		\label{fig: ds 6}
	\end{subfigure}
	\medskip
	\begin{subfigure}{0.32\textwidth}
		\centering
		\includegraphics[width=\linewidth]{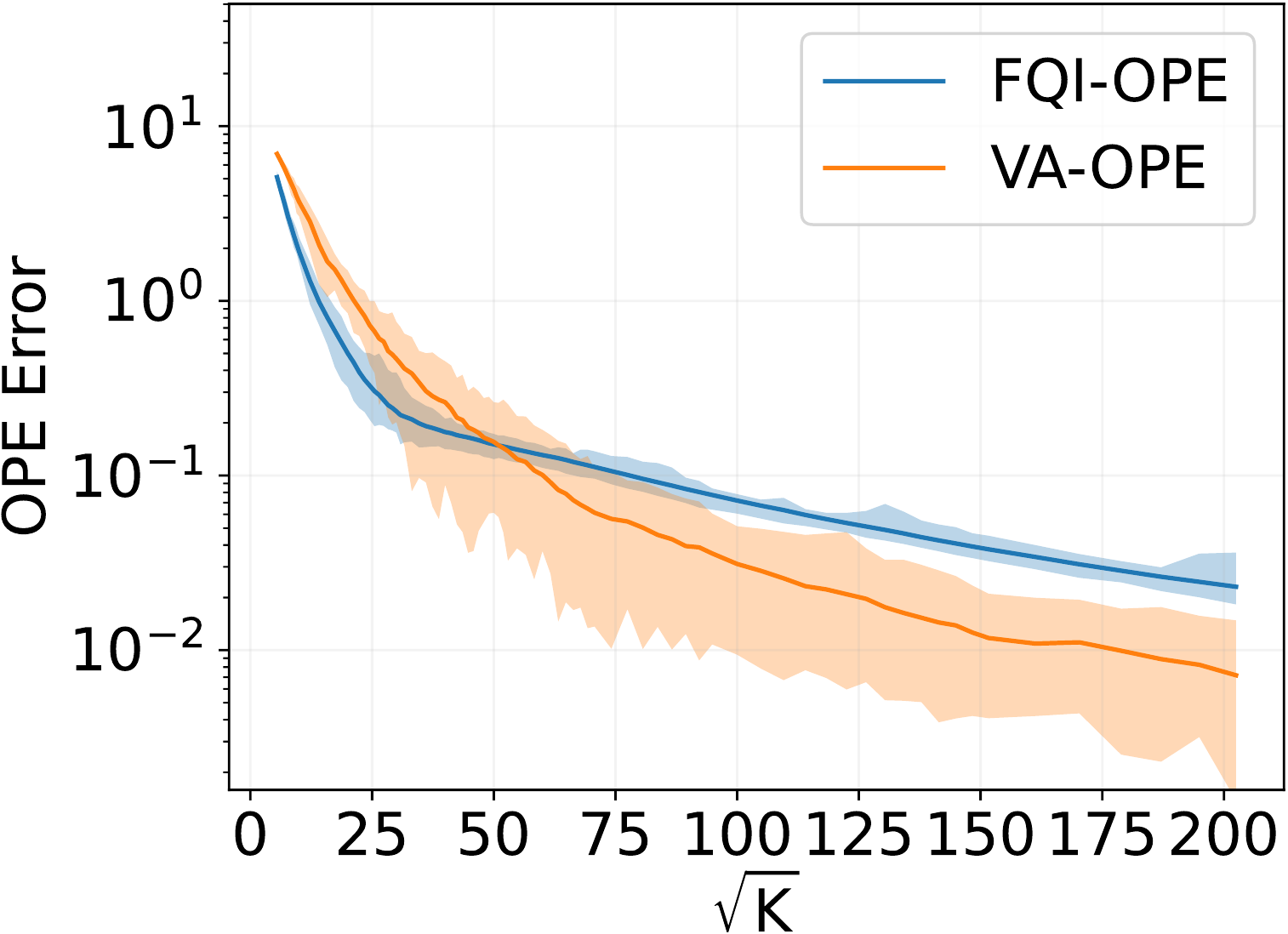}
		\caption{$H = 20, \ p = 0.2$.}
		\label{fig: ds 7}
	\end{subfigure}
	\hfill
	\begin{subfigure}{0.32\textwidth}
		\centering
		\includegraphics[width=\linewidth]{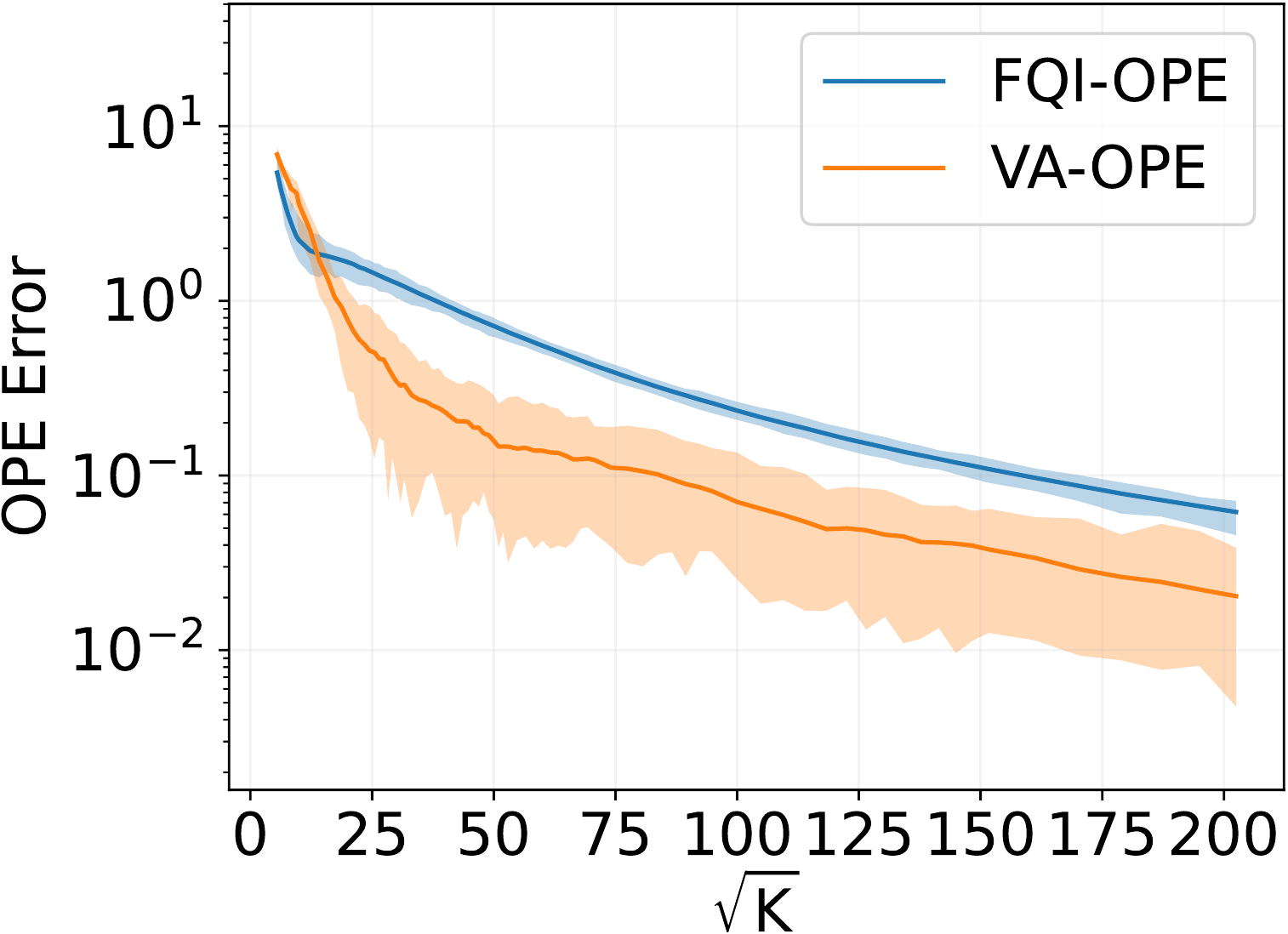}
		\caption{$H = 20, \ p = 0.7$.}
		\label{fig: ds 8}
	\end{subfigure}
	\hfill
	\begin{subfigure}{0.32\textwidth}
		\centering
		\includegraphics[width=\linewidth]{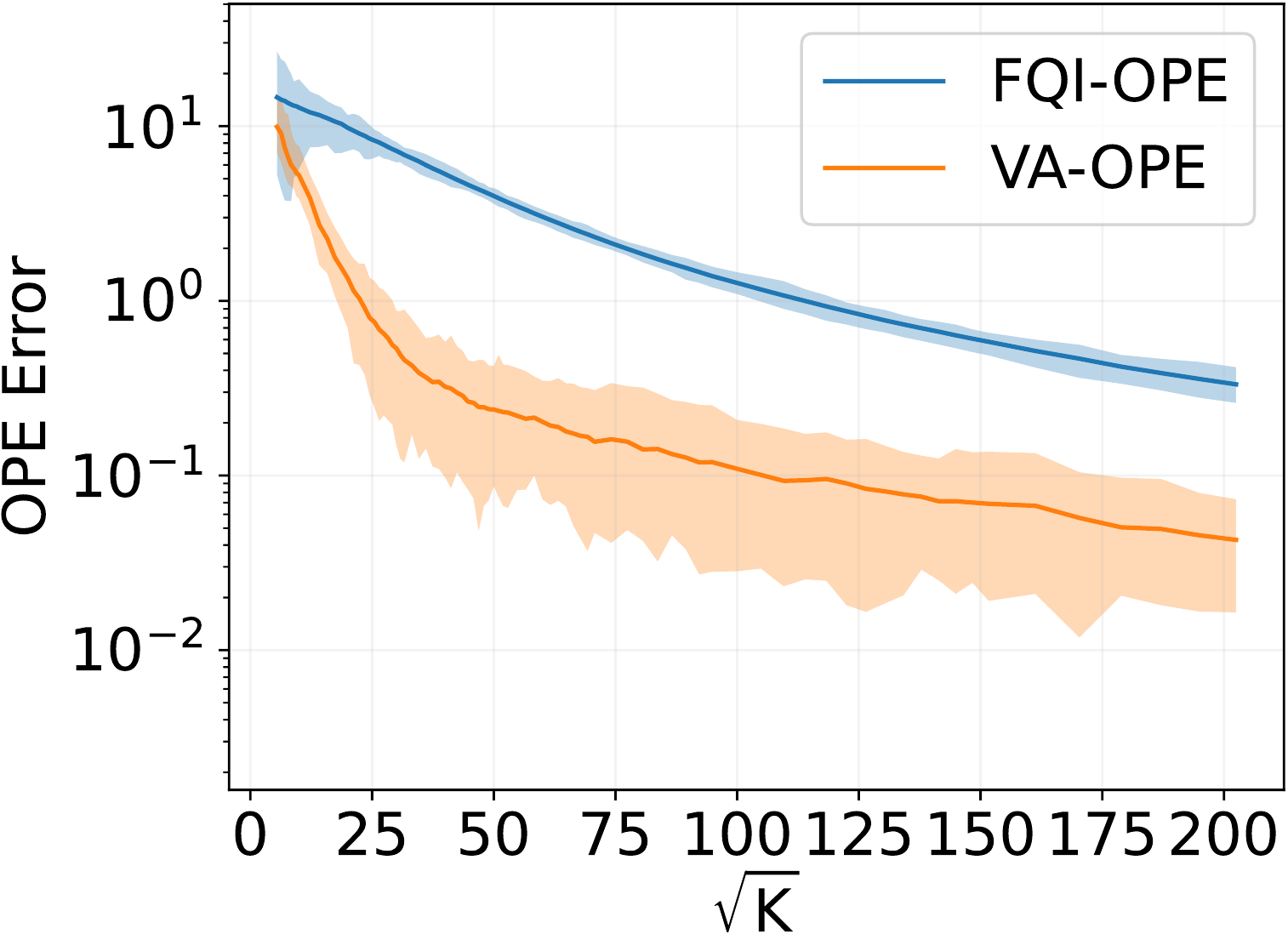}
		\caption{$H = 20, \ p = 0.9$.}
		\label{fig: ds 9}
	\end{subfigure}
	\medskip
	\begin{subfigure}{0.32\textwidth}
		\centering
		\includegraphics[width=\linewidth]{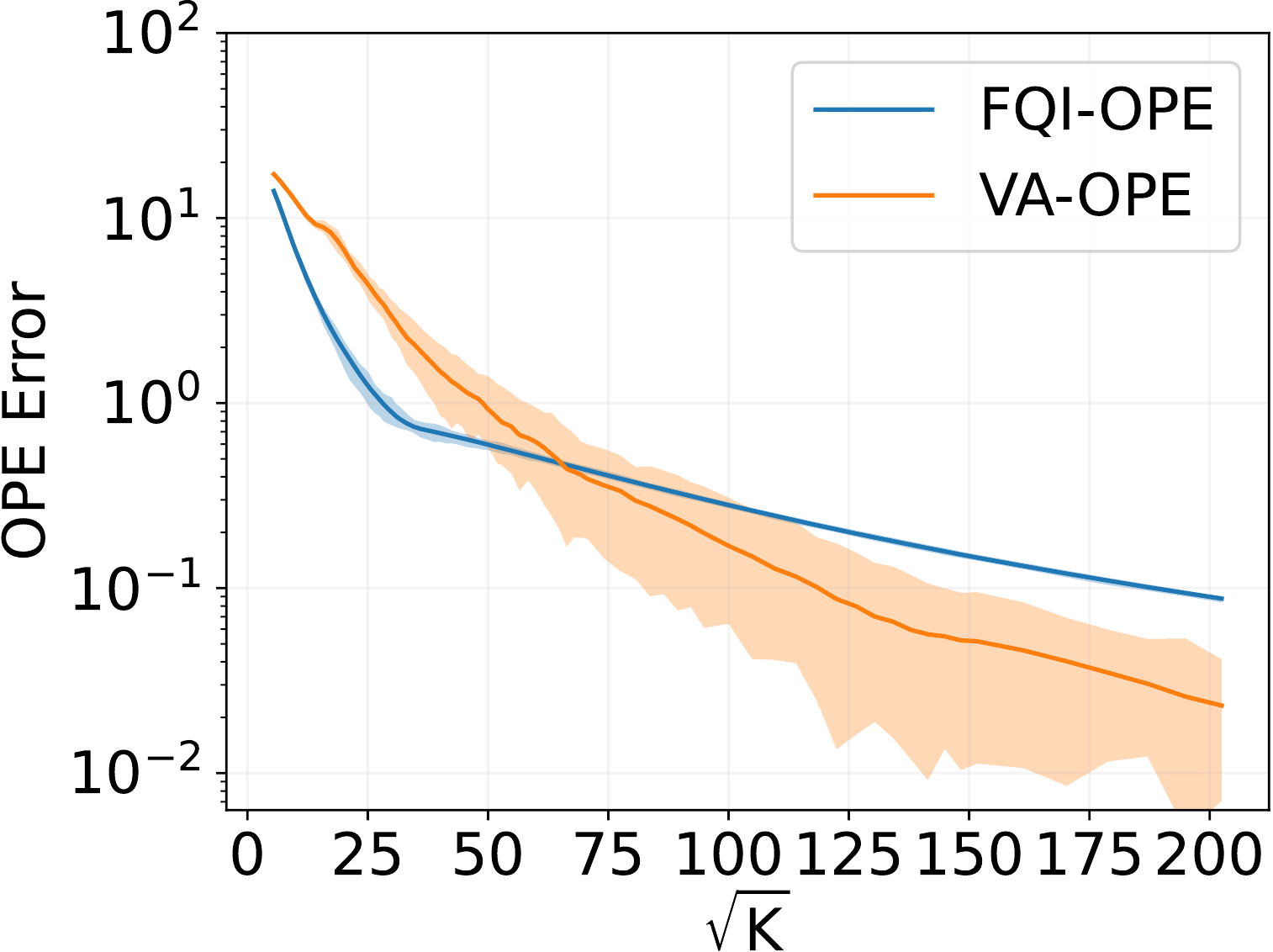}
		\caption{$H = 40, \ p = 0.2$.}
		\label{fig: ds 10}
	\end{subfigure}
	\hfill
	\begin{subfigure}{0.32\textwidth}
		\centering
		\includegraphics[width=\linewidth]{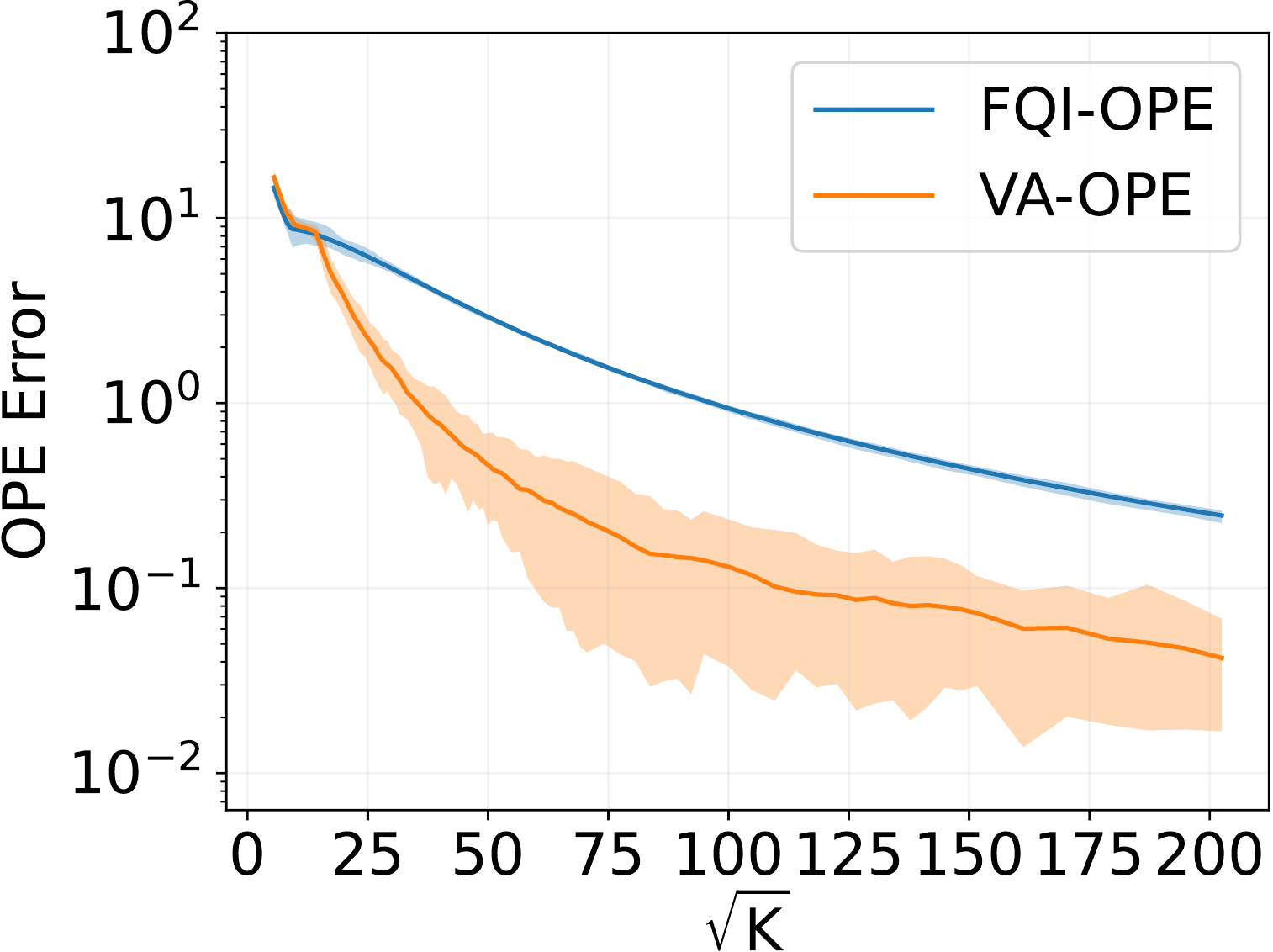}
		\caption{$H = 40, \ p = 0.7$.}
		\label{fig: ds 11}
	\end{subfigure}
	\hfill
	\begin{subfigure}{0.32\textwidth}
		\centering
		\includegraphics[width=\linewidth]{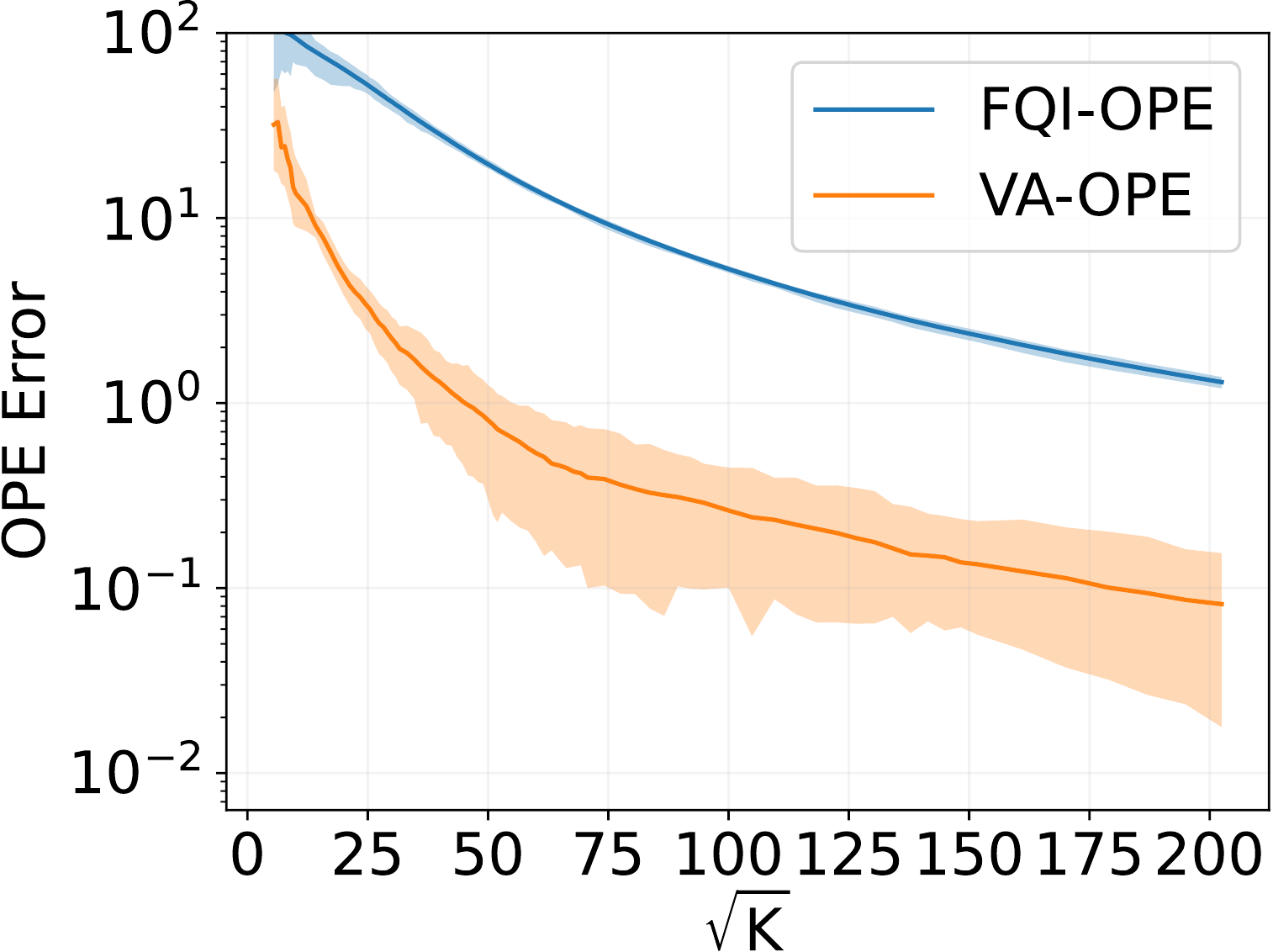}
		\caption{$H = 40, \ p = 0.9$.}
		\label{fig: ds 12}
	\end{subfigure}	
	\caption{Log-scaled OPE error vs. $\sqrt{K}$ under different levels of distribution shift and horizon $H$. The level of distribution shift is controlled by the parameter $p$, where larger $p$ corresponds to larger distribution shift. }\label{fig: ds}
\end{figure}

% Other comments are in order: (i) From the error bars it may seem that \texttt{VA-OPE} has significantly larger volatility than \texttt{FQI-OPE}. But note that this is not true since the y-axis is log-scaled, and thus the volatility of \texttt{VA-OPE} is slightly bigger but still at roughly the same level as that of \texttt{FQI-OPE}. 
% This is no surprise since \texttt{VA-OPE} has one extra source of uncertainty from the variance estimation. This also explains why sometimes \texttt{VA-OPE}'s curve is above the curve of \texttt{FQI-OPE} in the extreme case of very small $K$: the variance estimator would be far from the truth if $K$ is too small, restricting the performance of \texttt{VA-OPE}. Except for the extreme case, \texttt{VA-OPE} dominates. 
% And more importantly, the upper confidence bound of \texttt{VA-OPE} is almost always below the average error curve of \texttt{FQI-OPE}. 
% (ii) It might seem that the error bar of \texttt{VA-OPE} is expanding as $\sqrt{K}$ increases, which again is due to y-axis being log-scaled. 

\subsection{Impact of Distribution Shift}
We also illustrate the impact of distribution shift between the behavior policy and the target policy on the performance, which can be controlled by the value of $p$. In Figure \ref{fig: ds}\footnote{Note that the range of the y-axis differs among different rows.}, we compare the performance of \texttt{VA-OPE} and \texttt{FQI-OPE} under different values of $p$.

The subplots in the same row share the same value of $H$. It is clear that for larger distribution shift, the performance of \texttt{VA-OPE} is superior. The reason behind this is that for fixed $H$, the ratio $D_{\texttt{FQI}}/D_{\texttt{VA}}$ increases as $p$ increases. We further investigate this in the next subsection.

\subsection{Comparison of the Dominant Terms}
Finally we compare the dominant terms in the error upper bound of \texttt{VA-OPE} and \texttt{FQI-OPE} as defined in \eqref{eq: dominant terms}. Since both $D_{\texttt{VA}}$ and $D_{\texttt{FQI}}$ are theoretical values as the expectation over the occupancy measure induced by the transition kernel and the behavior/target policy, we simply estimate them by averaging over 1,000,000 independent trajectories. As presented in Figure \ref{fig: dominant term ratio}, our characterization of the distribution shift, $\sum_{h=1}^H \|\vb_h\|_{\bLambda_h^{-1}}$, is tighter. This is the main reason for the performance discrepancy that we have seen in the preceding subsections. 

\begin{figure}
	\centering
	\begin{subfigure}{0.32\textwidth}
		\centering
		\includegraphics[width=\linewidth]{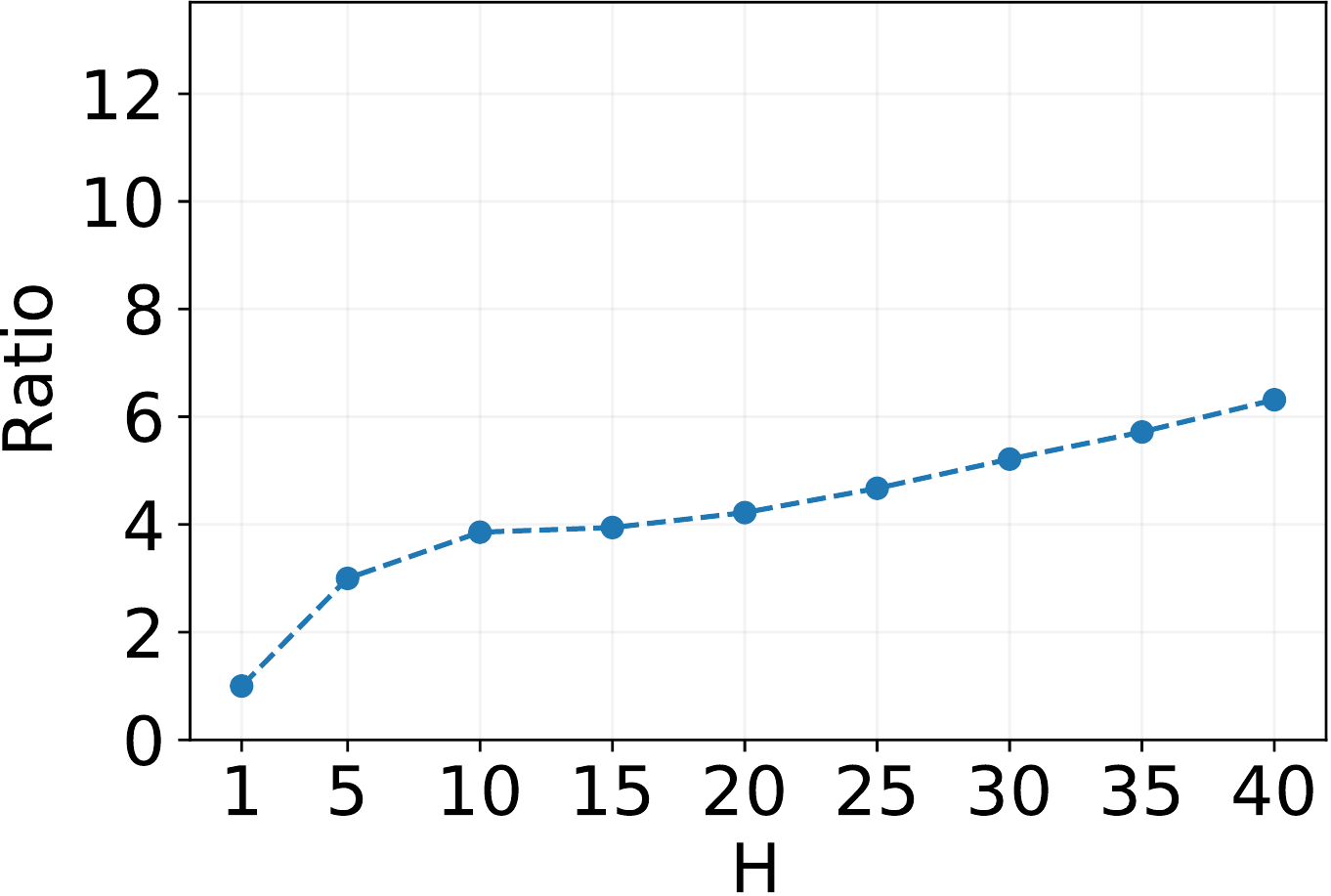}
		\caption{$p = 0.3$.}
		\label{fig: ratio 1}
	\end{subfigure}
	\hfill
	\begin{subfigure}{0.32\textwidth}
		\centering
		\includegraphics[width=\linewidth]{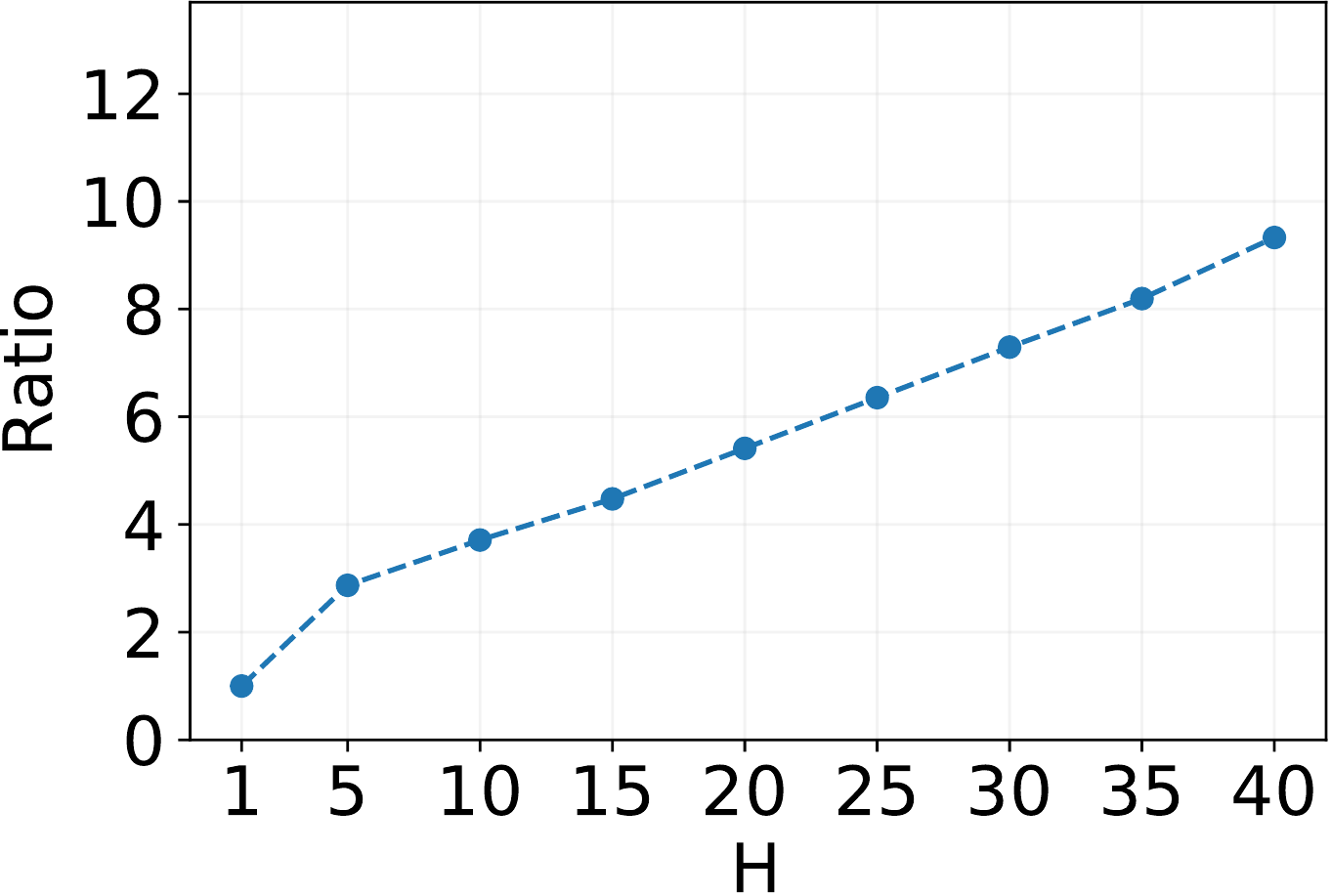}
		\caption{$p = 0.6$.}
		\label{fig: ratio 2}
	\end{subfigure}
	\hfill
	\begin{subfigure}{0.32\textwidth}
		\centering
		\includegraphics[width=\linewidth]{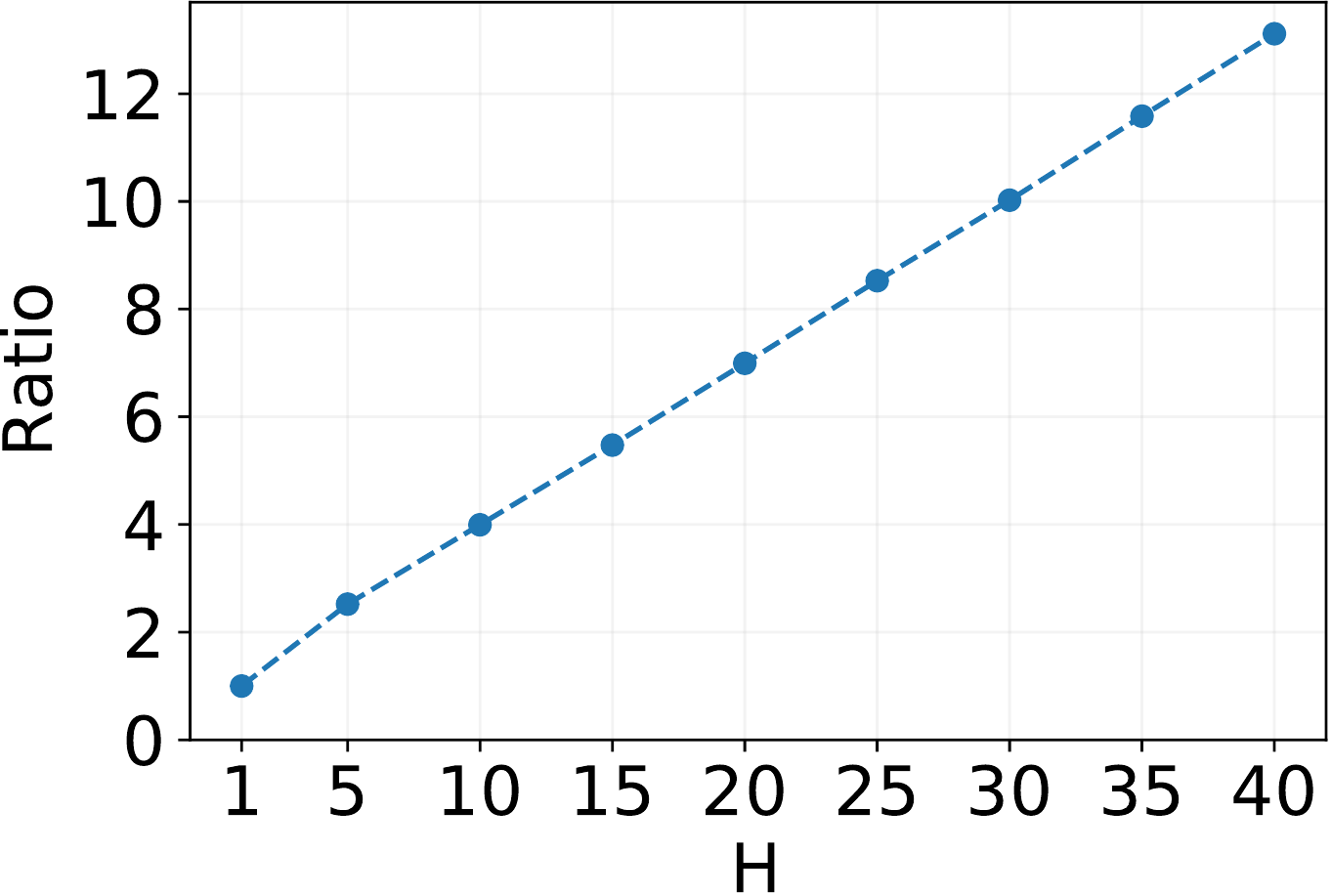}
		\caption{$p = 0.9$.}
		\label{fig: ratio 3}
	\end{subfigure}
	\caption{Ratio between dominant terms vs. $H$. The results are generated by averaging over 1,000,000 trajectories. }\label{fig: dominant term ratio}
\end{figure}

\subsection{Hardware Details}

All experiments are performed on an internal cluster with CPU and 30 GB of memory.

\section{Further Comparison with \citet{duan2020minimax}}
\label{sec:comparison with Duan}
%\subsection{Detailed Comparison with \citep{duan2020minimax}}

Consider the dominant term in the OPE error (omiting the logarithmic coefficients) given by Theorem 2 in \citet{duan2020minimax}, which was shown to be $\sum_{h=1}^H(H-h+1)\|\vb_h^\pi\|_{\bSigma_h^{-1}}/\sqrt{K}$ from their proof. As comparison, recall that our dominant term is about $\sum_{h=1}^H\|\vb_h^\pi\|_{\bLambda_h^{-1}}/\sqrt{K}$. The definition of $\bSigma_h$ in \eqref{def: Sigma} and that of $\bLambda_h$ in \eqref{def: Lambda} immediately imply $\bSigma_h\preceq [(H-h+1)^2+1]\cdot \bLambda_h$ as $\sigma_h^2$ is bounded above by $(H-h+1)^2+1$. Therefore, it holds that
\begin{align}\label{eq: comparison 1}
    \sum_{h=1}^H(H-h+1)\|\vb_h^\pi\|_{\bSigma_h^{-1}} & \geq  \sum_{h=1}^H \frac{(H-h+1)\|\vb_h^\pi\|_{\bLambda_h^{-1}}}{\sqrt{(H-h+1)^2+1}}.
\end{align}
The RHS of \eqref{eq: comparison 1} is close to $\sum_{h=1}^H \|\vb_h^\pi\|_{\bLambda_h^{-1}}$ if $H$ is large. Moreover, when $\VV_hV_{h+1}^\pi$ is small, the RHS of \eqref{eq: comparison 1} can be much smaller than the LHS with appropriate choice of $\eta_h$. In other words, our bound is tighter than that of \citet{duan2020minimax} in all scenarios, especially when $\VV_hV_{h+1}^\pi$ is small. 

Consider, for example, a scenario where the conditional variance of $V_h^\pi$, $h\in[H]$ is less than $H-h+1$, which is smaller than the crude upper bound of $(H-h+1)^2$ by a factor of $(H-h+1)$. Then by choosing $\eta_h = 1$ and $\sigmanoise=1$, we would have $\sigma_h^2\equiv H-h+2$, and 
\begin{align*}
        \frac{(H-h+1)\|\vb_h^\pi\|_{\bSigma_h^{-1}}}{\|\vb_h^\pi\|_{\bLambda_h^{-1}}} = \frac{H-h+1}{\sqrt{H-h+2}} ,    
\end{align*}
which suggests that $\|\vb_h^\pi\|_{\bLambda_h^{-1}}$ is smaller than its counterpart by a factor of $(H-h+1)/\sqrt{H-h+2}$. Also, as mentioned in the main text, the conditional variance of $V_h^\pi$ does \textbf{not} need to be uniformly smaller than $H-h+1$ for all $(s,a)\in \cS\times \cA$. It only needs to be small on average. 

Regarding their lower bound (Theorem 3), it only holds for a subclass of all MDP instances where the conditional variance $\VV_hV_{h+1}^\pi$ is on the order of $(H-h+1)^2$. Indeed, the theorem assumes there exists a high-value subset of states $\overline{\cS}$ and a low-value subset of states $\underline{\cS}$ under the target policy $\pi$ such that $V_{h+1}^\pi(s) \geq \frac{3}{4}(H-h+1)$ if $s\in\overline{\cS}$ and $V_{h+1}^\pi(s) \leq \frac{1}{4}(H-h+1)$ if $s \in \underline{\cS}$. They also require there is non-zero probability $\overline{p} \geq c >0$ and $\underline{p} \geq c>0$ of transitting into $\overline{\cS}$ and $\underline{\cS}$ respectively. These assumptions immediately imply $\VV_hV_{h+1}^\pi = \Omega((H-h+1)^2)$. Therefore, the prior result is only (nearly) minimax for a very small class of MDPs. This is confirmed by our numerical experiments in Appendix \ref{sec: simulation detail} where we compare the OPE error of \texttt{VA-OPE} and \texttt{FQI-OPE} under different settings of $H$. 
The results show that \texttt{VA-OPE}'s advantage over \texttt{FQI-OPE} increases as $H$ becomes larger. 
It thus remains open to derive an instance-dependent lower bound that matches our upper bound.

\section{The Uniform Convergence Result}\label{sec: proof of uniform convergence}

\subsection{Important Remark}\label{sec: statement general form}
Throughout the appendix, we consider and analyze a slightly more general form of Algorithm \ref{alg:weighted FQI}. We now explain.

Recall that in \eqref{def: sigmah}, we define $\sigma_h(\cdot,\cdot)$ as
\begin{align*}
    \sigma_h(s,a) = \sqrt{\max\{1,\VV_hV_{h+1}^\pi(s,a)\}+1},
\end{align*}
and the corresponding estimator is given by 
\begin{align*}
    \hat\sigma_h(\cdot,\cdot) \leftarrow \sqrt{\max\{1,\hat\VV_h\hatV_{h+1}^\pi(\cdot,\cdot) \}+ 1} . 
\end{align*}
Here taking maximum with $1$ is to deal with the situation where $\hat\VV_h \hatV_{h+1}^\pi(\cdot,\cdot)$ is close to zero or negative, and the second $1$ is to account for the variance of the rewards. Now as a more general scheme, we replace both with adjustable parameters: $\eta_h$ and $\sigmanoise^2$ such that $\eta_h \geq 1$ and $0 \leq \sigmanoise \leq 1$. Thereby, for each $h\in[H]$, we have
\begin{align*}
    \hat\sigma_h(\cdot,\cdot) \leftarrow \sqrt{\max\{\eta_h,\hat\VV_h\hatV_{h+1}^\pi(\cdot,\cdot) \}+ \sigmanoise^2} \, .     
\end{align*}
We allow the flexibility of the choices of $\{\eta_h\}_{h\in[H]}$ and $\sigmanoise$ in part for generality and theoretical interests.
These parameters will appear in the final results for the uniform convergence and the OPE error bound. The general algorithm is then presented as in Algorithm \ref{alg: general form}.

\begin{algorithm}[ht]
	\caption{ \texttt{VA-OPE} (general form)}
	\label{alg: general form}
	\begin{algorithmic}[1]
	\STATE {\bfseries Input:} target policy $\pi = \{\pi_h\}_{h \in [H]}$, datasets $\cD=\{ \{ (s_{k,h},a_{k,h},r_{k,h}, s_{k,h}') \}_{h\in[H]} \}_{k\in[K]}$ and $\check{\cD}=\{ \{ (\check{s}_{k,h},\check{a}_{k,h}, \check{r}_{k,h}, \check{s}_{k,h}') \}_{h\in[H]} \}_{k\in[K]}$, initial distribution $\xi_1$, $\hatw_{H+1}^\pi = 0$, $\lambda$, $\sigmanoise$, $\{\eta_h\}_{h\in[H]}$
	\FOR{$h=H,H-1,\dots,1$}
	\STATE $\hat\bSigma_h \leftarrow \sum_{k=1}^K \check\bphi_{k,h}\check\bphi_{k,h}^\top + \lambda\Ib_d$
	\STATE $\hatbeta_h \leftarrow \hat{\bSigma}_h^{-1} \sum_{k=1}^K \check\bphi_{k,h} \hatV_{h+1}^\pi(\check{s}_{k,h}')^2$
	\STATE $\hattheta_h \leftarrow \hat{\bSigma}_h^{-1} \sum_{k=1}^K \check\bphi_{k,h} \hatV_{h+1}^\pi(\check{s}_{k,h}')$
    \STATE $\hat\sigma_h(\cdot,\cdot) \leftarrow \sqrt{\max\{\eta_h,\hat\VV_h\hatV_{h+1}^\pi(\cdot,\cdot) \}+ \sigmanoise^2}$
	\STATE $\hat\bLambda_h\leftarrow\sum_{k=1}^K\bphi_{k,h} \bphi_{k,h}^\top/\hat\sigma_{k,h}^2+ \lambda \Ib_d$
	\STATE $Y_{k,h} \leftarrow  r_{k,h} + \langle \bphi_h^\pi(s_{k,h}'), \hatw_{h+1}^\pi \rangle$
	\STATE $\hatw_h^\pi \leftarrow \hat\bLambda_h^{-1} \sum_{k=1}^K \bphi_{k,h}Y_{k,h} / \hat\sigma_{k,h}^2$
	\STATE $\hat Q_h^\pi(\cdot,\cdot)\leftarrow \langle\bphi(\cdot,\cdot),\hat\wb_h^\pi\rangle$,\quad $ \hatV_{h}^\pi (\cdot) \leftarrow \langle \bphi_h^\pi(\cdot), \hatw_{h}^\pi \rangle$ 
	\ENDFOR
	\STATE {\bfseries Output:} $\hat{v}^\pi_1 \leftarrow \int_{\mathcal{S}} \hatV_1^\pi (s) \ \diff  \xi_1(s) $
	\end{algorithmic}
\end{algorithm}

Correspondingly, throughout the appendix we redefine for each $h \in [H]$: 
\begin{align}\label{eq:def sigma h general}
    \sigma_h(s,a) = \sqrt{\max\{\eta_h,\VV_hV_{h+1}^\pi(s,a)\}+\sigmanoise^2},
\end{align}
and thus $\bLambda_h$ defined in \eqref{def: Lambda} also becomes $(\eta_h, \sigmanoise^2)$-related.

Besides generality, this is actually also meaningful, because let's consider, for example, a situation where the agent actually knows that the reward is deterministic (i.e. there is no noise in the observed reward). Then the agent can choose $\sigmanoise=0$ (though this will not give a huge boost to the OPE error bound since the determinant factor in $\|\vb_h^\pi\|_{\bLambda_h^{-1}}$ is the variance $\VV_hV_{h+1}^\pi$).

\subsection{Recap of Notations}

Before presenting the theorems and proof, let's walk through the algorithm and remind the readers of the notations.

\paragraph{Variance estimation} Recall the dataset $\check \cD = \{\check \cD_h\}_{h\in[H]}$, where $\check \cD_h = \{(\check s_{k,h},\check a_{k,h},\check r_{k,h},\check s'_{k,h})\}_{k\in[K]}$. For each $h$, the dataset $\check \cD_h$ is used to compute the function $\hatsigma_h(\cdot,\cdot)$, which is an estimator for the conditional variance of $V_{h+1}^\pi$. To be more clear, let go through the inner loop of Algorithm \ref{alg: general form}. 

In the main text, due to the space limit, we use the abbreviation: 
\begin{align*}
    \check \bphi_{k,h} = \bphi(\check s_{k,h},\check a_{k,h}).
\end{align*}
For each $h$, the (biased and un-normalized) sample covariance matrix $\hat\bSigma_h$ (line 3) is given as
\begin{align*}
        \hat\bSigma_h = \sum_{k=1}^K \check\bphi_{k,h}\check\bphi_{k,h}^\top + \lambda\Ib_d,
\end{align*}and its normalized population counterpart $\bSigma_h$ is defined by \eqref{def: Sigma} as 
\begin{align*}
        \bSigma_h = \EE_{\bar\pi,h}\left[\bphi(s,a)\bphi(s,a)^\top\right]. 
\end{align*}
Then the Algorithm computes $\hat\VV_h \hatV_{h+1}^\pi$, which is an estimator of $\VV_h V_{h+1}^\pi$, as the following:
\begin{align*}
     [\hat\VV_h \hatV_{h+1}^\pi] (\cdot, \cdot)& = \langle \bphi(\cdot, \cdot) , \hatbeta_h^\pi \rangle_{[0,(H-h+1)^2]} - \left[  \langle \bphi(\cdot,\cdot), \hattheta_h^\pi \rangle_{[0,H-h+1]}   \right]^2 ,
\end{align*}where $\hatbeta_h^\pi$ and $\hattheta_h^\pi$ are computed in Algorithm \ref{alg: general form} based on the estimated value function $\hatV_{h+1}^\pi$ from last iteration, and the dataset $\check D_h$. Finally, the function $\hatsigma_h$ is computed.

\paragraph{Value function estimation}
Once we have the variance estimator $\hatsigma_h$, we can apply weighted regression to estimate the value function $V_h^\pi$ using the dataset $\cD_h = \{( s_{k,h}, a_{k,h}, r_{k,h}, s'_{k,h})\}_{k\in[K]}$. This is described by line \ref{alg: Lambda} to \ref{alg: Q and V} in Algorithm \ref{alg: general form}. Please note that we have adopt the abbreviation:
\begin{align*}
        \bphi_{k,h}  = \bphi(s_{k,h},a_{k,h}), \  \hatsigma_{k,h}  = \hatsigma_h(s_{k,h},a_{k,h}).
\end{align*}
Note that the weighted sample covariance matrix $\hat\bLambda_h$ in Algorithm \ref{alg: general form} is given as 
\begin{align*}
        \hat\bLambda_h = \sum_{k=1}^K\bphi_{k,h} \bphi_{k,h}^\top/\hat\sigma_{k,h}^2+ \lambda \Ib_d ,
\end{align*} with its normalized population counterpart $\bLambda_h$ defined by \eqref{def: Lambda} as 
\begin{align*}
    \bLambda_h = \EE_{\bar\pi,h}\left[\sigma_h(s,a)^{-2}\bphi(s,a)\bphi(s,a)^\top\right].
\end{align*}
 Also note that in the offline dataset $\cD$, for each $\cD_h$ and the data point $( s_{k,h}, a_{k,h}, r_{k,h}, s'_{k,h})$ in $\cD_h$, the reward $r_{k,h}$ is the random reward given by $r_{k,h} = r_h(s_{k,h},a_{k,h}) + \epsilon_{k,h}$, where $r_h(\cdot,\cdot)$ is an unknown deterministic function representing the (conditional) mean and $\epsilon_{k,h}$ is some independent random noise. We only observe $r_{k,h}$ and not $\epsilon_{k,h}$.

\paragraph{Function classes} Based on this characterization of the value functions, we define the following function class for each $h\in[H]$ and $L>0$:
\begin{align}\label{def: v function class}
    \cV_h(L) \coloneqq \left\{V(s)=\langle\bphi_h^\pi(s), \wb\rangle\bigg|\wb\in\RR^d,\|\wb\|_2\leq L , \ \sup_{s\in\cS}|V(s)|\leq H-h+2  \right\}. 
\end{align}
One can see that functions in $\cV_h(L)$ are parametrized by vectors $\wb\in\RR^d$.
From Proposition \ref{proposition: linear Q}, it is clear that $V_h^\pi\in\cV_h(2H\sqrt{d})$ for all  $h\in[H]$. 

We define the following function class for each $h \in [H]$ and $L_1,L_2 >0$:
\begin{align}\label{def: sigma function class}
         &\cT_h(L_1,L_2) \notag
        \\ & \coloneqq \left\{ \sigma(\cdot,\cdot)= \sqrt{\max\left\{\eta_h, \langle\bphi(\cdot,\cdot),\bbeta\rangle_{[0,(H-h+1)^2]} + \left[ \langle\bphi(\cdot,\cdot),\btheta\rangle_{[0,H-h+1]}\right]^2 \right\}+\sigmanoise^2} \ \Bigg| \  \| \bbeta\|\leq L_1 , \|\btheta\|\leq L_2 \right\},    
\end{align}which is parametrized by $\bbeta$, $\btheta \in \RR^d$. Later we will see that, with high probability, for all $h\in[H]$, we have $\hatsigma_h \in T_h(L_1,L_2)$ with above choice of $L_1=H^2\sqrt{\frac{Kd}{\lambda}}$ and $L_2=H\sqrt{\frac{Kd}{\lambda}}$, which is an immediate result by Theorem \ref{thm:uniform convergence} and Lemma \ref{lem:bound beta w}. Also note that $\sigma_h \in \cT_h(L_1,L_2)$, which is clear from \eqref{eq:var_hatV_decompose}.

\subsection{Formal Statement of Uniform Convergence Theorem}\label{sec: uni conv formal statement}

\textbf{A weaker data sampling assumption.} Recall Assumption \ref{assump:data_generation stage sampling} on the data sampling process introduced in the main text. It turns out that the uniform convergence result (Theorem \ref{thm:uniform convergence}) holds under a weaker assumption which is the following. 

\begin{assumption}[Trajectory-sampling Data]\label{assump:data_generation trajectory} We have two offline datasets $\cD$ and $\check{\cD}$ where each dataset consists of $K$ trajectories with horizon length equal to $H$. Each trajectory is independently generated by the behavior policy $\bar\pi$. That is, $\cD = \{ \cD_k\}_{k \in [K]}$, where each $\cD_k$ is given by $\cD_{k} = \{(s_{k,h}, a_{k,h}, r_{k,h})\}_{h\in[H]}$ such that $a_{k,h}\sim\pi_h(\cdot|s_{k,h})$ and $s_{k,h+1}\sim\PP_h(\cdot|s_{k,h},a_{k,h})$. For each $(k,h) \in [K]\times[H]$, the random reward $r_{k,h} = r_h(s_{k,h}, a_{k,h}) + \epsilon_{k,h}$, where $r_h(s_{k,h},a_{k,h})$ is the (unknown) expected reward and $\epsilon_{k,h}$ is the noise. Similarly, we have $\check{\cD} = \{ \check{\cD}_k\}_{k\in[K]}$, where $\check{\cD}_k = \{(\check{s}_{k,h}, \check{a}_{k,h} , \check{r}_{k,h} )\}_{h\in [H]}$. Here we denote $s'_{k,h} = s_{k,h+1}$ for simplicity.
\end{assumption}

Note that Assumption \ref{assump:data_generation stage sampling} is stronger than Assumption \ref{assump:data_generation trajectory} in the sense that Assumption \ref{assump:data_generation stage sampling} assumes an extra independence between the data points sampled at different stages. Therefore, as will be clear from the proof, since Theorem \ref{thm:uniform convergence} is established under Assumption \ref{assump:data_generation trajectory}, it automatically holds under the stronger Assumption \ref{assump:data_generation stage sampling}.

We now introduce the uniform convergence theorem. To simplify the notation, we define:
\begin{align*}
    C_{h,1} & = \sum_{i=h}^H \frac{1}{\iota_h} \,,  \ \  C_{h,2} = \sum_{i=h}^H \sqrt{\frac{C_{h,3}}{\iota_h}} \,, \ \  C_{h,3} = \frac{(H-h+1)^2}{\eta_h+\sigmanoise^2} \,, \ \ C_{h,4} = \left(\|\bLambda_h \|\cdot \|\bLambda_h^{-1}\| \right)^{1/2} \,.
\end{align*}Note that by setting $\eta_h = \sigmanoise = 1$ we recover the same $C_{h,2}$, $C_{h,3}$ as in the main text. 

\begin{theorem}[Uniform Convergence]\label{thm:uniform convergence}
Set $\lambda=1$ and $\eta_h\in(0,(H-h+1)^2]$ for all $h\in[H]$ in Algorithm \ref{alg:weighted FQI}. Under Assumption \ref{assump:linear_MDP}, \ref{assump: smallest eigenvalue} and~\ref{assump:data_generation trajectory}, there exists some universal constant $C$ such that if $K$ satisfies
\begin{align}\label{eq:K lower bound uniform thm}
    K &\geq C\cdot \frac{H^2d^2}{\kappa^2}\log\left(\frac{dHK}{\kappa\delta}\right)\cdot\max_{h\in[H]}\frac{(H-h+1)^2}{(\eta_h+\sigmanoise^2)^2} \cdot\max_{h\in[H]}\frac{(H-h+1)^2}{\iota_h(\eta_h+\sigmanoise^2)} \, ,
\end{align}
then with probability at least $1-\delta$, it holds for all $h\in[H]$ that $\sup_{s\in\cS}\left|\hat V_h^\pi(s)\right|\leq H-h+2$, and
\begin{align*}
    \sup_{s\in\cS}\left|\hat V_h^\pi(s)-V_h^\pi(s)\right|&\leq 
    C\cdot\frac{C_{h,2}d}{\sqrt{K}}\log\left(\frac{dH^2K}{\kappa\delta}\right) + C\cdot \frac{C_{h,1}H\sqrt{d}}{K} \, . 
\end{align*}
\end{theorem}

We now present the proof of Theorem \ref{thm:uniform convergence}.
The proof relies on a backward induction argument, i.e., we will show $|\hat V_h^\pi(s)-V_h^\pi(s)|$ is uniformly small for $h = H , H-1 , \cdots, 1$. For this purpose, we need to use the first form of error decomposition \eqref{eq: error decomp V} in Lemma \ref{lemma: error decomposition}. 

\subsection{Step 1: Base Case at Stage $h=H$}\label{sec:stage H}
We first bound the approximation error at the last stage $h=H$. From the algorithm we have $\hatV_{H+1}^\pi \equiv V_{H+1}^\pi \equiv 0$. Therefore, we have $\hat\btheta_H=\hat\bbeta_H=0$, $\hatsigma_H\equiv \sqrt{\eta_H + \sigmanoise^2}$, and
\begin{align*}
    \hat\bLambda_H = \frac{1}{\eta_H+\sigmanoise^2}\sum_{k=1}^K \bphi(s_{k,H},a_{k,H})\bphi(s_{k,H},a_{k,H})^\top +\lambda\Ib_d.
\end{align*}
By the error decomposition in \eqref{eq: error decomp V}, we have
\begin{align}\label{eq: H 1}
    V_H^\pi(s)-\hat V_H^\pi(s) &= \underbrace{-\bphi_H^\pi(s)^\top\hat\bLambda_H^{-1}\sum_{k=1}^K\frac{\bphi(s_{k,H},a_{k,H})}{\hat\sigma_H(s_{k,H},a_{k,H})^2}\epsilon_{k,H}}_{\Delta_1} + \underbrace{\lambda\bphi_H^\pi(s)^\top\hat\bLambda_H^{-1}\wb_H^\pi}_{\Delta_2}.
\end{align}
We will bound the two terms separately.

To bound $|\Delta_1|$, we first apply Cauchy-Schwartz inequality to obtain that
\begin{align}\label{eq: H diff E1 0}
    |\Delta_1| &\leq \|\bphi_H^\pi(s)\|_{\hat\bLambda_H^{-1}} \cdot \left\|\sum_{k=1}^K\frac{\bphi(s_{k,H},a_{k,H})}{\hat\sigma_H(s_{k,H},a_{k,H})^2}\epsilon_{k,H}\right\|_{\hat\bLambda_H^{-1}}
\end{align}
By Lemma \ref{lemma: quadratic form}, with probability at least $1-\delta$, we have
\begin{align}\label{eq: H quadratic form}
    \|\bphi_H^\pi(s)\|_{\hat\bLambda_H^{-1}}&\leq \frac{2}{\sqrt{K}}\cdot\|\bphi_H^\pi(s)\|_{\bLambda_H^{-1}}
\end{align}
for all $s\in\cS$, as long as $K$ satisfies that
\begin{align*}
    K\geq\max\left\{\frac{512\|\bLambda_H^{-1}\|^2}{(\eta_H+\sigmanoise^2)^2}\log\left(\frac{2d}{\delta}\right),2\lambda\|\bLambda_H^{-1}\|\right\}=\frac{512\|\bLambda_H^{-1}\|^2}{(\eta_H+\sigmanoise^2)^2}\log\left(\frac{2d}{\delta}\right).
\end{align*}
Note that $\Var(\epsilon_{k,h})=\sigmanoise^2\leq\hat\sigma_H(s_{k,H},a_{k,H})^2$ for all $k\in[K]$. Then by Theorem \ref{thm:self normalized bernstein} we have
\begin{align}\label{eq: H self-normalized}
    &\left\|\sum_{k=1}^K\hat\sigma_H(s_{k,H},a_{k,H})^{-2}\bphi(s_{k,H},a_{k,H})\epsilon_{k,H}\right\|_{\hat\bLambda_H^{-1}} \notag \\
    &\leq  8\sqrt{d\log\left( 1+\frac{K}{\lambda d(\eta_H+\sigmanoise^2)} \right) \cdot \log\left( \frac{4K^2}{\delta} \right)} + 4\sqrt{\frac{1}{\eta_H+\sigmanoise^2}}\log \left( \frac{4K^2}{\delta} \right) \notag \\
    &\leq  12\sqrt{d}\log\left(\frac{4K^2}{\delta}\right)    
\end{align}
with probability at least $1-\delta$.

Then by \eqref{eq: H diff E1 0}, it suffices to take a union bound over \eqref{eq: H quadratic form} and \eqref{eq: H self-normalized} to conclude that if $K\geq 128\|\bLambda_H^{-1}\|^2\log(2d/\delta)/(\eta_H+\sigmanoise^2)$ then
\begin{align}\label{eq: H diff E1}
        |\Delta_1|\leq \frac{12\sqrt{d}}{\sqrt{K}}\|\bphi_H^\pi(s)\|_{\bLambda_H^{-1}}\cdot\log\left(\frac{4K^2}{\delta}\right)
\end{align}
with probability at least $1-\delta$.

At the same time, we can bound $|\Delta_2|$ using the same argument.
\begin{align}\label{eq: H diff E2}
    |\Delta_2|\leq \lambda\|\bphi_H^\pi(s)\|_{\hat\bLambda_H^{-1}} \cdot \|\wb_H^\pi\|_{\hat\bLambda_H^{-1}}\leq \frac{4\lambda}{K}\cdot\|\bphi_H^\pi(s)\|_{\bLambda_H^{-1}}\cdot\|\wb_H^\pi\|_{\bLambda_H^{-1}},
\end{align}
where the second inequality holds on the same event as does \eqref{eq: H diff E1}.

Finally, we combine \eqref{eq: H 1}, \eqref{eq: H diff E1} and \eqref{eq: H diff E2}, and obtain that if $K\geq 512\|\Lambda_H^{-1}\|^2/(\eta_H+\sigmanoise^2)^2\log(4d/\delta)$ then 
\begin{align*}
    \sup_{s\in\cS}\left| V_H^\pi(s)-\hat V_H^\pi(s) \right|&\leq \frac{12\sqrt{d}}{\sqrt{K}}\log\left(\frac{4K^2}{\delta}\right)\cdot\sup_{s\in\cS}\|\bphi_H^\pi(s)\|_{\bLambda_H^{-1}} + \frac{4\lambda}{K} \cdot \sup_{s\in\cS}\|\bphi_H^\pi(s)\|_{\bLambda_H^{-1}}\cdot\|\wb_H^\pi\|_{\bLambda_H^{-1}}\\
    &\leq \frac{12\sqrt{d\|\bLambda_H^{-1}\|}}{\sqrt{K}}\log\left(\frac{4K^2}{\delta}\right) + \frac{8\lambda H\sqrt{d}\|\bLambda_H^{-1}\|}{K}
\end{align*}
with probability at least $1-\delta$, where the last inequality follows from Assumption \ref{assump:linear_MDP}, Proposition \ref{proposition: linear Q} and the choice that $\lambda=1$.
Note that since $\hat\sigma_H(\cdot,\cdot)\leq 1+\sigmanoise^2$, we have $\bLambda_H\succeq \bSigma_H/(1+\sigmanoise^2)$, which implies that $\|\bLambda_H^{-1}\|\leq 2\|\bSigma_H^{-1}\|$ as $\sigmanoise^2\leq 1$. Then we further have
\begin{align*}
    \sup_{s\in\cS}\left|V_H^\pi(s)-\hatV_H^\pi(s)\right|\leq \frac{12\sqrt{2d}}{\sqrt{K\kappa_H}}+\frac{16\lambda H^3\sqrt{d}}{K\kappa_H}
\end{align*}
Meanwhile, we can bound $\sup_{s\in\cS}|\hat V_H^\pi(s)|$ as follows
\begin{align*}
    \sup_{s\in\cS}|\hat V_H^\pi(s)| &\leq \sup_{s\in\cS}V_H^\pi(s)+\frac{12\sqrt{2d}}{\sqrt{K\kappa_H}}+\frac{16\lambda H\sqrt{d}}{K\kappa_H}\leq 2,
\end{align*}
when $K$ satisfies that $K\geq 600(\lambda+1) (d+H\sqrt{d})/\kappa_H$.

In conclusion, we have
\begin{align*}
    \sup_{s\in\cS}|\hat V_H^\pi(s)|&\leq 2,
\end{align*}
and
\begin{align*}
    \sup_{s\in\cS}\left|V_H^\pi(s)-\hat V_H^\pi(s)\right|&\leq \frac{12\sqrt{2d}}{\sqrt{K\kappa_H}}+\frac{16\lambda H\sqrt{d}}{K\kappa_H},
\end{align*}
given that $K$ satisfies
\begin{align}\label{eq: K lower bound stage H uniform}
    K\geq\max\left\{\frac{2048}{\kappa_H^2(\eta_H+\sigmanoise^2)^2}\log\left(\frac{2d}{\delta}\right),600(\lambda+1)\frac{d+H\sqrt{d}}{\kappa_H}\right\}
\end{align}

\subsection{Step 2: Induction Hypothesis}\label{sec: induction step}

For the induction hypothesis, we assume that if for all sufficiently large $K$, with probability at least $1 - (H-h)\delta$, the following event (denoted as $\cE_{h+1}$) holds:
\begin{align*}
     \sup_{s\in\cS}|\hatV_{h+2}^\pi(s)| \leq  H-h,
    \  \sup_{s\in\cS}|\hatV_{h+1}^\pi(s)| \leq  H-h + 1,
    \ \sup_s |\hat V_{h+1}^\pi(s)  - V_{h+1}^\pi(s)| \leq  \alpha_{H-h},
\end{align*}
where $\alpha_{H-h}\leq (\eta_h+\sigmanoise^2)/[8(H-h+1)]$.

We claim that if $K$ satisfies
\begin{align}\label{eq: K lower bound uniform induction}
        K\geq \frac{3600(H-h+1)^4d^2}{\kappa_h^2(\eta_h+\sigmanoise^2)^2}\cdot\log\left(\frac{dHK}{\kappa_h\delta}\right)
\end{align}
then with probability at least $1-(H-h+1)\delta$, the following event (denoted by $\cE_h$) holds:
\begin{align*}
        |\hatV_{h+1}^\pi| & <  H-h + 1 ,
        \\ |\hatV_{h}^\pi| & <  H-h + 2 ,
        \\ \sup_s |\hat V_{h}^\pi(s) - V_{h}^\pi(s)| & \leq\left(1+\frac{8\lambda}{\iota_hK}\right)\alpha_{H-h} + \frac{2\lambda H\sqrt{d}}{\iota_hK}\\ &\qquad + \frac{20}{\sqrt{K}}\cdot\left(\frac{d}{\sqrt{\iota_h}}+\frac{d(H-h+1)}{\sqrt{\iota_h(\eta_h+\sigmanoise^2)}}\right) \cdot\log\left(\frac{d(H-h+1)^2K}{\kappa_h(\eta_h+\sigmanoise^2)\delta}\right)
\end{align*}
We again bound the three terms in the error decomposition \eqref{eq: error decomp V} simultaneously. Let $\tilde\cE_h$ be the event given by Lemma \ref{lemma: uniform convergence induction bernstein} for $h$ such that $\PP(\tilde\cE_h)\geq 1-\delta$, where we have $L = (1+1/H)d\sqrt{K/\lambda}$.

Let's consider the event $\tilde\cE_h \cap \cE_{h+1}$, which satisfies $\PP \{ \tilde\cE_h \cap \cE_{h+1}\} \geq 1- (H-h+1)\delta$ by a union bound. Note that on $\cE_{h+1}$, we have $|\hatV_{h+1}^\pi|\leq H-h+1$.  Furthermore, since $|\hatV_{h+2}^\pi|\leq H-h$ on $\cE_{h+1}$, again by Lemma \ref{lem:bound beta w} with $B=H$, we see that $\hatV_{h+1}^\pi \in \cV_{h+1}(L)$. Therefore, by Lemma \ref{lemma: uniform convergence induction bernstein} it holds on $\tilde\cE_h\cap\cE_{h+1}$ that
\begin{align}\label{eq: uniform induction 1}
    \left\|\left(\frac{\hat\bLambda_h}{K}\right)^{-1}\right\|\leq \frac{8}{\iota_h},
\end{align}
and
\begin{align}\label{eq: uniform induction 2}
    &\left|\bphi(s,a)^\top\hat\bLambda_h^{-1}\sum_{k=1}^K\hat\sigma_h(s_{k,h},a_{k,h})^{-2}\bphi(s_{k,h},a_{k,h})\left(\PP_hV(s_{k,h},a_{k,h})-V(s_{k,h}')-\epsilon_{k,h}\right)\right|\notag\\
   & \leq  \frac{20}{\sqrt{K}}\cdot\left(\frac{d}{\sqrt{\iota_h}}+\frac{d(H-h+1)}{\sqrt{\iota_h(\eta_h+\sigmanoise^2)}}\right) \cdot\log\left(\frac{d(H-h+1)^2K}{\kappa_h(\eta_h+\sigmanoise^2)\delta}\right),
\end{align}
for all $(s,a)\in\cS\times\cA$.

Since it holds on $\cE_{h+1}$ that $\sup_{s\in\cS}|\hat V_{h+1}^\pi(s)-V_{h+1}^\pi(s)|\leq \alpha_{H-h}$, we have
\begin{align}\label{eq: uniform induction 3}
    \sup_{s\in\cS}\left|[\JJ_h\PP_h(V_{h+1}^\pi-\hat V_{h+1}^\pi)](s)\right|\leq \alpha_{H-h}.
\end{align}
Also by \eqref{eq: uniform induction 1} and $\sup_{s\in\cS}|\hat V_{h+1}^\pi(s)-V_{h+1}^\pi(s)|\leq \alpha_{H-h}$, it follows from Cauchy-Schwartz inequality that
\begin{align}\label{eq: uniform induction 4}
    \sup_{s\in\cS}\left| \lambda\bphi_h^\pi(s)^\top\hat\bLambda_h^{-1}\int_{\cS}\rbr{V_{h+1}^\pi(s)- \hat V_{h+1}^\pi(s)}\bmu_h(s)\text{d}s \right|\leq\frac{8\lambda\alpha_{H-h}}{\iota_hK}.
\end{align}
Similarly, for the last term in \eqref{eq: error decomp V} we have by Cauchy-Schwartz inequality that
\begin{align}\label{eq: uniform induction 5}
    \sup_{s\in\cS}\left|\lambda\bphi_h^\pi(s)^\top\hat\bLambda_h^{-1}\wb_h^\pi\right|&\leq \lambda\sup_{s\in\cS}\|\bphi_h^\pi(s)\|_2\cdot\|\hat\bLambda_h^{-1}\|\cdot\|\wb_h^\pi\|_2\leq \frac{2\lambda H \sqrt{d}}{\iota_hK},
\end{align}
where the second inequality follows from Proposition \ref{proposition: linear Q}.

Finally, combining \eqref{eq: uniform induction 2}, \eqref{eq: uniform induction 3}, \eqref{eq: uniform induction 4} and \eqref{eq: uniform induction 5}, we obtain by the error decomposition \eqref{eq: error decomp V} that
\begin{align}\label{eq: uniform induction 6}
    \sup_{s\in\cS}\left|V_h^\pi(s)-\hat V_h^\pi(s)\right| &\leq \frac{20}{\sqrt{K}}\cdot\left(\frac{d}{\sqrt{\iota_h}}+\frac{d(H-h+1)}{\sqrt{\iota_h(\eta_h+\sigmanoise^2)}}\right) \cdot\log\left(\frac{d(H-h+1)^2K}{\kappa_h(\eta_h+\sigmanoise^2)\delta}\right)\notag\\
    &\qquad+\left(1+\frac{8\lambda}{\iota_hK}\right)\alpha_{H-h} + \frac{2\lambda H\sqrt{d}}{\iota_hK}
\end{align}
Note that when $K$ satisfies \eqref{eq: K lower bound uniform induction}, we would have $\sup_{s\in\cS}|V_h^\pi(s)-\hat V_h^\pi(s)|\leq 1$, and thus
\begin{align}\label{eq: uniform induction 7}
  \sup_{s\in\cS}|\hat V_h^\pi(s)|\leq \sup_{s\in\cS}|V_h^\pi(s)|+\sup_{s\in\cS}\left|V_h^\pi(s)-\hat V_h^\pi(s)\right|\leq H-h+2. 
\end{align}
Therefore, by \eqref{eq: uniform induction 6} and \eqref{eq: uniform induction 7}, we conclude that $\tilde\cE_h \cap \cE_{h+1} \subseteq \cE_h$, which implies that 
\begin{align*}
    \PP\{\cE_h\} \geq \PP \{ \tilde\cE_h \cap \cE_{h+1}  \} \geq 1-(H-h+1)\delta.
\end{align*}

\subsection{Step 3: Recursion}

Let $\kappa=\min_{h\in[H]}\kappa_h$. Suppose $K$ satisfies that
\begin{align}\label{eq: K lowerbound uniform}
        K\geq \frac{3600H^2d^2}{\kappa^2}\log\left(\frac{dHK}{\kappa\delta}\right)\cdot\max_{h\in[H]}\frac{(H-h+1)^2}{(\eta_h+\sigmanoise^2)^2}\cdot\max_{h\in[H]}\frac{(H-h+1)^2}{\iota_h(\eta_h+\sigmanoise^2)}
\end{align}
We also define the following quantity
\begin{align}\label{eq: recursion 1}
    \xi_{H-h} = \frac{2\lambda H\sqrt{d}}{\iota_hK} + \frac{20}{\sqrt{K}}\cdot\left(\frac{d}{\sqrt{\iota_h}}+\frac{d(H-h+1)}{\sqrt{\iota_h(\eta_h+\sigmanoise^2)}}\right) \cdot\log\left(\frac{d(H-h+1)^2K}{\kappa_h(\eta_h+\sigmanoise^2)\delta}\right)
\end{align}
for all $h\in[H]$. With the choice of $K$ in \eqref{eq: K lowerbound uniform}, the following holds
\begin{align}\label{eq: recursion 2}
    \xi_{H-h}\leq \min\left\{\frac{1}{2He},\frac{1}{16H}\cdot\min_{h\in[H]}\frac{\eta_h+\sigmanoise^2}{H-h+1}\right\}
\end{align}
for all $h\in[H]$.

First by the base case at stage $h=H$ from subsection \ref{sec:stage H}, we have
\begin{align}\label{eq: recursion 3}
    \sup_{s\in\cS}\left|V_H^\pi(s)-\hat V_H^\pi(s)\right|&\leq \frac{12\sqrt{2d}}{\sqrt{K\kappa_H}}+\frac{16\lambda H\sqrt{d}}{K\kappa_H}\coloneqq \alpha_0.
\end{align}
Also by the choice of $K$ in \eqref{eq: K lowerbound uniform}, we have
\begin{align}\label{eq: recursion 4}
    \alpha_0\leq \min\left\{\frac{1}{2e},\frac{1}{16e}\cdot\min_{h\in[H]}\frac{\eta_h+\sigmanoise^2}{H-h+1}\right\}.
\end{align}
Then by the induction step at stage $h=H-1$ from subsection \ref{sec: induction step}, we have
\begin{align*}
    \sup_{s\in\cS}\left|V_{H-1}(s)-\hat V_{H-1}^\pi(s)\right|\leq \alpha_1,
\end{align*}
with
\begin{align*}
    \alpha_1 &\leq \left(1+\frac{\lambda C_{H-1}H}{K}\right)\alpha_0+\xi_1\leq \left(1+\frac{1}{H}\right)\alpha_0+\xi_1\leq \min\left\{1,\frac{\eta_H+\sigmanoise^2}{16}\right\}
\end{align*}
We then define $\alpha_{H-h}=(1+1/H)\alpha_{H-h+1}+\xi_{H-h}$ recursively for all $h\in[H-1]$. Note that for all $i\in[H-h]$ we have
\begin{align*}
    \alpha_i &\leq \left(1+\frac{1}{H}\right)^i\alpha_0 + \sum_{j=0}^i \left(1+\frac{1}{H}\right)^{i-j}\xi_j\leq e\cdot\alpha_0 + e\cdot\sum_{h=0}^i\xi_j\leq \min\left\{1,\frac{\eta_{H-i}+\sigmanoise^2}{8(i+1)}\right\},
\end{align*}
where the first inequality follows from the fact that $(1+1/n)^n\leq e$ for all positive integer $n$, and the second inequality is due to \eqref{eq: recursion 2} and \eqref{eq: recursion 4}.

Therefore, we may apply the induction step from the previous subsection to all $h\in[H-1]$ and obtain that
\begin{align*}
    \sup_{s\in\cS}\left|\hat V_h^\pi\right|\leq H-h+2
\end{align*}
and
\begin{align}\label{eq: recursion 5}
    \sup_{s\in\cS}\left|\hat V_h^\pi(s)-V_h^\pi(s)\right|&\leq \left(1+\frac{1}{H}\right)\alpha_{H-h}+\iota_{H-h}\notag\\
    &\leq \left(1+\frac{1}{H}\right)^{H-h}\alpha_0+\sum_{i=0}^{H-h}\left(1+\frac{1}{H}\right)^{H-h-i}\xi_i\notag\\
    &\leq e\cdot\alpha_0 + e\cdot\sum_{i=0}^{H-h}\xi_i
\end{align}
with probability at least $1-H\delta$ simultaneously for all $h\in[H]$.

Therefore, replacing $\delta$ by $\delta/H$ and plugging \eqref{eq: recursion 1} and \eqref{eq: recursion 3} into \eqref{eq: recursion 5}, we obtain that
\begin{align}\label{eq: recursion 6}
    & \sup_{s\in\cS}\left|\hat V_h^\pi(s)-V_h^\pi(s)\right| \notag
    \\&\leq \frac{12e\sqrt{2d}}{\sqrt{K\kappa_H}} + \frac{16e\lambda H\sqrt{d}}{K\kappa_H} + \frac{2e\lambda H\sqrt{d}}{K}\sum_{i=h}^{H-1}\frac{1}{\iota_h} + \frac{40ed}{\sqrt{K}}\log\left(\frac{dH^2K}{\kappa\delta}\right)\sum_{i=h}^{H-1}\frac{H-h+1}{\sqrt{\iota_h(\eta_h+\sigmanoise^2)}}.
\end{align}
We further define
\begin{align*}
    C_{h,1} = \sum_{i=h}^H \frac{1}{\iota_h},\ C_{h,2} = \sum_{i=h}^H \frac{H-h+1}{\sqrt{\iota_h(\eta_h+\sigmanoise^2)}},
\end{align*}
then we can simplify and rearrange \eqref{eq: recursion 6} to 
\begin{align*}
    \sup_{s\in\cS}\left|\hat V_h^\pi(s)-V_h^\pi(s)\right|&\leq 
    C\cdot\frac{C_{h,2}d}{\sqrt{K}}\log\left(\frac{dH^2K}{\kappa\delta}\right) + C\cdot \frac{C_{h,1}H\sqrt{d}}{K}.
\end{align*}
This completes the proof of Theorem \ref{thm:uniform convergence}.

\section{Proof of OPE Convergence}\label{sec: proof of OPE}

As stated in Appendix \ref{sec: statement general form}, we consider the general form of Theorem \ref{thm:ope}. Recall the following notation:
\begin{align*}
    C_{h,1} & = \sum_{i=h}^H \frac{1}{\iota_h} \,,  \ \  C_{h,2} = \sum_{i=h}^H \sqrt{\frac{C_{h,3}}{\iota_h}} \,, \ \  C_{h,3} = \frac{(H-h+1)^2}{\eta_h+\sigmanoise^2} \,, \ \ C_{h,4} = \left(\|\bLambda_h \|\cdot \|\bLambda_h^{-1}\| \right)^{1/2} \,.
\end{align*}

\begin{theorem}[General form of Theorem \ref{thm:ope}]\label{thm:ope general}
Set $\lambda=1$, $\eta_h\in(0,(H-h+1)^2]$ for all $h\in[H]$ and $\sigmanoise^2 \leq 1$. Under Assumptions \ref{assump:linear_MDP}, \ref{assump: smallest eigenvalue} and \ref{assump:data_generation stage sampling}, if $K$ satisfies 
\begin{align}\label{eq:K lower bound OPE general}
    K \geq C \cdot C_3 \cdot d^2 \left[ \log\left(\frac{dH^2K}{\kappa\delta}\right)\right]^2 ,
\end{align} where $C$ is some problem-independent universal constant and  
\begin{align*}
       C_3 \coloneqq \max & \bigg\{  \max_{h\in[H]} \frac{C_{h,3} \cdot C_{h,2}^2 }{ \iota_h^2 (\eta_h+\sigmanoise^2)^3} \, \ , \ \ \frac{H^4}{\sigmanoise^4\kappa^2} \, , \  \frac{H^2}{\sigmanoise^4\kappa^2} \cdot  \max_{h\in[H]}\frac{C_{h,3}}{\eta_h+\sigmanoise^2} \cdot\max_{h\in[H]}\frac{C_{h,3}}{\iota_h}   \bigg\}.
\end{align*}
Then with probability at least $1-\delta$, it holds that
\begin{align*}
        |v_1^\pi - \hat{v}_1^\pi| & \leq C \cdot \left[ \sum_{h=1}^H \left\| \vb_h^\pi \right\|_{\bLambda_h^{-1}} \right] \cdot \sqrt{\frac{\log(16H/\delta)}{K}} + C \cdot C_4 \cdot \log\left(\frac{16H}{\delta}\right) \cdot \left( \frac{1}{K^{3/4}} + \frac{1}{K} \right) \,,
\end{align*}where $C_4 \coloneqq \sum_{h=1}^H \left\{ \sqrt{C_{h,4} \cdot C_{h,2} \cdot \frac{(H-h+1)d}{\iota_h(\eta_h+\sigmanoise^2)^2}\cdot \log\left(\frac{dH^2K}{\kappa\delta}\right)} \cdot \left\| \vb_h^\pi\right\|_{\bLambda_h^{-1}} \right\}. $
\end{theorem}

Note that by setting $\eta_h = \sigmanoise = 1$ we recover Theorem \ref{thm:ope}.

The proof is based on the recursive error decomposition given by \eqref{eq:ope decomp} and the prerequisite result on uniform convergence. We will show the OPE convergence conditioned on the high probability event of uniform convergence established by Theorem \ref{thm:uniform convergence}. 

Recall the error decomposition for the OPE problem given by \eqref{eq:ope decomp} (proof in Section \ref{sec: proof of error decomp}):
\begin{align}\label{eq:error decomp ope proof}
    v_1^\pi - \hat{v}_1^\pi & = -\lambda\sum_{h=1}^H (\vb_h^\pi)^\top \hat\bLambda_h^{-1}\int_{\cS}\rbr{V_{h+1}^\pi(s)-\hat V_{h+1}^\pi(s)}\bmu_h(s)\text{d}s \notag \\
    &\quad + \sum_{h=1}^H (\vb_h^\pi)^\top\hat\bLambda_h^{-1}\sum_{k=1}^K\frac{\bphi(s_{k,h},a_{k,h})}{\hat\sigma_h(s_{k,h},a_{k,h})^2}\rbr{[\PP_h\hat V_{h+1}^\pi](s_{k,h},a_{k,h}) - \hat V_{h+1}^\pi(s_{k,h}')-\epsilon_{k,h}} \notag \\ 
    & \quad + \lambda \sum_{h=1}^H (\vb_h^\pi)^\top \hat\bLambda_h^{-1}\wb_h^\pi \notag \\
    & \coloneqq E_1 + E_2 + E_3 .     
\end{align}
It suffices to prove that each term can be bounded with high probability and then we can take a union bound. By the result of Theorem \ref{thm:uniform convergence}, we can condition on the event where both $\hat V_{h+1}^\pi$ and $\hat\sigma_h$ are good estimators of their population counterparts.

\begin{remark}
All the lemmas in the remaining of this Section \ref{sec: proof of OPE} will be proved under Assumptions \ref{assump:linear_MDP} and \ref{assump:data_generation stage sampling}. So we do not explicitly add these two assumptions into the description of the lemmas. 

Also, recall the function classes $\cV_h(L)$ and $\cT_h(L_1,L_2)$ defined by \eqref{def: v function class} and \eqref{def: sigma function class}. In the remaining of this section, we will assume $L$, $L_1$ and $L_2$ to be 
\begin{align*}
        L  = \frac{H+1}{\sqrt{\eta+\sigmanoise^2}} \sqrt{\frac{Kd}{\lambda}} ,
       \  L_1 = H^2\sqrt{\frac{Kd}{\lambda}},
        \ L_2  = H\sqrt{\frac{Kd}{\lambda}}.
\end{align*}The reason that we can make the above assumption is that, conditioning on the high probability event of uniform convergence (Theorem \ref{thm:uniform convergence}), it follows immediately from Lemma \ref{lem:bound beta w} that we have $\hatsigma_h \in \cT_h(L_1,L_2)$, and $\hatV_h^\pi \in \cV_h(L)$ for all $h\in[H]$ with the above choice of $L$, $L_1$ and $L_2$.
\end{remark}

\subsection{Bounding the $E_2$ Term in the OPE Decomposition}\label{sec:E2 ope}

We consider the term $E_2$ first. Decompose $E_2$ into $E_2 = \sum_{h=1}^H E_{2,h}$ where for each $h\in[H]$, $E_{2,h}$ is given as
\begin{align*}
        E_{2,h} \coloneqq  (\vb_h^\pi)^\top\hat\bLambda_h^{-1}\sum_{k=1}^K\frac{\bphi(s_{k,h},a_{k,h})}{\hat\sigma_h(s_{k,h},a_{k,h})^2}\rbr{[\PP_h\hat V_{h+1}^\pi](s_{k,h},a_{k,h}) - \hatV_{h+1}^\pi(s_{k,h}')-\epsilon_{k,h}}.
\end{align*}Further decompose $E_{2,h}$ into $E_{2,h} = \sum_{k=1}^K e_{h,k}$ where
\begin{align*}
    e_{h,k} = (\vb_h^\pi)^\top\hat\bLambda_h^{-1} \frac{\bphi(s_{k,h},a_{k,h})}{\hat\sigma_h(s_{k,h},a_{k,h})^2}\rbr{[\PP_h\hat V_{h+1}^\pi](s_{k,h},a_{k,h}) - \hatV_{h+1}^\pi(s_{k,h}')-\epsilon_{k,h}} . 
\end{align*}
In the following lemma, we consider the term $E_{2,h}$ for arbitrarily fixed $h \in[H]$. To simplify the notation, we omit the subscript $h$ and take $e_k = e_{h,k}$ once $h$ is fixed.

\begin{lemma}\label{lem:E2 bound component fix sigma}
For any $h\in[H]$, condition on $\hat V_{h+1}^\pi\in\cV_{h+1}(L)$ and the induced $\hatsigma_h(\cdot,\cdot) \in \cT_h(L_1,L_2)$ being fixed, such that $\hatsigma_h$ satisfies for all $(s,a)$
\begin{align}\label{eq:E2 bound component fix sigma lemma assump}
        \left| \hatsigma^2_h(s,a) - \sigmanoise^2 - \max\left\{\eta_h, \ \VV_h \hatV_{h+1}^\pi(s,a) \right\} \right| \leq \frac{C (H-h+1)^2\sqrt{d}}{\sqrt{K}},    
\end{align}for some $C>0$. If $K$ satisfies \eqref{eq:K lower bound E2 component fix sigma}, then with conditional probability at least $1-\delta$, we have 
\begin{align*}
   |E_{2,h}| \leq 2\sqrt{2\log\left(\frac{4}{\delta}\right)B_h}\cdot \frac{1}{\sqrt{K}}\cdot\|\vb_h^\pi\|_{\Gb_h^{-1}}+ \frac{8}{3} \log \left(\frac{4}{\delta}\right) \cdot \frac{2(H-h+1)+1}{\eta_h+\sigmanoise^2} \cdot \|\vb_h^\pi\|_{\Gb_h^{-1}} \cdot \left\| \Gb_h^{-1}\right\|^{1/2}\cdot \frac{1}{K} ,
\end{align*}where $B_h$ is a $\hatV_{h+1}^\pi$-dependent constant and $\Gb_h$ is a $\hatsigma_h$-dependent matrix given by
\begin{align*}
        B_h & = \max_{(s,a)\sim \nu_h} \frac{\VV_h \hatV_{h+1}^\pi(s,a) + \sigmanoise^2}{\max\left\{\eta_h, \VV_h \hatV_{h+1}^\pi(s,a)\right\}+\sigmanoise^2-\frac{C (H-h+1)^2\sqrt{d}}{\sqrt{K}}}
        \\ & \sim 1 + \tilde\cO(1/\sqrt{K}),
        \\ \Gb_h & \coloneqq \EE_h \left[  \frac{\bphi(s,a)\bphi(s,a)^\top}{\hatsigma^2_h(s,a)}  \middle| \hatsigma_h \right] ,
\end{align*}
\end{lemma}

\begin{remark}
Lemma \ref{lem:E2 bound component fix sigma} will be combined with Lemma \ref{lem:hatsigma concentration fixed V uniform}, which gives an explicit formula for the constant $C$ with high probability, as will be shown in Lemma \ref{lem:E2 bound component fix V}.
\end{remark}

\begin{proof}[Proof of Lemma \ref{lem:E2 bound component fix sigma}]
By definition,
\begin{align*}
    e_k = (\vb_h^\pi)^\top\hat\bLambda_h^{-1} \frac{\bphi(s_{k,h},a_{k,h})}{\hat\sigma_h(s_{k,h},a_{k,h})^2}\rbr{[\PP_h\hat V_{h+1}^\pi](s_{k,h},a_{k,h}) - \hatV_{h+1}^\pi(s_{k,h}')-\epsilon_{k,h}}
\end{align*}
for all $k\in[K]$. From Algorithm \ref{alg:weighted FQI}, it is clear that the function$\hatV_{h+1}^\pi(\cdot)$ depends on the dataset $\check \cD_i, \cD_{i}$ for $i\geq h+1$, and the function $\hatsigma_h(\cdot,\cdot)$ depends on $\hatV_{h+1}^\pi$ and the dataset $\check \cD_h $, which are all independent of the dataset $\cD_h$ under Assumption \ref{assump:data_generation trajectory}. Therefore, conditioning on $\hatV_{h+1}$ and $\hatsigma_h$ will not change the distribution of $\cD_h$.

Define $F_h=\{(s_{k,h},a_{k,h}),k\in[K]\}$, and for now we further condition on $F_h$ being fixed. Then $\hat\bLambda_h$ and $\hatsigma_{k,h}\coloneqq \hatsigma_h(s_{k,h},a_{k,h}), \ k\in[K]$ are both fixed. Define the filtration $\{\cF_k\}_{k\in[K]}$ conditioned on $F_h$ as  $\cF_k=\sigma\{s'_{1,h}, \epsilon_{1,h}, \cdots, s'_{k-1,h}, \epsilon_{k-1,h}| F_h\}$ for $1<k\leq K$, and $\cF_1$ as the empty $\sigma$-field. Then $\EE[e_k\mid \cF_k]=0$ implies that $\{e_k\}_{k\in[K]}$ is a martingale difference sequence. Since $\hatV_{h+1}^\pi \in\cV_{h+1}(L)$ and $\hatsigma_{h}\in\cT_h(L_1,L_2)$, we have $\hatsigma_{h}(s,a)^2\geq \eta_h+\sigmanoise^2$ for all $(s,a)$ and $|\hatV_{h+1}^\pi(s)|\leq H-h+1$ for all $s$. Also by Assumption \ref{assump:linear_MDP} we have $|\epsilon_{k,h}|\leq 1$ almost surely. This then implies 
\begin{align*}
    |e_k|\leq \underbrace{\frac{2(H-h+1)+1}{\eta_h + \sigmanoise^2}\cdot \|\vb_h^\pi\|_{\hat\bLambda_h^{-1}}\cdot\|\bphi(s_{k,h},a_{k,h})\|_{\hat\bLambda_h^{-1}}}_{c_{h,k}} \ ,
\end{align*}
and 
\begin{align*}
    \Var(e_k|F_h,\cF_k)&=\left[(\vb_h^\pi)^\top\hat\bLambda_h^{-1} \frac{\bphi(s_{k,h},a_{k,h})}{\hatsigma_{h}(s_{k,h},a_{k,h})}\right]^2\cdot \EE\left[\left( \frac{\PP_h\hat V_{h+1}^\pi(s_{k,h},a_{k,h})-\hat V_{h+1}^\pi(s'_{k,h})-\epsilon_{k,h}}{\hatsigma_{h}(s_{k,h},a_{k,h})}\right)^2 \middle| F_h,\cF_k \right]\\
    &\leq \left[ (\vb_h^\pi)^\top\hat\bLambda_h^{-1}  \frac{\bphi(s_{k,h},a_{k,h})\bphi(s_{k,h},a_{k,h})^\top}{\hatsigma_{h}(s_{k,h},a_{k,h})^2} \hat\bLambda_h^{-1}\vb_h^\pi \right]
    \\ &  \quad \cdot \left( \frac{\VV_h \hatV_{h+1}^\pi(s_{k,h},a_{k,h}) + \sigmanoise^2}{\max\left\{\eta_h, \VV_h \hatV_{h+1}^\pi(s_{k,h},a_{k,h})\right\}+\sigmanoise^2-\frac{C (H-h+1)^2\sqrt{d}}{\sqrt{K}}} \right),
\end{align*}where the last step is from Assumption \ref{assump:linear_MDP} that $\epsilon_{k,h}$ is independent random noise satisfying $\Var[\epsilon_{k,h}\mid s_{k,h},a_{k,k}]\leq \sigmanoise^2$, and \eqref{eq:E2 bound component fix sigma lemma assump}. Denote $c_h = \max_{k\in[K]} \{c_{h,k}\}$. We then have
\begin{align}\label{eq:E2 bound component constant c_h}
    c_h \leq \frac{2(H-h+1)+1}{\eta_h + \sigmanoise^2}\cdot \|\vb_h^\pi\|_{\hat\bLambda_h^{-1}}\cdot\|\hat\bLambda_h^{-1}\|^{1/2} .    
\end{align}
For simplicity, denote 
\begin{align*}
        b_{h,k} & \coloneqq  \frac{\VV_h \hatV_{h+1}^\pi(s_{k,h},a_{k,h}) + \sigmanoise^2}{\max\left\{\eta_h, \VV_h \hatV_{h+1}^\pi(s_{k,h},a_{k,h})\right\}+\sigmanoise^2-\frac{C (H-h+1)^2\sqrt{d}}{\sqrt{K}}}
        \\ & \leq 1 + \frac{\frac{C (H-h+1)^2\sqrt{d}}{\sqrt{K}}}{\eta_h+\sigmanoise^2-\frac{C (H-h+1)^2\sqrt{d}}{\sqrt{K}}}
        \\ & \sim 1 + \tilde\cO(1/\sqrt{K}) ,
\end{align*}and $b_h \coloneqq \max_k\{ b_{h,k}  \}$. Therefore, we further have
\begin{align*}
    \sum_{k=1}^K\Var(e_k|F_h,\cF_k) &\leq (\vb_h^\pi)^\top\hat\bLambda_h^{-1} \left(\sum_{k=1}^K\frac{\bphi(s_{k,h},a_{k,h})\bphi(s_{k,h},a_{k,h})^\top}{\hatsigma_h(s_{k,h},a_{k,h})^2}\right)\hat\bLambda_h^{-1}\vb_h^\pi \cdot b_h\\
    &= (\vb_h^\pi)^\top \hat\bLambda_h^{-1}\left(\hat\bLambda_h-\lambda\Ib_d\right)\hat\bLambda_h^{-1}\vb_h^\pi \cdot b_h\\
    &\leq b_h \cdot \|\vb_h^\pi\|_{\hat\bLambda_h^{-1}}^2,
\end{align*}since $(\hat\bLambda_h)^{-1/2} (\hat\bLambda_h - \lambda \Ib_d) (\hat\bLambda_h)^{-1/2}$ is a contraction. Then by Freedman's inequality \ref{thm:freedman}, we have
\begin{align*}
    \PP\left(\left|\sum_{k=1}^Ke_k\right|\geq \epsilon\bigg|F_h\right) \leq 2\exp\left(-\frac{\epsilon^2/2}{b_h\|\vb_h^\pi\|_{\hat\bLambda_h^{-1}}^2+c_h\epsilon/3}\right),
\end{align*} since $b_h$ and $c_h$ are fixed once we condition on $\hatV_{h+1}^\pi$, $\hatsigma_h$ and $F_h$. It follows that with conditional (on $\hatV_{h+1}^\pi, \hatsigma_h, F_h$) probability  at least $1-\delta$, 
\begin{align}\label{eq:E2 component event 1}
    \left|\sum_{k=1}^Ke_k\right|\leq \sqrt{2\log\left(\frac{2}{\delta}\right)b_h}\cdot\|\vb_h^\pi\|_{\hat\bLambda_h^{-1}}+\frac{2}{3}\log\frac{2}{\delta}\cdot c_h . 
\end{align}

Define the matrix $\Gb_h$ as the conditional expectation given as 
\begin{align}\label{eq:E2 bound component matrix Gh}
    \Gb_h & \coloneqq \EE_h \left[  \frac{\bphi(s,a)\bphi(s,a)^\top}{\hatsigma_h(s,a)^2}  \middle| \hatsigma_h \right] ,
\end{align}by recalling the notation $\EE_{h}[f(s,a)]=\int_{\cS\times\cA}f(s,a)\text{d}\nu_h(s,a)$ for any function $f$ on $\cS\times\cA$, with $\nu_h(\cdot,\cdot)$ being the occupancy measure of the MDP for stage $h$ induced by the behavior policy $\bar\pi$. Now, since conditioning on $\hatV_{h+1}^\pi$ and $\hatsigma_h$ does not change the distribution of $F_h$, by Lemma \ref{lemma: quadratic form}, if $K$ satisfies
\begin{align}\label{eq:K lower bound E2 component fix sigma}
        K\geq\max\left\{512(\eta_h+\sigmanoise^2)^{-2}\|\Gb_h^{-1}\|^2\log\left(\frac{2d}{\delta}\right),4\lambda\|\Gb_h^{-1}\|\right\},
\end{align}
then over the space of $F_h$, there exists an event $\cE_h$ such that $\PP(\cE_h)\geq 1-\delta$ and for all $F_h\in\cE_h$ we have
\begin{align}\label{eq:E2 bound component fixed sigma u op norm 1}
    \|\ub\|_{\hat\bLambda_h^{-1}}\leq \frac{2}{\sqrt{K}} \cdot \|\ub\|_{\Gb_h^{-1}}
\end{align}
for all $\ub\in\RR^d$. Combining  \eqref{eq:E2 bound component constant c_h}, \eqref{eq:E2 component event 1} and \eqref{eq:E2 bound component fixed sigma u op norm 1}, we conclude that, with conditional probability (on $\hatV_{h+1}^\pi, \hatsigma_h$ only) at least $1-2\delta$, 
\begin{align*}
   \left|\sum_{k=1}^Ke_k\right|
    \leq & \sqrt{2\log\left(\frac{2}{\delta}\right)B_h}\cdot \frac{2}{\sqrt{K}}\cdot\|\vb_h^\pi\|_{\Gb_h^{-1}}+ \frac{2}{3} \log \frac{2}{\delta} \cdot \frac{2(H-h+1)+1}{\eta_h+\sigmanoise^2} \cdot \|\vb_h^\pi\|_{\Gb_h^{-1}} \cdot \left\| \Gb_h^{-1}\right\|^{1/2}\cdot \frac{4}{K},
\end{align*}where $B_h$ is a $\hatV_{h+1}^\pi$-dependent constants given by
\begin{align*}
        B_h & = \max_{(s,a)\sim \nu_h} \frac{\VV_h \hatV_{h+1}^\pi(s,a) + \sigmanoise^2}{\max\left\{\eta_h, \VV_h \hatV_{h+1}^\pi(s,a)\right\}+\sigmanoise^2-\frac{C (H-h+1)^2\sqrt{d}}{\sqrt{K}}}
        \\ & \sim 1 + \tilde\cO(1/\sqrt{K}),
\end{align*}and $\Gb_h$ is a $\hatsigma_h$-dependent matrix given by \eqref{eq:E2 bound component matrix Gh}. Replacing $\delta$ with $\delta/2$ finishes the proof.
\end{proof}

In the next lemma, we relax the conditioning on $\hatsigma_h$ and condition on $\hatV_{h+1}^\pi$ only.

\begin{lemma}\label{lem:E2 bound component fix V}
For any $h\in[H]$, condition on $\hat V_{h+1}^\pi\in\cV_{h+1}(L)$ being fixed and satisfying $\sup_s|\hatV_{h+1}^\pi(s)-V_{h+1}^\pi(s)|\leq \rho$ for some $\rho \geq 0$, if $K$ satisfies \eqref{eq:K lower bound E2 component fix V 1} and
\begin{align}\label{eq:K lower bound E2 component fix V 2}
    K \geq \max \left\{ \frac{911}{(\eta_h+\sigmanoise^2)^2 \iota_h^2}\cdot\log\left(\frac{4d}{\delta}\right) \ , \  \frac{6\lambda}{\iota_h} \right\},
\end{align}
then with conditional probability at least $1-\delta$, we have
\begin{align*}
         |E_{2,h}|
        & \leq  2\sqrt{2\log\left(\frac{8}{\delta} \right)B_h} \cdot \left\| \vb_h^\pi\right\|_{\bLambda_h^{-1}} \cdot \frac{1}{\sqrt{K}}
        \\ &\quad + 2\sqrt{2\log\left(\frac{8}{\delta} \right)B_h} \cdot \sqrt{ C_0 C_1 \cdot \frac{1}{\iota_h}} \cdot \left\| \vb_h^\pi\right\|_{\bLambda_h^{-1}}\cdot (K^{1/4}\sqrt{\tilde\rho})\cdot\frac{1}{K^{3/4}}
        \\ &\quad + \frac{8}{3}\log\left( \frac{8}{\delta}\right)\cdot \frac{2(H-h+1)+1}{\eta_h+\sigmanoise^2} \sqrt{C_1} \cdot \frac{1}{\sqrt{\iota_h}} \cdot \left\| \vb_h^\pi\right\|_{\bLambda_h^{-1}}\cdot \frac{1}{K}
        \\ &\quad  + \frac{8}{3}\log\left( \frac{8}{\delta}\right)\cdot \frac{2(H-h+1)+1}{\eta_h+\sigmanoise^2} \cdot \sqrt{C_0} \cdot C_1 \cdot \frac{1}{{\iota_h}} \cdot \left\| \vb_h^\pi\right\|_{\bLambda_h^{-1}}\cdot \sqrt{\tilde\rho} \cdot\frac{1}{K}
\end{align*}where
\begin{align*}
        C_0 & = \left( \frac{\left\|\bLambda_h\right\|}{\iota_h} \right)^{1/2} , 
        \\C_1 & = \frac{1}{1- \tilde\rho/\iota_h}, 
        \\ \tilde\rho & = \frac{1}{(\eta_h+\sigmanoise^2)^2} \cdot \left( \frac{C_{K,h,\delta} (H-h+1)^2\sqrt{d}}{\sqrt{K}} + 4(H-h+1)\cdot\rho \right), 
        \\  C_{K,h,\delta} & = 12\sqrt{2} \cdot \frac{1}{\sqrt{\kappa_h}} \cdot \left[ \frac{1}{2} \log\left( \frac{\lambda+K}{\lambda}\right) + \frac{1}{d} \log\frac{8}{\delta} \right]^{1/2}+ 12\lambda\cdot \frac{1}{\kappa_h}.
\end{align*} and $B_h$ is a $\hatV_{h+1}^\pi$-dependent constant given by
\begin{align*}
    B_h = \max_{(s,a)\sim \nu_h} \frac{\VV_h \hatV_{h+1}^\pi(s,a) + \sigmanoise^2}{\max\left\{\eta_h, \VV_h \hatV_{h+1}^\pi(s,a)\right\}+\sigmanoise^2-\frac{C_{K,h,\delta} (H-h+1)^2\sqrt{d}}{\sqrt{K}}}. 
\end{align*}
\end{lemma}

\begin{remark}
Conditioning on the event in Theorem \ref{thm:uniform convergence}, we have $\rho \sim \tilde\cO(1/\sqrt{K})$ and thus $\tilde\rho \sim\cO(1/\sqrt{K})$, which means the term $\sqrt{K}\tilde\rho$ is a constant up to a logarithmic factor. This indicates that in the upper bound of $|E_{2,h}|$, only the first term is of order $\tilde\cO(1/\sqrt{K})$ . 
\end{remark}

\begin{proof}[Proof of Lemma \ref{lem:E2 bound component fix V}]
For simplicity, denote the function $\sigma_V(\cdot,\cdot)$ as 
\begin{align*}
        \sigma_V(\cdot,\cdot) \coloneqq \sqrt{ \max\left\{ \eta_h, \ \VV_h \hatV_{h+1}^\pi (\cdot,\cdot)\right\} + \sigmanoise^2 },
\end{align*}and recall that $\hatsigma_h(\cdot,\cdot)$ is an estimator for $\sigma_V(\cdot,\cdot)$ generated using the dataset $\check \cD_h$. Also recall the definition
\begin{align*}
    \sigma_h(\cdot,\cdot) \coloneqq \sqrt{ \max\left\{ \eta_h , \ \VV_h V_{h+1}^\pi (\cdot,\cdot)\right\} + \sigmanoise^2}. 
\end{align*}
First of all, by Lemma \ref{lem:hatsigma concentration fixed V uniform}, with probability at least $1-\delta$ over the space of $\check \cD_h$, the following event happens:
\begin{align}\label{eq:E2 bound component fix V eq 1}
    \sup_{s,a}|\hatsigma_h^2(s,a) - \sigma_V^2(s,a)| \leq \frac{C_{K,h,\delta} (H-h+1)^2\sqrt{d}}{\sqrt{K}},
\end{align} where $C_{K,h,\delta}$ is given by 
\begin{align*}
    & C_{K,h,\delta} = 12\sqrt{2} \cdot \frac{1}{\sqrt{\kappa_h}} \cdot \left[ \frac{1}{2} \log\left( \frac{\lambda+K}{\lambda}\right) + \frac{1}{d} \log\frac{4}{\delta} \right]^{1/2}+ 12\lambda \cdot \frac{1}{\kappa_h}.
\end{align*}Denote the above event of $\hatsigma_h$ by $\cE_{\hatsigma}$. For each fixed $\hatsigma_h \in \cE_{\hatsigma}$, we can then apply Lemma \ref{lem:E2 bound component fix sigma} with $C$ replace by $C_{K,h,\delta}$. This gives that, for any $h$, condition on $\hatV_{h+1}^\pi$ and $\hatsigma_h$, with probability at least $1-\delta$,
\begin{align}\label{eq:E2 bound component fix V eq bound 1}
   & |E_{2,h}| \notag
   \\ & \leq 2\sqrt{2\log\left(\frac{4}{\delta}\right)B_h}\cdot \frac{1}{\sqrt{K}}\cdot\|\vb_h^\pi\|_{\Gb_h^{-1}}+ \frac{8}{3} \log \left(\frac{4}{\delta}\right) \cdot \frac{2(H-h+1)+1}{\eta_h+\sigmanoise^2} \cdot \|\vb_h^\pi\|_{\Gb_h^{-1}} \cdot \left\| \Gb_h^{-1}\right\|^{1/2}\cdot \frac{1}{K} , 
\end{align}where 
\begin{align*}
    B_h & \coloneqq \max_{(s,a)\sim \nu_h} \frac{\VV_h \hatV_{h+1}^\pi(s,a) + \sigmanoise^2}{\max\left\{\eta_h, \VV_h \hatV_{h+1}^\pi(s,a)\right\}+\sigmanoise^2-\frac{C_{K,h,\delta} (H-h+1)^2\sqrt{d}}{\sqrt{K}}}
     \sim 1 + \tilde\cO(1/\sqrt{K}),
    \\ \Gb_h & \coloneqq  \EE_h \left[  \frac{\bphi(s,a)\bphi(s,a)^\top}{\hatsigma^2_h(s,a)}  \middle| \hatsigma_h \right].
\end{align*}
However, note that in the upper bound of $|E_{2,h}|$, the $\|\vb_h^\pi\|_{\Gb_h^{-1}}$ and $\left\| \Gb_h^{-1}\right\|^{1/2}$ term are $\hatsigma_h$-dependent, and so is the lower bound of the sample complexity given by \eqref{eq:K lower bound E2 component fix sigma}.
Therefore, it remains to derive a uniform upper bound of $E_{2,h}$ for all $\hatsigma_h \in \cE_{\hatsigma}$, and a uniform lower bound of $K$.

To get this, first note that since $\hatV_{h+1}^\pi$, $V_{h+1}^\pi\in\cV_{h+1}(L)$ and  $\sup_s|\hatV_{h+1}^\pi(s)-V_{h+1}^\pi(s)|\leq \rho$, we have 
\begin{align*}
    \sup_{s,a}|\sigma_V^2(s,a) - \sigma_h^2(s,a)| \leq 4(H-h+1)\rho. 
\end{align*}Using triangular inequality and \eqref{eq:E2 bound component fix V eq 1} gives
\begin{align*}
    \sup_{s,a}|\hatsigma_h^2(s,a) - \sigma_h^2(s,a)| \leq \frac{C_{K,h,\delta} (H-h+1)^2\sqrt{d}}{\sqrt{K}} + 4(H-h+1)\cdot\rho ,
\end{align*}for all $\hatsigma_h \in \cE_{\hatsigma}$.

Note that by definition, 
\begin{align*}
         \left\| \Gb_h - \bLambda_h \right\|
        &=  \left\| \EE_h\left[ \frac{\bphi(s,a)\bphi(s,a)^\top}{\hatsigma_h^2(s,a)}\right] - \EE_h\left[ \frac{\bphi(s,a)\bphi(s,a)^\top}{\sigma_h^2(s,a)}\right]  \right\| 
        \\& =  \left\| \EE_h\left[ \bphi(s,a)\bphi(s,a)^\top \frac{\hatsigma_h^2(s,a)-\sigma_h^2(s,a)}{\hatsigma_h^2(s,a)\cdot \sigma_h^2(s,a)}  \right] \right\| 
        \\ &\leq  \frac{1}{(\eta_h+\sigmanoise^2)^{2}} \cdot \left( \frac{C_{K,h,\delta} (H-h+1)^2\sqrt{d}}{\sqrt{K}} + 4(H-h+1)\cdot\rho \right)
        \\& \coloneqq  \tilde\rho
        \\ &\sim  \tilde\cO(1/\sqrt{K}),
\end{align*}where the inequality is from $\|\bphi(\cdot,\cdot)\|\leq 1$ and $|\sigma(\cdot,\cdot)|\geq \sqrt{\eta_h+\sigmanoise^2}$ for all $\sigma(\cdot,\cdot) \in \cT_h$. 
Combine the above inequality with Lemma \ref{lem:matrix basic}, and we have 
\begin{align}\label{eq:E2 bound component fix V Gb bound 1}
        \left\|\Gb_h^{-1} \right\| \leq \frac{\left\| \bLambda_h^{-1}\right\|}{1-\left\| \bLambda_h^{-1}\right\|\cdot \left\| \bLambda_h - \Gb_h \right\|} \leq \frac{\left\| \bLambda_h^{-1}\right\|}{1-\left\| \bLambda_h^{-1}\right\|\cdot \tilde\rho} ,
\end{align}
and by Lemma \ref{lem:matrix basic} again, 
\begin{align}\label{eq:E2 bound component u op norm2}
        \left\| \vb_h^\pi\right\|_{\Gb_h^{-1}} & \leq \left[1+\sqrt{ \left(\left\|\bLambda_h^{-1}\right\|\cdot\left\|\bLambda_h\right\| \right)^{1/2} \cdot \left\|\Gb_h^{-1}\right\| \cdot \tilde\rho} \right] \cdot \left\| \vb_h^\pi\right\|_{\bLambda_h^{-1}} \notag
        \\ & \leq \left[1+\sqrt{ \left(\left\|\bLambda_h^{-1}\right\|\cdot\left\|\bLambda_h\right\| \right)^{1/2} \cdot \frac{\left\| \bLambda_h^{-1}\right\|}{1-\left\| \bLambda_h^{-1}\right\|\cdot \tilde\rho} \cdot \tilde\rho} \right] \cdot \left\| \vb_h^\pi\right\|_{\bLambda_h^{-1}} \notag
        \\ & = \left[1+ \sqrt{\left(\left\|\bLambda_h^{-1}\right\| \right)^{3/2}\cdot\left(\left\|\bLambda_h\right\| \right)^{1/2} \cdot \frac{1}{1-\left\| \bLambda_h^{-1}\right\|\cdot \tilde\rho} \cdot \tilde\rho } \right] \cdot \left\| \vb_h^\pi\right\|_{\bLambda_h^{-1}} \notag
        \\ & = \left\| \vb_h^\pi\right\|_{\bLambda_h^{-1}} + \sqrt{\left(\left\|\bLambda_h^{-1}\right\| \right)^{3/2}\cdot\left(\left\|\bLambda_h\right\| \right)^{1/2} \cdot \frac{1}{1-\left\| \bLambda_h^{-1}\right\|\cdot \tilde\rho} } \cdot \sqrt{\tilde\rho} \cdot \left\| \vb_h^\pi\right\|_{\bLambda_h^{-1}} . 
\end{align} Note that the above holds when $K$ is sufficiently large such that $ \left\| \bLambda_h^{-1}\right\|\cdot \tilde\rho$ is less than, for example, 
\begin{align}\label{eq:K lower bound E2 component fix V 1}
    \left\| \bLambda_h^{-1}\right\|\cdot \tilde\rho \leq 1/4.
\end{align}  
We are now ready to derive an upper bound independent of $\hatsigma_h$. 
First define
\begin{align*}
        C_0  = \left(\left\|\bLambda_h^{-1}\right\| \right)^{1/2}\cdot\left(\left\|\bLambda_h\right\| \right)^{1/2}, \quad 
        C_1  = \frac{1}{1-\left\| \bLambda_h^{-1}\right\|\cdot \tilde\rho} \ .
\end{align*}Then we have $ \left\| \vb_h^\pi\right\|_{\Gb_h^{-1}} \leq \left\| \vb_h^\pi\right\|_{\bLambda_h^{-1}} + \sqrt{C_0 C_1 \left\|\bLambda_h^{-1}\right\| \cdot \tilde\rho} \ \cdot \left\| \vb_h^\pi\right\|_{\bLambda_h^{-1}}$, and $ \left\|\Gb_h^{-1}\right\|^{1/2}\leq\sqrt{C_1}\cdot \left\|\bLambda_h^{-1}\right\|^{1/2}$. It follows that 
\begin{align*}
        \left\| \vb_h^\pi\right\|_{\Gb_h^{-1}} \cdot \left\|\Gb_h^{-1}\right\|^{1/2} \leq \sqrt{C_1}\cdot \left\| \bLambda_h^{-1}\right\|^{1/2}\cdot\left\| \vb_h^\pi\right\|_{\bLambda_h^{-1}} +  C_1\cdot \left\| \bLambda_h^{-1}\right\| \cdot\sqrt{C_0} \cdot \left\| \vb_h^\pi\right\|_{\bLambda_h^{-1}}\cdot \sqrt{\tilde\rho} . 
\end{align*}
Plug into \eqref{eq:E2 bound component fix V eq bound 1}, and we have that, condition on $\hatV_{h+1}^\pi$, with probability at least $1-2\delta$, 
\begin{align*}
         |E_{2,h}|
        &\leq  2\sqrt{2\log\left(\frac{4}{\delta} \right)B_h} \cdot \left\| \vb_h^\pi\right\|_{\bLambda_h^{-1}} \cdot \frac{1}{\sqrt{K}}
        \\ &\quad  + 2\sqrt{2\log\left(\frac{4}{\delta} \right)B_h} \cdot  \sqrt{C_0 C_1 \left\|\bLambda_h^{-1}\right\|} \cdot \left\| \vb_h^\pi\right\|_{\bLambda_h^{-1}}\cdot \sqrt{\tilde\rho} \cdot\frac{1}{\sqrt{K}}
        \\ &\quad + \frac{8}{3}\log\left( \frac{4}{\delta}\right)\cdot \frac{2(H-h+1)+1}{\eta_h+\sigmanoise^2} \cdot \sqrt{C_1} \cdot \left\|\bLambda_h^{-1}\right\|^{1/2} \cdot \left\| \vb_h^\pi\right\|_{\bLambda_h^{-1}}\cdot \frac{1}{K}
        \\ &\quad + \frac{8}{3}\log\left( \frac{4}{\delta}\right)\cdot \frac{2(H-h+1)+1}{\eta_h+\sigmanoise^2} \cdot C_1 \cdot \left\|\bLambda_h^{-1}\right\|\cdot \sqrt{C_0} \cdot \left\| \vb_h^\pi\right\|_{\bLambda_h^{-1}}\cdot \sqrt{\tilde\rho} \cdot\frac{1}{K}
\end{align*}
Replacing $\delta$ by $\delta/2$ and using $1/\iota_h = \|\bLambda_h^{-1} \|$ gives the desired upper bound. It remains to show the lower bound. By \eqref{eq:K lower bound E2 component fix sigma}, \eqref{eq:E2 bound component fix V Gb bound 1}, and \eqref{eq:K lower bound E2 component fix V 1}, a uniform version of \ref{eq:K lower bound E2 component fix sigma} is given by 
\begin{align}\label{eq:E2 component K lower bound 1}
    K & \geq \max \left\{512(\eta_h+\sigmanoise^2)^{-2}\cdot \frac{16}{9}\|\bLambda_h^{-1}\|^2\log\left(\frac{4d}{\delta}\right),4\lambda\cdot\frac{4}{3}\|\bLambda_h^{-1}\|\right\} \notag 
    \\ & > \max \left\{ \frac{911}{(\eta_h+\sigmanoise^2)^2 \iota_h^2}\cdot\log\left(\frac{4d}{\delta}\right) \ , \  \frac{6\lambda}{\iota_h} \right\}.
\end{align}
\end{proof}

\begin{lemma}\label{lem:E2 bound}
If $K$ satisfies \eqref{eq:K lower bound E2 bound 1}, \eqref{eq:K lower bound E2 1}, \eqref{eq:K lower bound E2 2} and  
\begin{align}\label{eq:K lower bound E2 bound 3}
    K \geq \max_{h\in[H]} \max \left\{\frac{911}{(\eta_h+\sigmanoise^2)^2 \iota_h^2}\log\left(\frac{8Hd}{\delta}\right),\frac{6\lambda}{\iota_h}\right\},
\end{align}
then with probability at least $1-\delta$, we have
\begin{align*}
        |E_2| &\leq  2\sqrt{2\log\left( \frac{16H}{\delta} \right) B }  \cdot \left[ \sum_{h=1}^H \left\| \vb_h^\pi \right\|_{\bLambda_h^{-1}} \right] \cdot \frac{1}{\sqrt{K}}
        \\ & \qquad + \frac{16\sqrt{2}}{3}\log\left(\frac{16H}{\delta}\right)\cdot \sqrt{B} \cdot \left[ \sum_{h=1}^H A_1(h) \right] \cdot \frac{1}{K^{3/4}} 
        \\ & \qquad + \frac{16\sqrt{2}}{3}\log\left(\frac{16H}{\delta}\right)\cdot \sqrt{B} \cdot \left[\sum_{h=1}^H\left( A_2(h)+A_3(h)\right)\right]\cdot \frac{1}{K} , 
\end{align*}where
\begin{align*}
        A_1(h)  & = \sqrt{C_0(h) \cdot \frac{1}{\iota_h}} \cdot \left(K^{1/4}\sqrt{\tilde\rho(h)} \right) \cdot \left\| \vb_h^\pi\right\|_{\bLambda_h^{-1}} , 
        \\ A_2(h) & = \frac{2(H-h+1)+1}{\eta_h+\sigmanoise^2} \cdot \frac{1}{\sqrt{\iota_h}} \cdot \left\| \vb_h^\pi\right\|_{\bLambda_h^{-1}} , 
        \\ A_3(h) & = \frac{2(H-h+1)+1}{\eta_h+\sigmanoise^2} \cdot \sqrt{C_0(h)} \cdot \sqrt{\tilde\rho}\cdot \frac{1}{\iota_h} \cdot \left\| \vb_h^\pi\right\|_{\bLambda_h^{-1}} ,
        \\ C_0(h) & = \left(\frac{\left\|\bLambda_h\right\|}{\iota_h} \right)^{1/2} ,
        \\ \tilde\rho(h) & = \frac{1}{(\eta_h+\sigmanoise^2)^2} \cdot \left( \frac{C_{K,h,\delta} (H-h+1)^2\sqrt{d}}{\sqrt{K}} + 4(H-h+1)\cdot \tilde C (h)\cdot \frac{d}{\sqrt{K}} \right) , 
        \\  C_{K,h,\delta} & = 12\sqrt{2} \cdot \frac{1}{\sqrt{\kappa_h}} \cdot \left[ \frac{1}{2} \log\left( \frac{\lambda+K}{\lambda}\right) + \frac{1}{d} \log\frac{16H}{\delta} \right]^{1/2}+ 12\lambda\cdot \frac{1}{\kappa_h} ,
        \\  \tilde C(h) & = C\cdot C_{h,2}\cdot \log \left( \frac{dH^2K}{\kappa \delta}\right) , 
\end{align*}and $C$ is some universal constant, $B$ is a problem-dependent constant given by
\begin{align*}
        B & = \max_{h\in[H]}\max_{V\in\cV_{h+1}(L)} \max_{(s,a)\sim \nu_h} \frac{\VV_h V(s,a) + \sigmanoise^2}{\max\left\{\eta_h, \VV_h V(s,a)\right\}+\sigmanoise^2-\frac{C_{K,h,\delta} (H-h+1)^2\sqrt{d}}{\sqrt{K}}},
\end{align*}and $C_{h,2}$ are the same constants as in Theorem \ref{thm:uniform convergence}. 
\end{lemma}

\begin{proof}[Proof of Lemma \ref{lem:E2 bound}]
First, by Theorem \ref{thm:uniform convergence}, if $K$ satisfies
\begin{align}\label{eq:K lower bound E2 bound 1}
    K\geq C\cdot \frac{H^2d^2}{\kappa^2}\log\left(\frac{dHK}{\kappa\delta}\right)\cdot\max_{h\in[H]}\frac{(H-h+1)^2}{(\eta_h+\sigmanoise^2)^2} \cdot\max_{h\in[H]}\frac{(H-h+1)^2}{\iota_h(\eta_h+\sigmanoise^2)},
\end{align}for some problem-independent constant $C$, then with probability at least $1-\delta/2$, for all $h\in [H]$, we have 
\begin{align*}
    \sup_s \left| \hatV_{h+1}^\pi(s) - V_{h+1}^\pi(s) \right| \leq \tilde C \cdot\frac{d}{\sqrt{K}}  , 
\end{align*} where
\begin{align}\label{eq:E2 bound proof tilde C}
        \tilde C  \coloneqq C\cdot C_{h,2}\cdot \log \left( \frac{dH^2K}{\kappa \delta}\right) + C\cdot C_{h,1} \cdot \frac{H}{\sqrt{dK}},
\end{align} and 
\begin{align*}
    C_{h,1} = \sum_{i=h}^H \frac{1}{\iota_h} , \qquad C_{h,2} = \sum_{i=h}^H \frac{H-h+1}{\sqrt{\iota_h(\eta_h+\sigmanoise^2)}}.
\end{align*}For simplicity, we define 
\begin{align*}
        \tilde C(h) = C\cdot C_{h,2}\cdot \log \left( \frac{dH^2K}{\kappa \delta}\right) , 
\end{align*}for some different constant $C$, since the first term on the RHS of \eqref{eq:E2 bound proof tilde C} is much larger than the second one by using $\eta_h \leq (H-h+1)^2$. 

Now we can combine Theorem \ref{thm:uniform convergence} and Lemma \ref{lem:E2 bound component fix V} with the parameter $\rho$ replaced by $\tilde C \cdot \frac{d}{\sqrt{K}}$, take a union bound over all $H$ terms, and conclude that, with probability at least $1-\delta$, the result of Lemma \ref{lem:E2 bound component fix V} holds for all $h\in[H]$ : 
\begin{align}\label{eq:E2 bound E2h bound 1}
         |E_{2,h}|
         &\leq  2\sqrt{2\log\left(\frac{16H}{\delta} \right)B_h} \cdot \left\| \vb_h^\pi\right\|_{\bLambda_h^{-1}} \cdot \frac{1}{\sqrt{K}} \notag
        \\ &\quad  + 2\sqrt{2\log\left(\frac{16H}{\delta} \right)B_h} \cdot  \sqrt{C_0 C_1 \cdot \frac{1}{\iota_h}} \cdot \left\| \vb_h^\pi\right\|_{\bLambda_h^{-1}}\cdot (K^{1/4}\sqrt{\tilde\rho(h)})\cdot\frac{1}{K^{3/4}} \notag
        \\ & \quad + \frac{8}{3}\log\left( \frac{16H}{\delta}\right)\cdot \frac{2(H-h+1)+1}{\eta_h+\sigmanoise^2} \cdot \sqrt{C_1} \cdot \frac{1}{\sqrt{\iota_h}} \cdot \left\| \vb_h^\pi\right\|_{\bLambda_h^{-1}}\cdot \frac{1}{K} \notag
        \\ & \quad + \frac{8}{3}\log\left( \frac{16H}{\delta}\right)\cdot \frac{2(H-h+1)+1}{\eta_h+\sigmanoise^2} \cdot \sqrt{C_0} \cdot C_1 \cdot \frac{1}{\iota_h} \cdot \left\| \vb_h^\pi\right\|_{\bLambda_h^{-1}}\cdot \sqrt{\tilde\rho(h)}\cdot\frac{1}{K}
\end{align}
where for each $h\in[H]$, $\tilde\rho(h)$ is given as
\begin{align*}
        \tilde\rho (h) & = \frac{1}{(\eta_h+\sigmanoise^2)^2} \cdot \left( \frac{C_{K,h,\delta} (H-h+1)^2\sqrt{d}}{\sqrt{K}} + 4(H-h+1)\cdot \tilde C (h)\cdot \frac{d}{\sqrt{K}} \right), 
        \\  C_{K,h,\delta} & = 12\sqrt{2}  \cdot \frac{1}{\sqrt{\kappa_h}} \cdot \left[ \frac{1}{2} \log\left( \frac{\lambda+K}{\lambda}\right) + \frac{1}{d} \log\frac{16H}{\delta} \right]^{1/2}+ 12\lambda\cdot \frac{1}{\kappa_h}.
\end{align*} Now, by the expression of $B_h$, if $K$ satisfies
\begin{align}
    K & \geq \frac{4C_{K,h,\delta}^2H^4 d}{\sigmanoise^4}
     \geq 1152\cdot\max_{h\in[H]}\frac{1}{\kappa_h^2}\cdot \left[ \frac{1}{2} \log\left( \frac{\lambda+K}{\lambda}\right) + \frac{1}{d} \log\frac{16H}{\delta} \right] \cdot \frac{H^4d}{\sigmanoise^4},\label{eq:K lower bound E2 1}
\end{align}
then we have $B_h \leq 2$ for all $h\in[H]$. Also, by \eqref{eq:K lower bound E2 component fix V 1}, $K$ also needs to be large enough so that 
\begin{align}\label{eq:K lower bound E2 2}
        \max_{h\in[H]}\left\{\left\| \bLambda_h^{-1}\right\|\cdot \tilde\rho(h) \right\} = \max_{h\in[H]}\left\{\tilde\rho(h)/\iota_h \right\} \leq 1/4,
\end{align}which implies $C_1\leq 4/3$ for all $h$. We can then simplify \eqref{eq:E2 bound E2h bound 1} into 
\begin{align*}
         &|E_{2,h}|
        \\ & \leq  2\sqrt{2\log\left(\frac{16H}{\delta} \right)B_h} \cdot \left\| \vb_h^\pi\right\|_{\bLambda_h^{-1}} \cdot \frac{1}{\sqrt{K}} 
        \\ & \qquad + \frac{16\sqrt{2}}{3}\log\left(\frac{16H}{\delta}\right)\cdot \sqrt{B_h} \cdot A_1(h) \cdot \frac{1}{K^{3/4}} 
        \\ & \qquad + \frac{16\sqrt{2}}{3}\log\left(\frac{16H}{\delta}\right)\cdot \sqrt{B_h} \cdot \left[ A_2(h)+A_3(h)\right]\cdot \frac{1}{K} , 
\end{align*}where
\begin{align*}
        A_1(h)  & = \sqrt{C_0(h) \cdot \frac{1}{\iota_h}} \cdot \left(K^{1/4}\sqrt{\tilde\rho(h)} \right) \cdot \left\| \vb_h^\pi\right\|_{\bLambda_h^{-1}} , 
        \\ A_2(h) & = \frac{2(H-h+1)+1}{\eta_h+\sigmanoise^2} \cdot \frac{1}{\sqrt{\iota_h}} \cdot \left\| \vb_h^\pi\right\|_{\bLambda_h^{-1}} , 
        \\ A_3(h) & = \frac{2(H-h+1)+1}{\eta_h+\sigmanoise^2} \cdot \sqrt{C_0(h)} \cdot \sqrt{\tilde\rho}\cdot \frac{1}{\iota_h} \cdot \left\| \vb_h^\pi\right\|_{\bLambda_h^{-1}} .
\end{align*}By denoting 
\begin{align*}
        B & \coloneqq \max_{h\in[H]}\max_{V\in\cV_{h+1}(L)} \max_{(s,a)\sim \nu_h} \frac{\VV_h V(s,a) + \sigmanoise^2}{\max\left\{\eta_h, \VV_h V(s,a)\right\}+\sigmanoise^2-\frac{C_{K,h,\delta} (H-h+1)^2\sqrt{d}}{\sqrt{K}}},
\end{align*}which is less than $2$ by \eqref{eq:K lower bound E2 1},and using $|E_2|\leq \sum_{h=1}^K |E_{2,h}|$, we prove the upper bound. The lower bound of $K$ comes from \eqref{eq:K lower bound E2 bound 1}, \eqref{eq:K lower bound E2 1}, \eqref{eq:K lower bound E2 2} and \eqref{eq:E2 component K lower bound 1} . 

\end{proof}

\subsection{Bounding the $E_1$ Term in the OPE Decomposition}\label{sec:E1 ope}
Consider the term $E_1$ in \eqref{eq:error decomp ope proof}:
\begin{align*}
        |E_1| & \leq \lambda\sum_{h=1}^H \left| (\vb_h^\pi)^\top \hat\bLambda_h^{-1}\int_{\cS}\rbr{V_{h+1}^\pi(s)-\hat V_{h+1}^\pi(s)}\bmu_h(s)\text{d}s \right| 
         \coloneqq \sum_{h=1}^H |E_{1,h}|,
\end{align*}where for each $h \in [H]$,
\begin{align*}
        E_{1,h} \coloneqq \lambda(\vb_h^\pi)^\top \hat\bLambda_h^{-1}\int_{\cS}\rbr{V_{h+1}^\pi(s)-\hat V_{h+1}^\pi(s)}\bmu_h(s)\text{d}s . 
\end{align*}

\begin{lemma}\label{lem:E1 bound}
Under the same event where the result of Lemma \ref{lem:E2 bound} holds, if $K$ satisfies \eqref{eq:K lower bound E2 bound 3}, \eqref{eq:K lower bound E2 bound 1}, \eqref{eq:K lower bound E2 1} and \eqref{eq:K lower bound E2 2}, we have
\begin{align*}
    |E_1| \leq 4\sqrt{2}\lambda\left[\sum_{h=1}^H A_4(h) \right]\cdot \frac{H\sqrt{d}}{K},
\end{align*}where for each $h$,
\begin{align*}
        A_4(h) & =  \frac{H-h+1}{H}\cdot \frac{1}{\sqrt{\iota_h}} \cdot \left\| \vb_h^\pi\right\|_{\bLambda_h^{-1}}\cdot\left[1 + \sqrt{2C_0(h)\cdot \frac{1}{\iota_h}\cdot \tilde\rho(h)} \right],
\end{align*}and $C_0(h)$ and $\tilde\rho(h)$ are same constants as in Lemma \ref{lem:E2 bound}.
\end{lemma}

\begin{proof}[Proof of Lemma \ref{lem:E1 bound}]
By \eqref{eq:E2 bound component fixed sigma u op norm 1} and \eqref{eq:E2 bound component u op norm2} , we have that 
\begin{align}\label{eq:E1 bound component 1}
        \left\| \ub \right\|_{\hat\bLambda_h^{-1}} & \leq \frac{2}{\sqrt{K}} \cdot \left\{ \left\| \ub\right\|_{\bLambda_h^{-1}} + \sqrt{\left(\left\|\bLambda_h^{-1}\right\| \right)^{3/2}\cdot\left(\left\|\bLambda_h\right\| \right)^{1/2} \cdot \frac{1}{1-\left\| \bLambda_h^{-1}\right\|\cdot \tilde\rho(h)} }\cdot \sqrt{\tilde\rho(h)} \cdot \left\| \ub\right\|_{\bLambda_h^{-1}}\right\} \notag
        \\ & = \frac{2}{\sqrt{K}} \cdot \left\{ \left\| \ub\right\|_{\bLambda_h^{-1}} + \sqrt{C_0(h)\cdot \frac{1}{\iota_h} \cdot \frac{1}{1-\left\| \bLambda_h^{-1}\right\|\cdot \tilde\rho(h)} } \cdot \sqrt{\tilde\rho(h)} \cdot \left\| \ub\right\|_{\bLambda_h^{-1}}\right\},
\end{align}for all $\ub\in \RR^d$, where the constants take the same values as given in Lemma \ref{lem:E2 bound}, i.e., 
\begin{align*}
        \tilde\rho(h) & = \frac{1}{(\eta_h+\sigmanoise^2)^2} \cdot \left( \frac{C_{K,h,\delta} (H-h+1)^2\sqrt{d}}{\sqrt{K}} + 4(H-h+1)\cdot \tilde C (h)\cdot \frac{d}{\sqrt{K}} \right) , 
        \\  C_{K,h,\delta} & = 12\sqrt{2}  \cdot \frac{1}{\sqrt{\kappa_h}} \cdot \left[ \frac{1}{2} \log\left( \frac{\lambda+K}{\lambda}\right) + \frac{1}{d} \log\frac{16H}{\delta} \right]^{1/2}+ 12\lambda\cdot \frac{1}{\kappa_h} ,
        \\  \tilde C(h) & = C\cdot C_{h,2}\cdot \log \left( \frac{dH^2K}{\kappa \delta}\right)  , 
\end{align*} with $C_{h,2}$ being the same constant as in Theorem \ref{thm:uniform convergence}. Also, since the result of Lemma \ref{lem:barLambda concentration V class OPE} holds, we have
\begin{align*}
    \left\| \frac{\hat\bLambda_h}{K}  - \bLambda_h \right\| \leq & \frac{4\sqrt{2}}{(\eta_h+\sigmanoise^2) \sqrt{K}}\cdot \left( \log\frac{16Hd}{\delta}\right)^{1/2} + \frac{\lambda}{K} + \tilde\rho(h) \leq 2 \tilde\rho(h),
\end{align*}where the first step is by replacing $\delta$ with $\delta/4H$and the second step is by the choice of $K$ in Lemma \ref{lem:E2 bound}. Then It follows from Lemma \ref{lem:matrix basic} that 
\begin{align}\label{eq:E1 bound component 2}
        \left\| \hat\bLambda_h^{-1}  \right\| & \leq \frac{\left\| (K \bLambda_h)^{-1} \right\|}{1-\left\| (K \bLambda_h)^{-1} \right\|\cdot\left\| \hat\bLambda_h -K \bLambda_h \right\|} \leq \frac{1}{K} \cdot \frac{\left\| \bLambda_h^{-1} \right\|}{1-2\tilde\rho(h)\cdot \left\| \bLambda_h^{-1} \right\|} \leq \frac{2\left\| \bLambda_h^{-1} \right\|}{K},
\end{align} since $2\tilde\rho(h)\cdot \left\| \bLambda_h^{-1} \right\| \leq 1/2$ by \eqref{eq:K lower bound E2 2}. Also, since on the event of Lemma \ref{lem:E2 bound}, we have $|V_{h+1}^\pi(s)-\hat V_{h+1}^\pi(s)|\leq 2(H-h+1)$, Assumption \ref{assump:linear_MDP} then implies
\begin{align*}
    \left\| \int_{\cS}\rbr{V_{h+1}^\pi(s)-\hat V_{h+1}^\pi(s)}\bmu_h(s)\text{d}s \right\|_2 \leq 2(H-h+1)\sqrt{d}.
\end{align*}Together with \eqref{eq:E1 bound component 1} and \eqref{eq:E1 bound component 2} and Cauchy-Schwartz inequality, we conclude that 
\begin{align*}
         |E_{1,h}| & = \left| \lambda(\vb_h^\pi)^\top \hat\bLambda_h^{-1}\int_{\cS}\rbr{V_{h+1}^\pi(s)-\hat V_{h+1}^\pi(s)}\bmu_h(s)\text{d}s \right|
         \\ & \leq \lambda \cdot \frac{4\sqrt{2}(H-h+1)\sqrt{d}\left\| \bLambda_h^{-1}\right\|^{1/2}}{K} \cdot \left\| \vb_h^\pi\right\|_{\bLambda_h^{-1}} \cdot \left\{ 1 + \sqrt{2C_0(h) \cdot \frac{1}{\iota_h}\cdot \tilde\rho(h)} \right\}, 
\end{align*}and thus 
\begin{align*}
    |E_1| \leq 4\sqrt{2}\lambda\frac{H\sqrt{d}}{K}\left[\sum_{h=1}^H A_4(h) \right],
\end{align*}where for each $h$,
\begin{align*}
        A_4(h) \coloneqq \frac{H-h+1}{H}\cdot \frac{1}{\sqrt{\iota_h}} \cdot \left\| \vb_h^\pi\right\|_{\bLambda_h^{-1}}\cdot\left[1 + \sqrt{2C_0(h)\cdot \frac{1}{\iota_h}\cdot \tilde\rho(h)} \right],
\end{align*}and $C_0(h)$ and $\tilde\rho(h)$ are same constants as in Lemma \ref{lem:E2 bound}.

\end{proof}

\subsection{Bounding the $E_3$ Term in the OPE Decomposition}\label{sec:E3 ope}
It remains to bound the term $E_3$ in \eqref{eq:error decomp all} given by:
\begin{align*}
    E_3  \coloneqq \lambda \sum_{h=1}^H (\vb_h^\pi)^\top \hat\bLambda_h^{-1}\wb_h^\pi 
     = \sum_{h=1}^H E_{3,h},
\end{align*}where
$E_{3,h}  = \lambda (\vb_h^\pi)^\top \hat\bLambda_h^{-1}\wb_h^\pi$. Similar to Lemma \ref{lem:E1 bound}, we have the following lemma.
\begin{lemma}\label{lem:E3 bound}
Under the same event where the result of Lemma \ref{lem:E2 bound} and Lemma \ref{lem:E1 bound} holds, if $K$ satisfies \eqref{eq:K lower bound E2 bound 3}, \eqref{eq:K lower bound E2 bound 1}, \eqref{eq:K lower bound E2 1} and \eqref{eq:K lower bound E2 2}, we have
\begin{align*}
    |E_3| \leq 4\sqrt{2} \lambda \left(\sum_{h=1}^H A_5(h)\right) \cdot \frac{H\sqrt{d}}{K},
\end{align*}
where
\begin{align*} 
        A_5(h) & = \frac{1}{\sqrt{\iota_h}} \cdot \left\| \vb_h^\pi\right\|_{\bLambda_h^{-1}} \cdot\left\{ 1 + \sqrt{2C_0(h)\cdot \frac{1}{\iota_h} \cdot \tilde\rho(h)} \right\} ,
\end{align*}and $C_0(h)$ and $\tilde\rho(h)$ are same constants as in Lemma \ref{lem:E2 bound}.
\end{lemma}

\begin{proof}[Proof of Lemma \ref{lem:E3 bound}]
First note that
\begin{align*}
        |E_{3,h}| & \leq  \lambda \cdot \left\| \vb_h^\pi\right\|_{\hat\bLambda_h^{-1}} \cdot \left\| \wb_h^\pi\right\|_{\hat\bLambda_h^{-1}} 
        \\ & \leq \lambda \cdot \left\| \vb_h^\pi\right\|_{\hat\bLambda_h^{-1}} \cdot \left\| \wb_h^\pi\right\|_2 \cdot \left\|\hat\bLambda_h^{-1} \right\|^{1/2}
        \\ & \leq \lambda \cdot \frac{2}{\sqrt{K}} \cdot \left\{ \left\| \vb_h^\pi\right\|_{\bLambda_h^{-1}} + \sqrt{2C_0(h)\cdot\left\| \bLambda_h^{-1}\right\|\cdot \tilde\rho(h)} \cdot \left\| \vb_h^\pi\right\|_{\bLambda_h^{-1}}\right\} \cdot \frac{\sqrt{2}}{\sqrt{K}}\left\| \bLambda_h^{-1}\right\|^{1/2} \cdot 2H\sqrt{d}
        \\ & = 4 \sqrt{2} \lambda\cdot \frac{1}{\sqrt{\iota_h}} \cdot \left\| \vb_h^\pi\right\|_{\bLambda_h^{-1}} \cdot\left\{ 1 + \sqrt{2C_0(h)\cdot \frac{1}{\iota_h} \cdot \tilde\rho(h)} \right\} \cdot \frac{H\sqrt{d}}{K} , 
\end{align*}where the third step is by \eqref{eq:E1 bound component 1}, \eqref{eq:E1 bound component 2} and Proposition \ref{proposition: linear Q}.
We then conclude that
\begin{align*}
        |E_3| \leq 4\sqrt{2} \lambda \left(\sum_{h=1}^H A_5(h)\right) \cdot \frac{H\sqrt{d}}{K} ,
\end{align*}where for each $h \in [H]$,  
\begin{align*}
        A_5(h) & = \frac{1}{\sqrt{\iota_h}} \cdot \left\| \vb_h^\pi\right\|_{\bLambda_h^{-1}} \cdot\left\{ 1 + \sqrt{2C_0(h)\cdot \frac{1}{\iota_h} \cdot \tilde\rho(h)} \right\} . 
\end{align*}
\end{proof}

\subsection{Proof of Theorem \ref{thm:ope general}}

\begin{proof}[Proof of Theorem \ref{thm:ope general}]
By \eqref{eq:ope decomp}, and Lemmas \ref{lem:E2 bound}, \ref{lem:E1 bound} and \ref{lem:E3 bound}, we have that with probability at least $1-\delta$, 
\begin{align}\label{eq:ope proof upper bound 1}
        |v_1^\pi - \hat{v}_1^\pi| \leq  & \sqrt{2\log\left( \frac{16H}{\delta} \right) B }  \cdot \left[ \sum_{h=1}^H \left\| \vb_h^\pi \right\|_{\bLambda_h^{-1}} \right] \cdot \frac{1}{\sqrt{K}} \notag
        \\ & \qquad + \frac{16\sqrt{2}}{3}\log\left(\frac{16H}{\delta}\right)\cdot \sqrt{B} \cdot \left[ \sum_{h=1}^H A_1(h) \right] \cdot \frac{1}{K^{3/4}} \notag
        \\ & \qquad + \frac{16\sqrt{2}}{3}\log\left(\frac{16H}{\delta}\right)\cdot \sqrt{B} \cdot \left[\sum_{h=1}^H\left( A_2(h)+A_3(h)\right)\right]\cdot \frac{1}{K} \notag
        \\ &\qquad  + 4\sqrt{2}\lambda\left[\sum_{h=1}^H \left(A_4(h)+A_5(h)\right) \right]\cdot \frac{H\sqrt{d}}{K} .
\end{align} 
We now compute a lower bound for $K$. This comes from the lower bound of $K$ required by Theorem \ref{thm:uniform convergence}, Lemma \ref{lem:E2 bound}, Lemma \ref{lem:E1 bound} and Lemma \ref{lem:E3 bound}. Recall \eqref{eq:K lower bound E2 2}, \eqref{eq:K lower bound E2 bound 3}, \eqref{eq:K lower bound E2 bound 1} and \eqref{eq:K lower bound E2 1}:
\begin{align}\label{eq:K lower bound in proof all}
\begin{aligned}
        K & \geq \max_{h\in[H]} \frac{16\left\| \bLambda_h^{-1}\right\|^2}{(\eta_h+\sigmanoise^2)^4} \cdot \left( {C_{K,h,\delta} (H-h+1)^2\sqrt{d}} + 4(H-h+1)\cdot \tilde C(h) \cdot d \right)^2 , 
        \\ K & \geq \max_{h\in[H]} \max \left\{\frac{911}{(\eta_h+\sigmanoise^2)^2 \iota_h^2}\log\left(\frac{8Hd}{\delta}\right),\frac{6\lambda}{\iota_h}\right\},
        \\ K & \geq C\cdot \frac{H^2d^2}{\kappa^2}\log\left(\frac{dHK}{\kappa\delta}\right)\cdot\max_{h\in[H]}\frac{(H-h+1)^2}{(\eta_h+\sigmanoise^2)^2} \cdot\max_{h\in[H]}\frac{(H-h+1)^2}{\iota_h(\eta_h+\sigmanoise^2)},
        \\ K & \geq 1152\cdot \left[ \frac{1}{2} \log\left( \frac{\lambda+K}{\lambda}\right) + \frac{1}{d} \log\frac{16H}{\delta} \right] \cdot \frac{H^4d}{\kappa^2 \sigmanoise^4} . 
\end{aligned}
\end{align}
It remains to simplify the expression. For the first lower bound in \eqref{eq:K lower bound in proof all}, note that $\tilde C(h) \geq  C_{K,h,\delta} \cdot H$, and thus
\begin{align}
          {C_{K,h,\delta} (H-h+1)^2\sqrt{d}} + 4(H-h+1)\cdot \tilde C(h) \cdot  d
          <  8(H-h+1)\cdot \tilde C(h) \cdot d.\label{eq:simplify rho tilde term 1}
\end{align}
Therefore, it suffices to let $K$ satisfy
\begin{align}\label{eq:K lower bound in proof 1}
        K \geq \max_{h\in[H]} C^2 \cdot \frac{1024}{\iota_h^2(\eta_h+\sigma^2)^4} \cdot (H-h+1)^2 d^2 \cdot \left[ \sum_{i=h}^H \frac{H-h+1}{\sqrt{\iota_h(\eta_h+\sigmanoise^2)}} \right]^2 \cdot \left[ \log\left(\frac{dH^2K}{\kappa\delta}\right)\right]^2,
\end{align}where $C$ is the problem-independent universal constant from the proof of Theorem \ref{thm:uniform convergence}. 
The second lower bound in \eqref{eq:K lower bound in proof all} is much smaller than \eqref{eq:K lower bound in proof 1} and thus can be omitted. We then consider the third and the fourth lower bound together. They can be combined into 
\begin{align}\label{eq:K lower bound in proof 2}
        K \geq C \cdot \frac{H^2 d^2}{\sigmanoise^4\kappa^2} \cdot \max\left\{ \max_{h\in[H]}\frac{(H-h+1)^2}{(\eta_h+\sigmanoise^2)^2} \cdot\max_{h\in[H]}\frac{(H-h+1)^2}{\iota_h(\eta_h+\sigmanoise^2)}, H^2\right\} \cdot \log\left(\frac{dHK}{\kappa\delta}\right). 
\end{align}
Denote 
\begin{align*}
        C_{h,3} \coloneqq \frac{(H-h+1)^2}{\eta_h+\sigmanoise^2} . 
\end{align*}
Then \eqref{eq:K lower bound in proof 1} is simplified to 
\begin{align}\label{eq:K lower bound in proof bc 1}
        K \geq C \max_{h\in[H]} \frac{C_{h,3} d^2}{ \iota_h^2 (\eta_h+\sigmanoise^2)^3} \cdot \left[ \sum_{i=h}^H \sqrt{\frac{C_{h,3}}{\iota_h}} \right]^2 \cdot \left[ \log\left(\frac{dH^2K}{\kappa\delta}\right)\right]^2 ,
\end{align}and \eqref{eq:K lower bound in proof 2} can be simplified
\begin{align}\label{eq:K lower bound in proof bc 2}
        K \geq C \cdot \frac{H^2 d^2}{\sigmanoise^4 \kappa^2} \cdot \max\left\{ \max_{h\in[H]}\frac{C_{h,3}}{\eta_h+\sigmanoise^2} \cdot\max_{h\in[H]}\frac{C_{h,3}}{\iota_h}, H^2\right\} \cdot \log\left(\frac{dHK}{\kappa\delta}\right) . 
\end{align}
We then combine \eqref{eq:K lower bound in proof bc 1} and \eqref{eq:K lower bound in proof bc 2} and get that 
\begin{align*}
    K \geq C \cdot C_3(h) \cdot d^2 \left[ \log\left(\frac{dH^2K}{\kappa\delta}\right)\right]^2 ,
\end{align*}where $C$ is some problem-independent universal constant and  
\begin{align*}
    C_3(h) & \coloneqq \max\left\{  \max_{h\in[H]} \frac{C_{h,3} }{ \iota_h^2 (\eta_h+\sigmanoise^2)^3} \cdot \left[ \sum_{i=h}^H \sqrt{\frac{C_{h,3}}{\iota_h}} \right]^2 , \frac{H^2}{\sigmanoise^4\kappa^2} \cdot \left( \max_{h\in[H]}\frac{C_{h,3}}{\eta_h+\sigmanoise^2} \cdot\max_{h\in[H]}\frac{C_{h,3}}{\iota_h}\right), \frac{H^4}{\sigmanoise^4\kappa^2}  \right\} .
\end{align*}
To simplify the upper bound given by \eqref{eq:ope proof upper bound 1}, first note that by the choice of $K$, we have $B<2$. By \eqref{eq:simplify rho tilde term 1}, we have that $\tilde\rho(h)$ satisfies
\begin{align}\label{eq:simplify rho tilde term 2}
    \tilde\rho(h) \leq 8C \frac{1}{\sqrt{K}}\cdot \frac{(H-h+1)d}{(\eta_h+\sigmanoise^2)^2} \cdot \left[\sum_{i=h}^H \frac{H-h+1}{\sqrt{\iota_h(\eta_h+\sigmanoise^2)}}\right] \cdot \log\left(\frac{dH^2K}{\kappa\delta}\right)  .
\end{align}It follows that
\begin{align}\label{eq:simplify A1 1}
        A_1(h) \leq C\cdot \sqrt{C_0(h) \cdot \frac{(H-h+1)d}{(\eta_h+\sigmanoise^2)^2} \cdot \left[\sum_{i=h}^H \frac{H-h+1}{\sqrt{\iota_h(\eta_h+\sigmanoise^2)}}\right] \cdot \frac{1}{\iota_h}  \cdot \log\left(\frac{dH^2K}{\kappa\delta}\right)} \cdot \left\| \vb_h^\pi\right\|_{\bLambda_h^{-1}} ,
\end{align}for some (different) universal constant $C$. Also, it is not hard to see $A_2(h)$ and $A_3(h)$ are less than the RHS of \eqref{eq:simplify A1 1} up to a constant factor by our choice of $K$, which gives 
\begin{align}\label{eq:simplify A1 A2 A3}
    & \sum_{h=1}^H A_2(h)+A_3(h) \notag
    \\ & \leq C \sum_{h=1}^H \sqrt{C_0(h) \cdot \frac{(H-h+1)d}{(\eta_h+\sigmanoise^2)^2} \cdot \left[\sum_{i=h}^H \frac{H-h+1}{\sqrt{\iota_h(\eta_h+\sigmanoise^2)}}\right] \cdot \frac{1}{\iota_h}  \cdot \log\left(\frac{dH^2K}{\kappa\delta}\right)} \cdot \left\| \vb_h^\pi\right\|_{\bLambda_h^{-1}} , 
\end{align}for some universal constant $C$. To bound $A_4(h)+A_5(h)$, note that $A_4(h)\leq A_5(h)$ and thus 
\begin{align}\label{eq:simplify A4 A5 1}
        A_4(h) + A_5(h) \leq  \frac{2}{\sqrt{\iota_h}}\cdot \left\| \vb_h^\pi\right\|_{\bLambda_h^{-1}} \cdot \left\{ 1 + \sqrt{2C_0(h)\cdot \frac{1}{\iota_h} \cdot \tilde\rho(h)} \right\},
\end{align}where 
\begin{align*}
        & \frac{2}{\sqrt{\iota_h}}\cdot \left\| \vb_h^\pi\right\|_{\bLambda_h^{-1}} \cdot \sqrt{ 2C_0(h)\cdot \frac{1}{\iota_h} \cdot \tilde\rho(h) }
        \\ & \leq  \frac{C}{K^{1/4}} \sqrt{C_0(h)\frac{(H-h+1)d}{(\eta_h+\sigmanoise^2)^2} \cdot \left[\sum_{i=h}^H \frac{H-h+1}{\sqrt{\iota_h(\eta_h+\sigmanoise^2)}}\right] \cdot \frac{1}{\iota_h^{2}}  \cdot \log\left(\frac{dH^2K}{\kappa\delta}\right)} \cdot \left\| \vb_h^\pi\right\|_{\bLambda_h^{-1}}.
\end{align*}Recall \eqref{eq:ope proof upper bound 1}. By our choice of $K$, it is clear that 
\begin{align*}
        & H\sqrt{d} \cdot \frac{2}{\sqrt{\iota_h}}\cdot \left\| \vb_h^\pi\right\|_{\bLambda_h^{-1}} \cdot \sqrt{ 2C_0(h)\cdot \frac{1}{\iota_h} \cdot \tilde\rho(h) } 
        \\ & \leq C\cdot \sqrt{C_0(h) \cdot \frac{(H-h+1)d}{(\eta_h+\sigmanoise^2)^2} \cdot \left[\sum_{i=h}^H \frac{H-h+1}{\sqrt{\iota_h(\eta_h+\sigmanoise^2)}}\right] \cdot \frac{1}{\iota_h}  \cdot \log\left(\frac{dH^2K}{\kappa\delta}\right)} \cdot \left\| \vb_h^\pi\right\|_{\bLambda_h^{-1}} ,
\end{align*}where the RHS of the above is exactly the RHS of \eqref{eq:simplify A1 1} up to a constant factor. Therefore, we can combine $[\sum_{h=1}^H A_4(h)+A_5(h)]H\sqrt{d}$ with \eqref{eq:simplify A1 A2 A3}, and together with \eqref{eq:simplify A4 A5 1}, the last two terms on the RHS of \eqref{eq:ope proof upper bound 1} can be upper bounded by 
\begin{align}
         &\frac{16\sqrt{2}}{3}\log\left(\frac{16H}{\delta}\right)\cdot \sqrt{B} \cdot \left[\sum_{h=1}^H\left( A_2(h)+A_3(h)\right)\right]\cdot \frac{1}{K}
         + 4\sqrt{2}\lambda\left[\sum_{h=1}^H \left(A_4(h)+A_5(h)\right) \right]\cdot \frac{H\sqrt{d}}{K}  \notag 
        \\ &\leq  C \cdot C_4 \cdot \log\left(\frac{16H}{\delta}\right)\cdot \frac{1}{K} , 
\end{align}where $C$ is some universal constant and $C_4$ is given by 
\begin{align*}
        C_4 & \coloneqq \sum_{h=1}^H \left\{ \sqrt{C_0(h) \cdot \frac{(H-h+1)d}{(\eta_h+\sigmanoise^2)^2} \cdot \left[\sum_{i=h}^H \frac{H-h+1}{\sqrt{\iota_h(\eta_h+\sigmanoise^2)}}\right] \cdot \frac{1}{\iota_h}  \cdot \log\left(\frac{dH^2K}{\kappa\delta}\right)} \cdot \left\| \vb_h^\pi\right\|_{\bLambda_h^{-1}} \right\} 
        \\ & = \sum_{h=1}^H \left\{ \sqrt{C_0(h) \cdot C_{h,2} \cdot \frac{(H-h+1)d}{\iota_h(\eta_h+\sigmanoise^2)^2}\cdot \log\left(\frac{dH^2K}{\kappa\delta}\right)} \cdot \left\| \vb_h^\pi\right\|_{\bLambda_h^{-1}} \right\}.
\end{align*}Plugging in the formula for $C_0(h)$ given in Lemma \ref{lem:E2 bound} finishes the proof. Note that in Theorem \ref{thm:uniform convergence}, the notation $C_0(h)$ is changed to $C_{h,4}$.

\end{proof}

\section{Proof of Error Decomposition}\label{sec: proof of error decomp}

\begin{proof}
Since $Q_h^\pi(s,a)=\bphi(s,a)^\top\wb_h^\pi=r_h(s,a) + [\PP_hV_{h+1}^\pi](s,a)$ for some vector $\wb_h^\pi\in\RR^d$, we further have
\begin{align*}
    Q_h^\pi(s,a)&=\bphi(s,a)^\top\hat\bLambda_h^{-1}\rbr{\sum_{k=1}^K\frac{\bphi(s_{k,h},a_{k,h})\bphi(s_{k,h},a_{k,h})^\top}{\hat\sigma_h(s_{k,h},a_{k,h})^2}+\lambda\Ib_d}\wb_h^\pi\\
    &=\bphi(s,a)^\top\hat\bLambda_h^{-1}\sum_{k=1}^K\frac{\bphi(s_{k,h},a_{k,h})}{\hat\sigma_h(s_{k,h},a_{k,h})^2}Q_h^\pi(s_{k,h},a_{k,h})+\lambda\bphi(s,a)^\top\hat\bLambda_h^{-1}\wb_h^\pi\\
    &=\bphi(s,a)^\top\hat\bLambda_h^{-1}\sum_{k=1}^K\frac{\bphi(s_{k,h},a_{k,h})}{\hat\sigma_h(s_{k,h},a_{k,h})^2}\rbr{r_h(s_{k,h},a_{k,h}) + [\PP_hV_{h+1}^\pi](s_{k,h},a_{k,h})}\\
    &\quad + \lambda\bphi(s,a)^\top\hat{\bLambda}_h^{-1}\wb_h^\pi .
\end{align*}
It follows that
\begin{align*}
    Q_h^\pi(s,a)-\hatQ_h^\pi(s,a)&=\bphi(s_h,a_h)^\top\hat\bLambda_h^{-1}\sum_{k=1}^K\frac{\bphi(s_{k,h},a_{k,h})}{\hat\sigma_h(s_{k,h},a_{k,h})^2}\rbr{[\PP_hV_{h+1}^{\pi}](s_{k,h},a_{k,h})- \hat V_{h+1}^\pi(s_{k,h}') - \epsilon_{k,h}}\\
    &\quad + \lambda\bphi(s_h,a_h)^\top\hat\bLambda_h^{-1}\wb_h^\pi\\
    &=\bphi(s,a)^\top\hat\bLambda_h^{-1}\sum_{k=1}^K\frac{\bphi(s_{k,h},a_{k,h})}{\hat\sigma_h(s_{k,h},a_{k,h})^2}\rbr{[\PP_hV_{h+1}^\pi](s_{k,h},a_{k,h})-[\PP_h\hat V_{h+1}^\pi](s_{k,h},a_{k,h})}\\
    & \quad+ \bphi(s,a)^\top\hat\bLambda_h^{-1}\sum_{k=1}^K\frac{\bphi(s_{k,h},a_{k,h})}{\hat\sigma_h(s_{k,h},a_{k,h})^2}\rbr{[\PP_h\hat V_{h+1}^\pi](s_{k,h},a_{k,h})-\hat V_{h+1}^\pi(s_{k,h}')-\epsilon_{k,h}}\\
    &\quad + \lambda\bphi(s,a)^\top\hat\bLambda_h^{-1}\wb_h^\pi.
\end{align*}
where $\epsilon_{k,h}$ is the noise in reward.
Note that
\begin{align*}
    [\PP_h V_{h+1}^\pi](s_{k,h},a_{k,h})-[\PP_h\hat V_{h+1}^\pi](s_{k,h},a_{k,h})&= \int_{\cS} \rbr{V_{h+1}^\pi(s)-\hat V_{h+1}^\pi(s)}\langle\bphi(s_{k,h},a_{k,h}),\bmu_h(s)\rangle\text{d}s\\
    &=\bphi(s_{k,h},a_{k,h})^\top\int_{\cS}\rbr{V_{h+1}^\pi(s)-\hat V_{h+1}^\pi(s)}\bmu_h(s)\text{d}s,
\end{align*}
and thus
\begin{align}\label{eq:error decomposition Q}
    & Q_h^\pi(s,a)-\hat Q_h^\pi(s,a) \notag
    \\ &=\bphi(s,a)^\top \hat\bLambda_h^{-1}\sum_{k=1}^K \frac{\bphi(s_{k,h},a_{k,h})\bphi(s_{k,h},a_{k,h})^\top}{\hat\sigma(s_{k,h},a_{k,h})^2} \int_{\cS} \rbr{V_{h+1}^\pi(s')-\hat V_{h+1}^\pi(s')}\bmu_h(s')\diff s' \notag \\
    &\quad + \bphi(s,a)^\top\hat\bLambda_h^{-1}\sum_{k=1}^K\frac{\bphi(s_{k,h},a_{k,h})}{\hat\sigma_h(s_{k,h},a_{k,h})^2}\rbr{[\PP_h\hat V_{h+1}^\pi](s_{k,h},a_{k,h})-\hat V_{h+1}^\pi(s_{k,h}')-\epsilon_{k,h}} \notag \\
    &\quad+ \lambda\bphi(s_h,a_h)^\top\hat\bLambda_h^{-1}\wb_h^\pi \notag \\
    &= [\PP_h(V_{h+1}^\pi-\hat V_{h+1}^\pi)](s,a) -\lambda\bphi(s,a)^\top\hat\bLambda_h^{-1}\int_{\cS}\rbr{V_{h+1}^\pi(s)- \hat V_{h+1}^\pi(s)}\bmu_h(s)\text{d}s \notag \\
    &\quad + \bphi(s,a)^\top\hat\bLambda_h^{-1}\sum_{k=1}^K\frac{\bphi(s_{k,h},a_{k,h})}{\hat\sigma_h(s_{k,h},a_{k,h})^2}\rbr{[\PP_h\hat V_{h+1}^\pi](s_{k,h},a_{k,h})-\hat V_{h+1}^\pi(s_{k,h}') - \epsilon_{k,h}} \notag \\
    &\quad + \lambda\bphi(s,a)^\top\hat\bLambda_h^{-1}\wb_h^\pi.
\end{align}
Then by the Bellman equation, we have
\begin{align}
    V_h^\pi(s)-\hat V_h^\pi(s) &= \JJ_h(Q_h^\pi-\hat Q_h^\pi)(s)\notag\\
    &= \JJ_h\PP_h(V_{h+1}^\pi-\hat V_{h+1}^\pi)(s) - \lambda\JJ_h\bphi(s)^\top\hat\bLambda_h^{-1}\int_{\cS}\rbr{V_{h+1}^\pi(s')-\hat V_{h+1}^\pi(s')}\bmu_h(s')\text{d}s'\notag\\
    &\quad + \JJ_h\bphi(s)^\top\hat\bLambda_h^{-1}\sum_{k=1}^K\frac{\bphi(s_{k,h},a_{k,h})}{\hat\sigma_h(s_{k,h},a_{k,h})^2}\rbr{[\PP_h\hat V_{h+1}^\pi](s_{k,h},a_{k,h})-\hat V_{h+1}^\pi(s_{k,h}')-\epsilon_{k,h}}\notag\\
    &\quad + \lambda\JJ_h\bphi(s)^\top\hat\bLambda_h^{-1}\wb_h^\pi,
\end{align}
where $\JJ_h f(\cdot)=\int_{\cA}f(\cdot,a)\pi_h(a|\cdot)\text{d}a$ for any function $f:\cS\times\cA\to\RR$. Recursively expanding the above equation, we obtain
\begin{align}\label{eq:error decomp all}
    &V_1^\pi(s)-\hat V_1^\pi(s) \notag
    \\ &=-\lambda\sum_{h=1}^H\rbr{\prod_{i=1}^{h-1}\JJ_i\PP_i}\JJ_h\bphi(s)^\top\hat\bLambda_h^{-1}\int_{\cS}\rbr{V_{h+1}^\pi(s')-\hat V_{h+1}^\pi(s')}\bmu_h(s')\diff s' \notag \\
    &\quad  + \sum_{h=1}^ H\rbr{\prod_{i=1}^{h-1}\JJ_i\PP_i}\JJ_h\bphi(s)^\top\hat\bLambda_h^{-1}\sum_{k=1}^K\frac{\bphi(s_{k,h},a_{k,h})}{\hat\sigma_h(s_{k,h},a_{k,h})^2}\rbr{[\PP_h\hat V_{h+1}^\pi](s_{k,h},a_{k,h}) - \hat V_{h+1}^\pi(s_{k,h}')-\epsilon_{k,h}} \notag \\ 
    & \quad+ \lambda \sum_{h=1}^H \rbr{\prod_{i=1}^{h-1} \JJ_i \PP_i}\JJ_h \bphi(s_1)^\top\hat\bLambda_h^{-1}\wb_h^\pi.
\end{align}  Here with a slight abuse of notation we define $\prod_{i=1}^{h-1} \JJ_i \PP_i = 1$ when $h=1$.
We then have 
\begin{align}\label{eq:error decomp ope in proof}
    v_1^\pi - \hat{v}_1^\pi & = -\lambda\sum_{h=1}^H (\vb_h^\pi)^\top \hat\bLambda_h^{-1}\int_{\cS}\rbr{V_{h+1}^\pi(s)-\hat V_{h+1}^\pi(s)}\bmu_h(s)\text{d}s\notag\\
    &\quad + \sum_{h=1}^ H(\vb_h^\pi)^\top\hat\bLambda_h^{-1}\sum_{k=1}^K\frac{\bphi(s_{k,h},a_{k,h})}{\hat\sigma_h(s_{k,h},a_{k,h})^2}\rbr{[\PP_h\hat V_{h+1}^\pi](s_{k,h},a_{k,h}) - \hat V_{h+1}^\pi(s_{k,h}')-\epsilon_{k,h}} \notag\\ 
    & \quad + \lambda \sum_{h=1}^H (\vb_h^\pi)^\top \hat\bLambda_h^{-1}\wb_h^\pi\notag\\
    & \coloneqq E_1 + E_2 + E_3 , 
\end{align} where for simplicity we write $\vb_h^\pi = \EE_\pi \left[ \left( \prod_{i=1}^{h-1} \JJ_i \PP_i \right) \JJ_h \bphi(s_1) \bigg| s_1 \sim \xi_1 \right] = \EE_{\pi,h}[\bphi(s_h,a_h)]$ by recalling the definition of $\EE_{\pi,h}[\cdot]$ given in the text following \eqref{eq:def of occupancy measure}. 
\end{proof}

\section{Lemmas for Uniform Convergence}
All lemmas in this section are under the Assumption of Theorem \ref{thm:uniform convergence}.

\subsection{Convergence of $\hat\sigma$}

\begin{lemma}\label{lem:hatsigma component concentration fixed V uniform conv}
For any $h\in[H]$ and any $\hat V_{h+1}^\pi \in \cV_{h+1}(L)$, with probability at least $1-\delta$, it holds for all $(s,a)\in\cS\times\cA$ that
\begin{align*}
    \left| \langle\bphi(s,a),\hatbeta_h\rangle_{[0,(H-h+1)^2]} - \PP_h(\hat V_{h+1})^2(s,a) \right|&\leq  C'_{K,\delta} \cdot \frac{(H-h+1)^2\sqrt{d}}{\sqrt{K}} \left[ \frac{1}{2} \log\left( \frac{\lambda+K}{\lambda}\right) + \frac{1}{d} \log\frac{4}{\delta} \right]^{1/2},
\end{align*}
and
\begin{align*}
    \left| \langle\bphi(s,a),\hattheta_h\rangle_{[0,H-h+1]} - \PP_h(\hat V_{h+1})(s,a) \right|&\leq  C'_{K,\delta} \cdot \frac{(H-h+1)\sqrt{d}}{\sqrt{K}} \left[ \frac{1}{2} \log\left( \frac{\lambda+K}{\lambda}\right) + \frac{1}{d} \log\frac{4}{\delta} \right]^{1/2},
\end{align*}where 
\begin{align*}
    C'_{K,\delta} \coloneqq  4\sqrt{2}  \frac{1}{\sqrt{\kappa_h}} + 4\lambda \cdot \frac{1}{\kappa_h}  \cdot  \left[ \frac{1}{2} \log\left( \frac{\lambda+K}{\lambda}\right) + \frac{1}{d} \log\frac{4}{\delta} \right]^{-1/2} .   
\end{align*}
\end{lemma}

\begin{proof}[Proof of Lemma \ref{lem:hatsigma component concentration fixed V uniform conv}]
First we consider $\langle\bphi(s,a),\hatbeta_h\rangle_{[0,(H-h+1)^2]}$. Note that, since $\PP_h(\hat V_{h+1}^\pi)^2(s,a) \in [0,(H-h+1)^2]$,
\begin{align*}
    |\langle\bphi(s,a),\hatbeta_h\rangle_{[0,(H-h+1)^2]}-\PP_h(\hat V_{h+1}^\pi)^2(s,a)|&\leq |\langle\bphi(s,a),\hatbeta_h\rangle - \PP_h(\hat V_{h+1}^\pi)^2(s,a)|.
\end{align*}
It then suffices to bound the RHS.
\begin{align*}
    & \langle\bphi(s,a),\hatbeta_h\rangle - \PP_h(\hat V_{h+1})^2(s,a) 
    \\ &= \bphi(s,a)^\top(\hat\bSigma_h)^{-1} \sum_{k=1}^K \bphi(\check{s}_{k,h}, \check{a}_{k,h}) \hatV_{h+1}^\pi(\check{s}_{k,h}')^2-\PP_h(\hatV_{h+1}^\pi)^2(s,a)
    \\  &= \bphi(s,a)^\top(\hat\bSigma_h)^{-1} \sum_{k=1}^K \bphi(\check{s}_{k,h}, \check{a}_{k,h}) \hatV_{h+1}^\pi(\check{s}_{k,h}')^2 - \bphi(s,a)^\top \int_\cS (\hatV_{h+1}^\pi)^2 (s') \diff \bmu_h(s'). 
\end{align*}Note that 
\begin{align*}
        & \bphi(s,a)^\top \int_\cS (\hatV_{h+1}^\pi)^2 (s') \diff \bmu_h(s') 
        \\ &=  \bphi(s,a)^\top (\hat\bSigma_h)^{-1} \left(\sum_{k=1}^K \bphi(\check{s}_{k,h},\check{a}_{k,h})\bphi(\check{s}_{k,h},\check{a}_{k,h})^\top + \lambda\Ib_d \right) \int_\cS (\hatV_{h+1}^\pi)^2 (s') \diff \bmu_h(s')
        \\ &=  \bphi(s,a)^\top (\hat\bSigma_h)^{-1} \sum_{k=1}^K \bphi(\check{s}_{k,h},\check{a}_{k,h}) \PP_h(\hatV_{h+1}^\pi)^2(\check{s}_{k,h},\check{a}_{k,h}) + \lambda \bphi(s,a) (\hat\bSigma_h)^{-1} \int_\cS (\hatV_{h+1}^\pi)^2 (s') \diff \bmu_h(s') ,
\end{align*}and it follows that 
\begin{align*}
        & \langle\bphi(s,a),\hatbeta_h\rangle - \PP_h(\hat V_{h+1})^2(s,a) 
    \\ &= \underbrace{\bphi(s ,a )^\top(\hat\bSigma_h)^{-1}\sum_{k=1}^K\bphi(\check s_{k,h},\check a_{k,h})\left[(\hat V_{h+1}^\pi)^2(\check s'_{k,h})-\PP_h(\hat V_{h+1}^\pi)^2(\check s_{k,h},\check a_{k,h})\right]}_{A_1(s,a)}\\
    & \quad \underbrace{- \lambda \bphi(s,a)^\top (\hat\bSigma_h)^{-1} \int_\cS (\hatV_{h+1}^\pi)^2 (s') \diff \bmu_h(s')}_{A_2(s,a)}.
\end{align*}
To bound $|A_1|$, we first apply Cauchy-Schwartz inequality to obtain that
\begin{align*}
    |E_1(s,a)| &\leq \|\bphi(s,a)\|_{\hat\bSigma_h^{-1}} \cdot \left\|\sum_{k=1}^K\bphi(\check s_{k,H}, \check a_{k,H})\left[(\hat V_{h+1}^\pi)^2(\check s'_{k,h})-\PP_h(\hat V_{h+1}^\pi)^2(\check s_{k,h},\check a_{k,h})\right]\right\|_{\hat\bSigma_h^{-1}} . 
\end{align*}
By Lemma \ref{lemma: quadratic form}, if $K$ satisfies
\begin{align}\label{eq:K lower bound hatsigma component uniform}
    K\geq\max\left\{512\|\bSigma_h^{-1}\|^2\log\left(\frac{4d}{\delta}\right),4\lambda\|\bSigma_h^{-1}\|\right\} ,
\end{align}then with probability at least $1-\delta/2$, for all $(s,a)\in \cS\times \cA$, 
\begin{align*}
    \|\bphi(s,a)\|_{\hat\bSigma_h^{-1}} \leq  \frac{2}{\sqrt{K}} \cdot \|\bphi(s,a)\|_{\bSigma_h^{-1}} .
\end{align*} 
By Lemma \ref{lem:self normalized fix hatV}, for fixed $\hatV_{h+1}^\pi$, with probability at least $1-\delta/2$, we have
\begin{align*}
    & \left\|{\sum_{k=1}^K\bphi(\check s_{k,h},\check a_{k,h})\left[(\hat V_{h+1}^\pi)^2(\check s'_{k,h})-\PP_h(\hat V_{h+1}^\pi)^2(\check s_{k,h},\check a_{k,h})\right]}\right\|_{\hat\bSigma_h^{-1}}
    \\ &\leq  2\sqrt{2}(H-h+1)^2 \left[  \frac{d}{2} \log\left( \frac{\lambda+K}{\lambda}\right) + \log\frac{4}{\delta} \right]^{1/2} .
\end{align*}
Combining the two inequalities above, we have that, with probability at least $1-\delta$, 
\begin{align*}
         |A_1(s,a)|
         &\leq 2\sqrt{2}(H-h+1)^2 \left[  \frac{d}{2} \log\left( \frac{\lambda+K}{\lambda}\right) + \log\frac{4}{\delta} \right]^{1/2} \cdot \frac{2}{\sqrt{K}}\cdot\|\bphi(s,a)\|_{\Sigma_h^{-1}} 
        \\& \leq 4\sqrt{2}  \| \Sigma_h^{-1} \|^{1/2} \left[  \frac{1}{2} \log\left( \frac{\lambda+K}{\lambda}\right) + \frac{1}{d}\log\frac{4}{\delta} \right]^{1/2} \cdot \frac{(H-h+1)^2 \sqrt{d}}{\sqrt{K}} ,
\end{align*}for all $(s,a)$.
At the same time, we can bound $A_2$ as 
\begin{align*}
        |A_2(s,a)| &\leq  \lambda \|\bphi(s,a)\|_{\hat\bSigma_h^{-1}} \cdot \left\| \int_\cS (\hatV_{h+1}^\pi)^2 (s') \diff \bmu_h(s') \right\|_{\hat\bSigma_h^{-1}}
        \\ &\leq  \lambda \cdot \frac{2}{\sqrt{K}}\|\bphi(s,a)\|_{\bSigma_h^{-1}}  \cdot \frac{2}{\sqrt{K}}\left\| \int_\cS (\hatV_{h+1}^\pi)^2 (s') \diff \bmu_h(s') \right\|_{\bSigma_h^{-1}} 
        \\ &\leq  4 \lambda \left\| \bSigma_h^{-1} \right\| \cdot \frac{(H-h+1)^2\sqrt{d}}{K}, 
\end{align*}where the last step is by Assumption \ref{assump:linear_MDP}. 
We then conclude that, if $K$ satisfies \eqref{eq:K lower bound hatsigma component uniform}, then with probability at least $1-\delta$, for all $(s,a)$, 
\begin{align*}
        & \left| \langle\bphi(s,a),\hatbeta_h\rangle - \PP_h(\hat V_{h+1})^2(s,a)  \right|
        \\&\leq   |A_1(s,a)| + |A_2(s,a)|
        \\ &\leq  4\sqrt{2}  \| \bSigma_h^{-1} \|^{1/2} \left[  \frac{1}{2} \log\left( \frac{\lambda+K}{\lambda}\right) + \frac{1}{d}\log\frac{4}{\delta} \right]^{1/2} \cdot \frac{(H-h+1)^2 \sqrt{d}}{\sqrt{K}} + 4 \lambda \left\| \bSigma_h^{-1} \right\| \cdot \frac{(H-h+1)^2\sqrt{d}}{K} 
        \\ &=  4\sqrt{2}  \frac{1}{\sqrt{\kappa_h}} \left[  \frac{1}{2} \log\left( \frac{\lambda+K}{\lambda}\right) + \frac{1}{d}\log\frac{4}{\delta} \right]^{1/2} \cdot \frac{(H-h+1)^2 \sqrt{d}}{\sqrt{K}} + \frac{4 \lambda}{\kappa_h} \cdot \frac{(H-h+1)^2\sqrt{d}}{K} ,
\end{align*}where in  the last step we use the definition $\kappa_h \coloneqq \lambda_{\min}(\bSigma_h)$. Note that by Assumption \ref{assump: smallest eigenvalue}, we have $\kappa_h>0$ for all $h\in[H]$. At the same time, we can bound $ \langle\bphi(s,a),\hattheta_h\rangle_{[0,H-h+1]} - \PP_h(\hat V_{h+1})(s,a) $ in a similar way as 
\begin{align*}
        & \left| \langle\bphi(s,a),\hattheta_h\rangle_{[0,H-h+1]} - \PP_h(\hat V_{h+1})(s,a) \right|
        \\ &\leq  4\sqrt{2}  \frac{1}{\sqrt{\kappa_h}} \left[  \frac{1}{2} \log\left( \frac{\lambda+K}{\lambda}\right) + \frac{1}{d}\log\frac{4}{\delta} \right]^{1/2} \cdot \frac{(H-h+1) \sqrt{d}}{\sqrt{K}} + \frac{4\lambda}{\kappa_h} \cdot \frac{(H-h+1)\sqrt{d}}{K} . 
\end{align*}
\end{proof}

\begin{lemma}\label{lem:hatsigma concentration fixed V uniform}
For any $h\in[H]$ and any $\hatV_{h+1}^\pi \in \cV_{h+1}(L)$, with probability at least $1-\delta$, it holds for all $(s,a)\in\cS\times\cA$ that 
\begin{align*}
        \left| \hatsigma^2_h(s,a) - \sigmanoise^2 - \max\left\{\eta_h, \ \VV_h \hatV_{h+1}^\pi(s,a) \right\} \right| & \leq \frac{C_{K,h,\delta} (H-h+1)^2\sqrt{d}}{\sqrt{K}}
        \\ & \leq \frac{20 (H-h+1)^2\sqrt{d}}{\kappa_h\sqrt{K}}\cdot\sqrt{\log\left(\frac{K}{\lambda\delta}\right)},
\end{align*}where
\begin{align}\label{eq:constant hatsigma Ckdelta}
    & C_{K,h,\delta} = 12\sqrt{2}  \frac{1}{\sqrt{\kappa_h}} \cdot \left[ \frac{1}{2} \log\left( \frac{\lambda+K}{\lambda}\right) + \frac{1}{d} \log\frac{4}{\delta} \right]^{1/2}+ 12\lambda \frac{1}{\kappa_h}.
\end{align}
\end{lemma}

\begin{proof}[Proof of Lemma \ref{lem:hatsigma concentration fixed V uniform}]
Recall that by definition,
\begin{align*}
    \VV_h \hatV_{h+1}^\pi(s,a) = \PP_h (\hatV_{h+1}^\pi)^2(s,a) - \left( \PP_h \hatV_{h+1}^\pi(s,a)\right)^2.
\end{align*}We then have 
\begin{align*}
        & \left| \left[ \langle\bphi(s,a),\hatbeta_h \rangle_{[0,(H-h+1)^2]} - \left(\langle\bphi(s,a),\hattheta_h \rangle_{[0,H-h+1]}\right)^2 \right] - \VV_h \hatV_{h+1}^\pi(s,a)  \right| 
        \\ &\leq  \left|  \langle\bphi(s,a),\hatbeta_h \rangle_{[0,(H-h+1)^2]} - \PP_h (\hatV_{h+1}^\pi)^2(s,a) \right| + 2(H-h+1)\cdot\left| \langle\bphi(s,a),\hattheta_h \rangle_{[0,H-h+1]} - \PP_h \hatV_{h+1}^\pi(s,a) \right|,
\end{align*}and the rest follows from Lemma \ref{lem:hatsigma component concentration fixed V uniform conv} and the fact that $\max\{\eta_h,\cdot\}$ is a contraction mapping.
\end{proof}

\begin{lemma}\label{lemma: uniform convergence hatLambda}
For any $h\in[H-1]$, let $V\in\cV_{h+1}(L)\cap\{V:\sup_{s\in\cS}|V(s)-V_{h+1}^\pi(s)|\leq \rho\}$ for some sufficiently small $\rho<(\eta_h+\sigmanoise^2)/[12(H-h+1)]$. Suppose $K$ satisfies that
\begin{align}\label{eq: K lowerbound uniform convergence hatLambda}
    K\geq \frac{3600(H-h+1)^4d}{\kappa_h^2\inf_{s,a}\sigma_h(s,a)^2}\cdot\log\left(\frac{Kd}{\lambda\delta}\right)
\end{align}Then for any $\delta\in(0,1)$, it holds with probability at least $1-\delta$ that
\begin{align*}
    \bignorm{\left(\frac{\hat\bLambda_h}{K}\right)^{-1}}\leq \frac{4}{\iota_h}.
\end{align*}
\end{lemma}
\begin{proof}[Proof of Lemma \ref{lemma: uniform convergence hatLambda}]
By Lemma \ref{lem:hatsigma concentration fixed V uniform}, there exists an event $\check\cE$ over $\{(\check s_{k,h},\check a_{k,h}),k\in[K]\}$ such that $\PP(\check\cE)\geq 1-\delta$ and on this event it holds for all $(s,a)\in\cS\times\cA$ that
\begin{align*}
    \left| \hat\sigma_h^2(s,a) - \sigmanoise^2 - \max\left\{\eta_h, \ \VV_h V_{h+1}^\pi(s,a) \right\} \right| \leq \frac{20 (H-h+1)^2\sqrt{d}}{\kappa_h\sqrt{K}}\sqrt{\log\left(\frac{K}{\lambda\delta}\right)} + 4(H-h+1)\cdot\rho.
\end{align*}
Then by \eqref{eq: K lowerbound uniform convergence hatLambda} and the assumption on $\rho$, we have
\begin{align}\label{eq: uniform convergence hatLambda 1}
    \frac{1}{3}\sigma_h(s,a)\leq \hat\sigma_h(s,a)\leq\frac{5}{3}\sigma_h(s,a)
\end{align}
for all $(s,a)\in\cS\times\cA$. In the following argument we condition on $\check\cE$, and this will not affect the distribution of $\{(s_{k,h},a_{k,h}),k\in[K]\}$ by independence.

Recall that
\begin{align*}
    \hat\bLambda_h = \sum_{k=1}^K\hat\sigma_h(s_{k,h},a_{k,h})^{-2}\bphi(s_{k,h},a_{k,h})\bphi(s_{k,h},a_{k,h})^\top+\lambda\Ib_d.
\end{align*}
Since $\hat\sigma_h\geq \inf_{s,a}\sigma_h(s,a)/3$, it then follows from Lemma \ref{lem:matrix McDiarmid} that
\begin{align}\label{eq: uniform convergence hatLambda difference}
    \left\|\frac{\hat\bLambda_h}{K}-\EE\left[\frac{\hat\bLambda_h}{K}\right]\right\|\leq \frac{12\sqrt{2}}{\sqrt{K}\cdot\inf_{s,a}\sigma_h(s,a)^2}\cdot\sqrt{\log\left(\frac{2d}{\delta}\right)}.
\end{align}
To bound $\norm{(\hat\bLambda_h/K)^{-1}}$, we use the fact that 
\begin{align*}
    \bignorm{\left(\frac{\hat\bLambda_h}{K}\right)^{-1}} &\leq \bignorm{\EE\left[\frac{\hat\bLambda_h}{K}\right]^{-1}} + \bignorm{\left(\frac{\hat\bLambda_h}{K}\right)^{-1} -  \EE\left[\frac{\hat\bLambda_h}{K}\right]^{-1}  }
    \\ & \leq \bignorm{\EE\left[\frac{\hat\bLambda_h}{K}\right]^{-1}} + \bignorm{\left(\frac{\hat\bLambda_h}{K}\right)^{-1}} \cdot \bignorm{\EE\left[\frac{\hat\bLambda_h}{K}\right]^{-1}}  \cdot \bignorm{\frac{\hat\bLambda_h}{K} -  \EE\left[\frac{\hat\bLambda_h}{K}\right]},
\end{align*}which implies 
\begin{align}\label{eq: uniform convergence hatLambda 2}
        \bignorm{\left(\frac{\hat\bLambda_h}{K}\right)^{-1}} &\leq \bignorm{\EE\left[\frac{\hat\bLambda_h}{K}\right]^{-1}}\cdot \left( 1 -  \bignorm{\EE\left[\frac{\hat\bLambda_h}{K}\right]^{-1}}  \cdot \bignorm{\frac{\hat\bLambda_h}{K} -  \EE\left[\frac{\hat\bLambda_h}{K}\right]} \right)^{-1}\notag\\
        &\leq \left(\bignorm{\EE\left[\frac{\hat\bLambda_h}{K}\right]^{-1}}^{-1}-\bignorm{\frac{\hat\bLambda_h}{K} -  \EE\left[\frac{\hat\bLambda_h}{K}\right]} \right)^{-1}
\end{align}
        Note that by \eqref{eq: uniform convergence hatLambda 1}, we have
\begin{align}\label{eq: uniform convergence hatLambda 3}
    \EE\left[\frac{\hat\bLambda_h}{K}\right]&=\frac{1}{K}\sum_{k=1}^K\EE\left[\frac{\bphi(s_{k,h},a_{k,h})\bphi(s_{k,h},a_{k,h})^\top}{\hat\sigma_h(s_{k,h},a_{k,h})^2}\right] + \frac{\lambda}{K}\Ib_d\notag\\
    &\succeq \frac{1}{2K}\sum_{k=1}^K\EE\left[\frac{\bphi(s_{k,h},a_{k,h})\bphi(s_{k,h},a_{k,h})^\top}{\sigma_h(s,a)^2}\right] + \frac{\lambda}{K}\Ib_d\notag\\
    &=\frac{1}{2}\bLambda_h+\frac{\lambda}{K}\Ib_d.
\end{align}
Finally combining \eqref{eq: uniform convergence hatLambda difference}, \eqref{eq: uniform convergence hatLambda 2} and \eqref{eq: uniform convergence hatLambda 3} yields
\begin{align*}
        \bignorm{\left(\frac{\hat\bLambda_h}{K}\right)^{-1}}& \leq \frac{1}{\left(\frac{\iota_h}{2} +\frac{\lambda}{K}\right) - \frac{12\sqrt{2}}{\sqrt{K}\cdot\inf_{s,a}\sigma_h(s,a)^2}\cdot \sqrt{\log\left(\frac{2d}{\delta}\right)}}\leq \frac{4}{\iota_h},
\end{align*}
where the second inequality follows from \eqref{eq: K lowerbound uniform convergence hatLambda}.
\end{proof}

\begin{lemma}\label{lemma: uniform convergence hatLambda function class}
For any $h\in[H-1]$, let $\rho$ be some positive constant such that $\rho<(\eta_h+\sigmanoise^2)/[12(H-h+1)]$. For any $\delta\in(0,1)$, suppose K satisfies that
\begin{align}\label{eq: K lowerbound uniform convergence hatLambda function class}
    K \geq \frac{3600(H-h+1)^4d^2}{\kappa_h^2\inf_{s,a}\sigma_h(s,a)^2}\cdot\log\left(\frac{d(H-h+1)KL}{\iota_h\kappa_h\lambda\delta}\right).
\end{align}
Then it holds with probability at least $1-\delta$ that
\begin{align*}
    \left\|\left(\frac{\hat\bLambda_h}{K}\right)^{-1}\right\|\leq\frac{8}{\iota_h}.
\end{align*}
for all $V\in\cV_{h+1}(L)\cap\{V:\sup_{s\in\cS}|V(s)-V_{h+1}^\pi(s)|\leq\rho\}$.
\end{lemma}
\begin{proof}[Proof of Lemma \ref{lemma: uniform convergence hatLambda function class}]

Let $\epsilon>0$ be a constant to be determined later and $\cC_V$ be a $\epsilon-$cover of $\cV_{h+1}(L)\cap\{V:\sup_{s\in\cS}|V(s)-V_{h+1}^\pi(s)|\leq \rho\}$. By Lemma \ref{lemma: uniform convergence hatLambda}, the choice of $K$ in \eqref{eq: K lowerbound uniform convergence hatLambda function class} and a union bound, we have
\begin{align}\label{eq: uniform convergence hatLambda function class 0}
    \left\|\left(\frac{\hat\bLambda_h}{K}\right)^{-1}\right\|\leq \frac{4}{\iota_h}
\end{align}
for all $V\in\cC_V$, given $K$ satisfies that
\begin{align}\label{eq: uniform convergence hatLambda function class 1}
     K \geq \frac{3600(H-h+1)^4d}{\kappa_h^2\inf_{s,a}\sigma_h(s,a)^2}\cdot\log\left(\frac{Kd\cN_{\epsilon}}{\lambda\delta}\right)
\end{align}
where $\cN_\epsilon$ is the $\epsilon-$covering number of $\cV_{h+1}(L)\cap\{V:\sup_{s\in\cS}|V(s)-V_{h+1}^\pi(s)|\leq\rho\}$.
    
For any $V_1\in\cV_{h+1}(L)\cap\left\{V:\sup_{s\in\cS}|V(S)-V_{h+1}^\pi(s)|\leq \rho\right\}$, there exists $V_2\in\cC_V$ such that $\sup_{s\in\cS}|V_1(s)-V_2(s)|\leq \epsilon$. Let $\sigma_1$ and $\hat\bLambda_{h,1}$ be the variance estimator and the weighted covariance induced by $V_1$, and $\sigma_2$ and $\hat\bLambda_{h,2}$ that of $V_2$. Then we have
\begin{align}\label{eq: uniform convergence hatLambda function class 2}
    & \left|\sigma_1^2(s,a) - \sigma_2^2(s,a)\right| \notag 
    \\ &\leq \left|\langle\bphi(s,a),\hat\bbeta_{h,1}-\hat\bbeta_{h,2}\rangle\right| + 2(H-h+1)\left|\langle\bphi(s,a),\hat\btheta_{h,1}-\hat\btheta_{h,2}\rangle\right|\notag\\
    &\leq \left|\hat\bSigma_h^{-1}\sum_{k=1}^K\bphi(\check s_{k,h},\check a_{k,h})(V_1^2(\check s_{k,h}')-V_2^2(\check s_{k,h}'))\right|  + \left|\hat\bSigma_h^{-1}\sum_{k=1}^K\bphi(\check s_{k,h},\check a_{k,h})(V_1(\check s_{k,h}')-V_2(\check s_{k,h}'))\right|\notag\\
    &\leq \frac{4(H-h+2)^2K}{\kappa_h}\cdot\epsilon,
\end{align}
where the second inequality is due to Assumption \ref{assump:linear_MDP} and the third inequality follows from the fact that $V_1,V_2\in\cV_{h+1}(L)$.

Therefore, we can bound the difference between $\hat\bLambda_{h,1}$ and $\hat\bLambda_{h,2}$ as follows.
\begin{align}\label{eq: uniform convergence hatLambda function class 3}
    \left\|\frac{\hat\bLambda_{h,1}}{K}-\frac{\hat\bLambda_{h,2}}{K}\right\| &= \left\|\frac{1}{K}\sum_{k=1}^K\bphi(s_{k,h},a_{k,h})\bphi(s_{k,h},a_{k,h})^\top\cdot\frac{\sigma_1(s,a)^2-\sigma_2(s,a)^2}{\sigma_1(s,a)^2\sigma_2(s,a)^2}\right\|\notag\\
    &\leq \frac{1}{K}\sum_{k=1}^K\frac{|\sigma_1(s,a)^2-\sigma_2(s,a)^2|}{\sigma_1(s,a)^2\sigma_2(s,a)^2}\notag\\
    &\leq \frac{4(H-h+2)^2K}{\kappa_h(\eta_h+\sigmanoise^2)^2}\cdot\epsilon,
\end{align}
where the first inequality follows from Assumption \ref{assump:linear_MDP}, and the second inequality is due to \eqref{eq: uniform convergence hatLambda function class 2}. When $\epsilon$ is small enough, by \eqref{eq: uniform convergence hatLambda function class 0} we have
\begin{align*}
    \lambda_{\min}(\hat\bLambda_{h,1}/K)\geq \lambda_{\min}(\hat\bLambda_{h,2}/K)-\|\hat\bLambda_{h,1}-\hat\bLambda_{h,2}\|/K\geq \frac{\iota_h}{4} - \frac{4(H-h+2)^2K}{\kappa_h(\eta_h+\sigmanoise^2)}\cdot\epsilon,
\end{align*}
which further implies that
\begin{align*}
    \left\|\left(\frac{\hat\bLambda_{h,1}}{K}\right)^{-1}\right\|\leq \left(\frac{\iota_h}{4}-\frac{4(H-h+2)^2K}{\kappa_h(\eta_h+\sigmanoise^2)}\cdot\epsilon\right)^{-1}\leq \frac{8}{\iota_h}
\end{align*}
if we choose $\epsilon=\iota_h\kappa_h(\eta_h+\sigmanoise^2)/[32(H-h+2)^2K]$. In this case, by Lemma \ref{lem:covering V}, we have
\begin{align}\label{eq: uniform convergence hatLambda function class 4}
    \log\cN_\epsilon \leq d\cdot\left(1+\frac{64L(H-h+2)^2K}{\iota_h\kappa_h(\eta_h+\sigmanoise^2)}\right).
\end{align}
Therefore, by \eqref{eq: uniform convergence hatLambda function class 1}, \eqref{eq: uniform convergence hatLambda function class 3} and \eqref{eq: uniform convergence hatLambda function class 4}, it suffices to choose $K$ such that
\begin{align*}
    K \geq \frac{3600(H-h+1)^4d^2}{\kappa_h^2\inf_{s,a}\sigma_h(s,a)^2}\cdot\log\left(\frac{d(H-h+1)KL}{\iota_h\kappa_h\lambda\delta}\right).
\end{align*}
\end{proof}

\subsection{Bernstein Inequality for the Self-Normalized Martingales}
\begin{lemma}\label{lem: E2 bernstein uniform}
For any $h \in [H-1]$ and any fixed $\hatV_{h+1}^\pi\in\cV_{h+1}(L)$, let $\hat\sigma_h$ be as defined in Line \ref{alg: sigma} of Algorithm \ref{alg: general form} and $\hat\bLambda_h$ be as defined in \eqref{def: hatLambda}. Suppose $K$ satisfies that
\begin{align}\label{eq: K lowerbound bernstein uniform}
    K\geq \frac{1600(H-h+1)^4d}{\kappa_h^2(\eta_h+\sigmanoise^2)^2}\cdot\log\left(\frac{K}{\lambda\delta}\right)
\end{align}Then for any $\delta\in(0,1)$, it holds with probability at least $1-\delta$ that
\begin{align*}
    & \bignorm{\sum_{k=1}^K\hat\sigma_h(s_{k,h},a_{k,h})^{-2}\bphi(s_{k,h},a_{k,h})\rbr{\PP_h\hat V_{h+1}^\pi(s_{k,h},a_{k,h}) - \hat V_{h+1}^\pi(s_{k,h}')-\epsilon_{k,h} }}_{\hat\bLambda_{h}^{-1}} \\ 
    &\leq  \sqrt{2d\log\left(1+\frac{K}{\lambda d(\eta_h+\sigmanoise^2)}\right)\cdot\log\left(\frac{4K^2}{\delta}\right)} + \frac{4(2H-2h+3)}{\sqrt{\eta_h+\sigmanoise^2}}\log\left(\frac{4K^2}{\delta}\right)
\end{align*}
\end{lemma}

\begin{proof}[Proof of Lemma \ref{lem: E2 bernstein uniform}]
Let $\check\cE$ be the event given by Lemma \ref{lem:hatsigma concentration fixed V uniform}, on which it holds for all $(s,a)\in\cS\times\cA$ that
\begin{align}\label{eq: uniform bernstein 1}
    \left|\hat\sigma_h^2(s,a)-\sigmanoise^2-\max\{\eta_h,\VV_h\hat V_{h+1}^\pi(s,a)\}\right|\leq \frac{20(H-h+1)^2\sqrt{d}}{\kappa_h\sqrt{K}}\cdot\sqrt{\log\left(\frac{K}{\lambda\delta}\right)}.
\end{align}
Now conditioning on $\check\cE$, it will not affect the distribution of $\{(s_{k,h},a_{k,h}),k\in[K]\}$ by independence. In the following argument, we omit the explicit notation for conditioning on $\check\cE$ for simplicity.

Define $\xb_k = \phi(s_{k,h}, a_{k,h})/ \hat\sigma_h(s_{k,h},a_{k,h})$, which is a deterministic function of $(s_{k,h}, a_{k,h})$ since $\hattheta_h$ and $\hatbeta_h$ are fixed. Define $\zeta_k = \rbr{\PP_h\hat V_{h+1}^\pi(s_{k,h},a_{k,h}) - \hat V_{h+1}^\pi(s_{k,h}')-\epsilon_{k,h} }/\hatsigma_{k,h}$, which is a function of $s_{k,h}, a_{k,h}, \epsilon_{k,h}, s'_{k,h}$. Now we define the filtration $\{ \cF_k \}_{k=0}^K$ by $\cF_0 = \sigma(s_{1,h}, a_{1,h})$, $\cF_1=\sigma(s_{1,h}, a_{1,h}, \epsilon_{1,h}, s'_{1,h},s_{2,h}, a_{2,h})$ , $\cdots$, $\cF_k = \sigma(s_{1,h}, a_{1,h}, \epsilon_{1,h},s'_{1,h}, \cdots, s_{k,h}, a_{k,h},\epsilon_{k,h}, s'_{k,h}, s_{k+1,h}, a_{k+1,h})$ for $k = 1, \cdots, K-1$, and $\cF_K = \sigma( \cF_{K-1}, \epsilon_{K,h},s'_{K,h})$. Then we see that $\xb_k$ is $\cF_{k-1}$-measurable and $\zeta_k$ is $\cF_k$-measurable. Furthermore, since $\EE[\hat V_{h+1}^\pi(s_{k,h}')\mid \cF_{k-1}] = \PP_h\hat V_{h+1}^\pi(s_{k,h},a_{k,h})$, $\EE[\epsilon_{k,h}\mid \cF_{k-1}]=0$ and $\hatsigma_{k,h}$ is $\cF_{k-1}$-measurable,  $\zeta_k \mid \cF_{k-1}$ has zero-mean. Also, by construction we have $|\zeta_k| \leq (2H-2h+3)/\sqrt{\eta_h+\sigmanoise^2}$, and it follows from \eqref{eq: uniform bernstein 1} that
\begin{align*}
    \Var(\zeta_k\mid\cF_{k-1}) &\leq \frac{\VV_h\hat V_{h+1}^\pi(s_{k,h},a_{k,h})+\sigmanoise^2}{\max\{\eta_h,\VV_h\hat V_{h+1}^\pi(s_{k,h},a_{k,h})\}+\sigmanoise^2-\frac{20 (H-h+1)^2\sqrt{d}}{\kappa_h\sqrt{K}}\cdot\sqrt{\log\left(\frac{K}{\lambda\delta}\right)}}\leq 2,
\end{align*}
as long as $K$ satisfies \eqref{eq: K lowerbound bernstein uniform}.

Then by Theorem \ref{thm:self normalized bernstein}, with probability at least $1-\delta$, we have
\begin{align*}
    \left\|\sum_{k=1}^K\xb_k\zeta_k\right\|_{\hat\bLambda_h^{-1}}&\leq \sqrt{2d\log\left(1+\frac{K}{\lambda d(\eta_h+\sigmanoise^2)}\right)\cdot\log\left(\frac{4K^2}{\delta}\right)} + \frac{4(2H-2h+3)}{\sqrt{\eta_h+\sigmanoise^2}}\log\left(\frac{4K^2}{\delta}\right).
\end{align*}
Since $\PP(\check\cE)\geq 1-\delta$, the overall probability is at least $(1-\delta)^2\geq 1-2\delta$ by independence. Finally replacing $\delta$ by $\delta/2$ completes the proof.
\end{proof}

\begin{lemma}\label{lemma: uniform convergence bernstein class}
Let $\epsilon>0$ be a constant. For any $h\in[H]$ and $\delta\in(0,1)$, suppose $K$ satisfies that
\begin{align}\label{eq: K lowerbound bernstein function class}
    K\geq \frac{1600(H-h+1)^4d^2}{\kappa_h^2(\eta_h+\sigmanoise^2)^2}\cdot\log\left(\frac{(H-h+1)^2KL}{\lambda\kappa_h(\eta_h+\sigmanoise^2)\delta}\right)
\end{align}
where $\cN_\epsilon$ is the $\epsilon$-covering number of $\cV_{h+1}(L)$.  Then with probability at least $1-\delta$, it holds for all function $V\in\cV_{h+1}(L)$ that
\begin{align}\label{eq: E2 Bernstein function class 1}
    &\left\|\sum_{k=1}^K\hat\sigma_h(s_{k,h},a_{k,h})^{-2}\bphi(s_{k,h},a_{k,h})\left(\PP_hV(s_{k,h},a_{k,h})-V(s_{k,h}')-\epsilon_{k,h})\right)\right\|_{\hat\bLambda_h^{-1}}^2\notag\\
    &\leq  50\left(d + \frac{\sqrt{d}(H-h+1)}{\sqrt{\eta_h+\sigmanoise^2}}\right)^2\cdot\log^2\left(\frac{K(H-h+1)^2L}{\kappa_h(\eta_h+\sigmanoise^2)\delta}\right).
\end{align}
\end{lemma}
\begin{proof}[Proof of Lemma \ref{lemma: uniform convergence bernstein class}]
For the simplicity of presentation, we first define some notations. We define the following vector,
\begin{align*}
    \vb_{V} := \sum_{k=1}^K \frac{\bphi(s_{k,h},a_{k,h})\rbr{\PP_h V(s_{k,h},a_{k,h}) - V(s_{k,h}')-\epsilon_{k,h}}}{\hat\sigma_h(s_{k,h},a_{k,h})^2}.
\end{align*} and the following matrix
\begin{align*}
    \bGamma_V := \sum_{k=1}^K \bphi(s_{k,h},a_{k,h}) \bphi(s_{k,h},a_{k,h})^\top / \hat\sigma_h^2(s_{k,h}, a_{k,h}) + \lambda \Ib_d.
\end{align*}It remains to show that, with probability at least $1-\delta$, for any function $V \in \cV_{h+1}(L)$, $\vb_{V}^\top \bGamma_V^{-1} \vb_{V}$ is no greater than the R.H.S. of \eqref{eq: E2 Bernstein function class 1}. In the following argument, for $V_1 \in \cV_{h+1}(L)$, we denote $\vb_1 = \vb_{V_1}$, $\bGamma_1 = \bGamma_{V_1}$ and $\sigma_1$ the variance estimator induced by $V_1$, and similar for $V_2$. 

Let $\cC_V$ be the smallest $\epsilon$-cover of $\cV_{h+1}(L)$, and $\cN_\epsilon=|\cC_V|$ the $\epsilon$-covering number of $\cV_{h+1}(L)$. For any $V_1 \in \cV_{h+1}(L)$, there exists $V_2\in\cC_V$ such that $\textnormal{dist}(V_1,V_2) = \sup_{s} |V_1(s) - V_2(s)|\leq \epsilon$. Note that we have the following decomposition:
\begin{align}\label{eq: E2 Bernstein function class 2}
    \vb_1^\top \Ab_1^{-1}  \vb_1 \leq \vb_2^\top \Ab_2^{-1} \vb_2 + \left|\vb_1^\top \Ab_1^{-1} \vb_1  - \vb_2^\top \Ab_2^{-1} \vb_2 \right|.
\end{align}
By Lemma \ref{lem:covering V}, when $K$ satisfies \eqref{eq: K lowerbound bernstein function class}, we have
\begin{align*}
     K\geq \frac{1600(H-h+1)^4d}{\kappa_h^2(\eta_h+\sigmanoise^2)^2}\cdot\log\left(\frac{K\cN_{\epsilon}}{\lambda\delta}\right)
\end{align*}
Then by Lemma \ref{lem: E2 bernstein uniform} and a union bound, we have that, with probability at least $1-\delta$,  
\begin{align}\label{eq: E2 Bernstein function class 3}
    \vb_2^\top \Ab_2^{-1} \vb_2 \leq \left(\sqrt{2d\log\left(1+\frac{K}{\lambda d(\eta_h+\sigmanoise^2)}\right)\cdot\log\left(\frac{4K^2\cN_\epsilon}{\delta}\right)} + \frac{4(2H-2h+3)}{\sqrt{\eta_h+\sigmanoise^2}}\log\left(\frac{4K^2\cN_\epsilon}{\delta}\right)\right)^2. 
\end{align}
It remains to bound the second term in \eqref{eq: E2 Bernstein function class 2}. We first bound $\norm{\vb_1- \vb_2}_2$. 
\begin{align}\label{eq: E2 Bernstein function class 4}
        &\norm{\vb_1-\vb_2} \notag \\
        &=  \Bignorm{ \sum_{k=1}^K \bphi(s_{k,h},a_{k,h}) \left( \frac{\PP_h V_1(s_{k,h},a_{k,h}) - V_1(s_{k,h}')-\epsilon_{k,h}}{\sigma_1^2(s_{k,h},a_{k,h})} - \frac{\PP_h V_2(s_{k,h},a_{k,h}) - V_2(s_{k,h}')-\epsilon_{k,h}}{\sigma_2^2(s_{k,h},a_{k,h})}\right)} \notag
        \\ &\leq  \sum_{k=1}^K \norm{\bphi(s_{k,h}, a_{k,h})}_2 \cdot \left| \frac{ \PP_h V_1(s_{k,h},a_{k,h}) - V_1(s_{k,h}')-\epsilon_{k,h}}{\sigma_1^2(s_{k,h},a_{k,h})} - \frac{\PP_h V_2(s_{k,h},a_{k,h}) - V_2(s_{k,h}')-\epsilon_{k,h}}{\sigma_2^2(s_{k,h},a_{k,h})} \right| \notag
        \\ &\leq  \sum_{k=1}^K \left| \frac{ \rbr{\PP_h V_1(s_{k,h},a_{k,h}) - V_1(s_{k,h}')-\epsilon_{k,h}}}{\sigma_1^2(s_{k,h},a_{k,h})} - \frac{\rbr{\PP_h V_2(s_{k,h},a_{k,h}) - V_2(s_{k,h}')-\epsilon_{k,h}}}{\sigma_2^2(s_{k,h},a_{k,h})} \right|
\end{align}
where the first inequality follows from Cauchy-Schwartz inequality and the second inequality is due to Assumption \ref{assump:linear_MDP}.

Note that for any real-valued function $f_1(\cdot), f_2(\cdot)$ and positive function $g_1(\cdot), g_2(\cdot)$ bounded away from $0$, we have
\begin{align}\label{eq: E2 Bernstein function class 5}
    \left|\frac{f_1}{g_1} - \frac{f_2}{g_2} \right| & = \left|\frac{f_1g_2 - f_1g_1+f_1g_1-g_1f_2}{g_1g_2} \right|  \notag
    \\ & \leq \left| \frac{f_1(g_2-g_1)}{g_1g_2} \right| + \left|  \frac{g_1(f_1-f_2)}{g_1 g_2} \right| \notag
    \\ & \leq \frac{1}{\inf g_1 \inf g_2} \left[(\sup |f_1|) \cdot |g_2-g_1| + (\sup g_1) \cdot |f_1-f_2|\right] .
\end{align} 
Now, by the construction we have $\sigma_1^2(\cdot,\cdot) \in [\eta_h+\sigmanoise^2, (H-h+1)^2+\sigmanoise^2]$, and $\PP_h V_1(\cdot,\cdot) - V_1(\cdot) -\epsilon_{k,h} \in [-2H+2h-3,2H-2h+3]$, and the same for $\sigma_2$ and $V_2$. Also note that for all $(s,a)$,
\begin{align}\label{eq: E2 Bernstein function class 6}
    & \left|\sigma_1^2(s,a) - \sigma_2^2(s,a)\right| \notag 
    \\ &\leq \left|\langle\bphi(s,a),\hat\bbeta_{h,1}-\hat\bbeta_{h,2}\rangle\right| + 2(H-h+1)\left|\langle\bphi(s,a),\hat\btheta_{h,1}-\hat\btheta_{h,2}\rangle\right|\notag\\
    &\leq \left|\hat\bSigma_h^{-1}\sum_{k=1}^K\bphi(\check s_{k,h},\check a_{k,h})(V_1^2(\check s_{k,h}')-V_2^2(\check s_{k,h}'))\right| + \left|\hat\bSigma_h^{-1}\sum_{k=1}^K\bphi(\check s_{k,h},\check a_{k,h})(V_1(\check s_{k,h}')-V_2(\check s_{k,h}'))\right|\notag\\
    &\leq \frac{4(H-h+2)^2K}{\kappa_h}\cdot\epsilon,
\end{align}
where the second inequality is due to Assumption \ref{assump:linear_MDP} and the third inequality follows from the fact that $V_1,V_2\in\cV_{h+1}(L)$.

Denote $\ub=\vb_2-\vb_1$. Combining \eqref{eq: E2 Bernstein function class 4}, \eqref{eq: E2 Bernstein function class 5} and \eqref{eq: E2 Bernstein function class 6} yields that
\begin{align}\label{eq: E2 Bernstein function class 7}
    \norm{\ub}&\leq \sum_{k=1}^K \frac{1}{(\eta_h+\sigmanoise^2)^2} \left[ \frac{4(2H-2h+3)(H-h+2)^2K}{\kappa_h}\cdot\epsilon + ((H-h+1)^2+\sigmanoise^2) \cdot 2\epsilon \right]\notag\\
    &\leq  \frac{10(H-h+1)^2K^2}{\kappa_h(\eta_h+\sigmanoise^2)^2}\epsilon,
\end{align}
where the second inequality is due to the fact that $\sigmanoise^2\leq 1$.

Next, by Lemma \ref{lem:auxiliary inverse matrix distance} and \eqref{eq: E2 Bernstein function class 6}, we have 
\begin{align}\label{eq: E2 Bernstein function class 8}
    \norm{\Ab_1^{-1} - \Ab_2^{-1}} \leq \frac{8K^2((H-h+1)^2 + \sigmanoise^2)}{\lambda^2\kappa_h (\eta_h+\sigmanoise^2)^2}\cdot \epsilon. 
\end{align}Also note that
\begin{align}\label{eq: E2 Bernstein function class 9}
        \left|\vb_1^\top \Ab_1^{-1} \vb_1 - \vb_2^\top \Ab_2^{-1} \vb_2 \right| &= \left|\vb_1^\top \Ab_1^{-1} \vb_1  -(\vb_1+\ub)^\top \Ab_2^{-1} (\vb_1+\ub)   \right| \notag 
        \\ & \leq \left| \vb_1^\top (\Ab_1^{-1} - \Ab_2^{-1}) \vb_1 \right| + 2\left| \vb_1^\top \Ab_2^{-1} \ub \right| + \left| \ub^\top \Ab_2^{-1} \ub \right|.
\end{align}
By the definition, we have $\norm{\vb_1}_2, \norm{\vb_2}_2 \leq (2H-2h+3)K/(\eta_h+\sigmanoise^2)$, and $\norm{\Ab_1^{-1}}, \norm{\Ab_2^{-1}} \leq 1/\lambda$. It then follows from \eqref{eq: E2 Bernstein function class 8} and \eqref{eq: E2 Bernstein function class 9} that 

\begin{align*}
    \left|\vb_1^\top \Ab_1^{-1} \vb_1  - \vb_2^\top \Ab_2^{-1} \vb_2 \right| &\leq \frac{(2H-2h+3)^2K^2}{(\eta_h+\sigmanoise^2)^2}\cdot\frac{8K^2((H-h+1)^2+\sigmanoise^2)}{\lambda^2\kappa_h (\eta_h+\sigmanoise^2)^2} \cdot \epsilon\notag\\
    &\qquad + \frac{2(2H-2h+3)K}{\lambda (\eta_h+\sigmanoise^2)}\cdot \frac{10K(H-h+1)^2K^2}{\kappa_h(\eta_h+\sigmanoise^2)^2}\cdot\epsilon\notag\\
    &\qquad + \frac{100(H-h+1)^4K^4}{\lambda\kappa_h^2(\eta_h+\sigmanoise^2)^4}\cdot\epsilon^2\notag\\
    &\leq \frac{200(H-h+1)^4K^4}{\kappa_h^2(\eta_h+\sigmanoise^2)^2}\cdot\epsilon,
\end{align*}
by the choice of $\lambda=1$. We then choose $\epsilon = \kappa_h^2(\eta_h+\sigmanoise^2)^2/[200(H-h+1)^4K^5]$, and thus 
\begin{align}\label{eq: E2 Bernstein function class 10}
    \left|\vb_1^\top \Ab_1^{-1} \vb_1  - \vb_2^\top \Ab_2^{-1} \vb_2 \right|\leq \frac{1}{K}
\end{align}
Now by Lemma \eqref{lem:covering V}, we have
\begin{align}\label{eq: E2 Bernstein function class 11}
    \cN_\epsilon\leq \left(1+\frac{400(H-h+1)^4K^5L}{\kappa_h^2(\eta_h+\sigmanoise^2)^2}\right)^d.
\end{align}
Then combining \eqref{eq: E2 Bernstein function class 2}, \eqref{eq: E2 Bernstein function class 3}, \eqref{eq: E2 Bernstein function class 10} and \eqref{eq: E2 Bernstein function class 11} yields
\begin{align*}
    &\left\|\sum_{k=1}^K\hat\sigma_h(s_{k,h},a_{k,h})^{-2}\bphi(s_{k,h},a_{k,h})\left(\PP_hV(s_{k,h},a_{k,h})-V(s_{k,h}')-\epsilon_{k,h})\right)\right\|_{\hat\bLambda_h^{-1}}^2\notag\\
    &\leq  50\left(d + \frac{d(H-h+1)}{\sqrt{\eta_h+\sigmanoise^2}}\right)^2\cdot\log^2\left(\frac{K(H-h+1)^2L}{\kappa_h(\eta_h+\sigmanoise^2)\delta}\right).
\end{align*}
This completes the proof.
\end{proof}

\subsection{Bounding the error terms}
Finally, we prove the following key lemma for completing the induction step in the proof of Theorem \ref{thm:uniform convergence}.

\begin{lemma}\label{lemma: uniform convergence induction bernstein}
    Set $L=(1+1/H)d\sqrt{K/\lambda}$. For any $h\in[H-1]$, let $\rho$ be some positive constant such that $\rho<(\eta_h+\sigmanoise^2)/[12(H-h+1)]$. For any $\delta\in(0,1)$, suppose $K$ satisfies that
    \begin{align}\label{eq: K lowerbound uc induction bernstein}
        K\geq \frac{3600(H-h+1)^4d^2}{\kappa_h^2(\eta_h+\sigmanoise^2)^2}\cdot\log\left(\frac{dHK}{\kappa_h\delta}\right)
    \end{align}
    Then the following two events hold simultaneously with probability at least $1-\delta$:
    \begin{enumerate}
    \item $\tilde\cE_1$: for all $V\in\cV_{h+1}(L)\cap\{V:\sup_{s\in\cS}|V(s)-V_{h+1}^\pi(s)|\leq\rho\}$, 
    \begin{align}\label{eq: E2 uniform convergence 0}
        \left\|\left(\frac{\hat\bLambda_h}{K}\right)^{-1}\right\| \leq \frac{8}{\iota_h};
    \end{align}
    
    \item $\tilde\cE_2$: for all function $V(\cdot) \in \cV_{h+1}(L)\cap\{V:\sup_{s\in\cS}|V(s)-V_{h+1}^\pi(s)|\leq\rho\}$ and all $(s,a)$ pairs, 
    \begin{align*}
    &\left|\bphi(s,a)^\top\hat\bLambda_h^{-1}\sum_{k=1}^K\hat\sigma_h(s_{k,h},a_{k,h})^{-2}\bphi(s_{k,h},a_{k,h})\left(\PP_hV(s_{k,h},a_{k,h})-V(s_{k,h}')-\epsilon_{k,h}\right)\right|\\
    &\leq  \frac{20}{\sqrt{K}}\cdot\left(\frac{d}{\sqrt{\iota_h}}+\frac{d(H-h+1)}{\sqrt{\iota_h(\eta_h+\sigmanoise^2)}}\right) \cdot\log\left(\frac{d(H-h+1)^2K}{\kappa_h(\eta_h+\sigmanoise^2)\delta}\right)
    \end{align*}
\end{enumerate}
\end{lemma}

\begin{proof}[Proof of Lemma \ref{lemma: uniform convergence induction bernstein}]
We want to show $\PP \{\tilde\cE_1\cap \tilde\cE_2\}\geq 1-\delta$. It follows from Lemma \ref{lemma: uniform convergence hatLambda function class} and \eqref{eq: K lowerbound uc induction bernstein} that $\PP(\tilde\cE_1)\geq 1-\delta$.

To show that $\PP(\tilde\cE_2)\geq 1-\delta$, first by Lemma \ref{lemma: uniform convergence bernstein class}, we have
\begin{align}\label{eq: E2 uniform convergence 1}
    &\left\|\sum_{k=1}^K\hat\sigma_h(s_{k,h},a_{k,h})^{-2}\bphi(s_{k,h},a_{k,h})\left(\PP_hV(s_{k,h},a_{k,h})-V(s_{k,h}')-\epsilon_{k,h})\right)\right\|_{\hat\bLambda_h^{-1}}^2\notag\\
    \leq & 50\left(d + \frac{d(H-h+1)}{\sqrt{\eta_h+\sigmanoise^2}}\right)^2\cdot\log^2\left(\frac{K(H-h+1)^2L}{\kappa_h(\eta_h+\sigmanoise^2)\delta}\right),
\end{align}
for all $V\in\cV_{h+1}(L)$.

It follows from Cauchy-Schwartz inequality that
\begin{align*}
    &\bphi(s,a)^\top\hat\bLambda_h^{-1}\sum_{k=1}^K\hat\sigma_h(s_{k,h},a_{k,h})^{-2}\bphi(s_{k,h},a_{k,h})\left(\PP_hV(s_{k,h},a_{k,h})-V(s_{k,h}')-\epsilon_{k,h}\right)\\
    &\leq\|\bphi(s,a)\|_{\hat\bLambda_h^{-1}}\cdot\left\|\sum_{k=1}^K\hat\sigma_h(s_{k,h},a_{k,h})^{-2}\bphi(s_{k,h},a_{k,h})\left([\PP_hV](s_{k,h},a_{k,h})-V(s_{k,h}')-\epsilon_{k,h}\right)\right\|_{\hat\bLambda_h^{-1}}\\
    &\leq  \|\hat\bLambda_h^{-1}\|^{1/2}\cdot\left\|\sum_{k=1}^K\hat\sigma_h(s_{k,h},a_{k,h})^{-2}\bphi(s_{k,h},a_{k,h})\left([\PP_hV](s_{k,h},a_{k,h})-V(s_{k,h}')-\epsilon_{k,h}\right)\right\|_{\hat\bLambda_h^{-1}}\\
    &\leq  \frac{20}{\sqrt{K}}\cdot\left(\frac{d}{\sqrt{\iota_h}}+\frac{d(H-h+1)}{\sqrt{\iota_h(\eta_h+\sigmanoise^2)}}\right) \cdot\log\left(\frac{d(H-h+1)^2K}{\kappa_h(\eta_h+\sigmanoise^2)\delta}\right),
\end{align*}
where the second inequality follows from Assumption \ref{assump:linear_MDP} and the third inequality follows from \eqref{eq: E2 uniform convergence 0} and \eqref{eq: E2 uniform convergence 1}. Note that this holds for all $(s,a)\in\cS\times\cA$ as we directly bound the operator norm of $\hat\bLambda_h$. Replacing $\delta$ by $\delta/2$ completes the proof.
\end{proof}

\section{Lemmas for OPE Convergence}

\subsection{Concentration of $\hat\sigma$}\label{sec:concentration of hatsigma ope}
Recall that in the algorithm, to estimate the variance, we use $\hatbeta_h$ and $\hattheta_h$ which are estimated using the function $\hatV_{h+1}^\pi$ and $\{\check s_{k,h}, \check a_{k,h}, \check s'_{k,h} \}_{k\in[K]}$.

For the next lemma we denote the function $\sigma(\cdot,\cdot)$ as computed from some function $V(\cdot)$ and data $\check \cD_h$.
\begin{lemma}\label{lem:hatsigma concentration V class}
Let $\rho \geq 0$. For any $V \in \cV_{h+1}(L) \cap \{V: \sup_s|V(s)-V_{h+1}^\pi(s)|\leq \rho\} $, with probability at least $1-\delta$, we have 
\begin{align*}
        \left| \sigma^2(s,a) - \sigmanoise^2 - \max\left\{\eta_h, \ \VV_h V_{h+1}^\pi(s,a) \right\} \right| \leq \frac{C_{K,h,\delta} (H-h+1)^2\sqrt{d}}{\sqrt{K}} + 4(H-h+1)\cdot\rho,
\end{align*}for all $(s,a)\in\cS\times \cA$ where
\begin{align*}
    & C_{K,h,\delta} = 12\sqrt{2} \cdot \frac{1}{\sqrt{\kappa_h}} \cdot \left[ \frac{1}{2} \log\left( \frac{\lambda+K}{\lambda}\right) + \frac{1}{d} \log\frac{4}{\delta} \right]^{1/2}+ 12\lambda\cdot\frac{1}{\kappa_h}.
\end{align*}
\end{lemma}

\begin{proof}[Proof of Lemma \ref{lem:hatsigma concentration V class}]
By Lemma \ref{lem:hatsigma concentration fixed V uniform}, with probability at least $1-\delta$, we have
\begin{align*}
        \left| \sigma^2(s,a) - \sigmanoise^2 - \max\left\{\eta_h, \ \VV_h V(s,a) \right\} \right| \leq \frac{C_{K,h,\delta} (H-h+1)^2\sqrt{d}}{\sqrt{K}}.
\end{align*} Note that if two functions $f_1$, $f_2:\cS \to \RR$ satisfies $\sup_s |f_1(s)-f_2(s)|\leq \rho$, $\sup_s |f_1(s)|\leq H-h+1$, and $\sup_s |f_2(s)|\leq H-h+1$, then for all $(s,a)$,
\begin{align*}
    \left| \VV_h f_1(s,a) - \VV_h f_2(s,a) \right| \leq 4(H-h+1)\cdot \rho.
\end{align*}
Then using the triangular inequality and $|\VV_h V(s,a) - \VV_h V_{h+1}^\pi(s,a)|\leq 4(H-h+1)\rho$ finishes the proof.
\end{proof}

\subsection{Concentration of Weighted Sample Covariance Matrices}\label{sec:concentration of hatbLambda ope}
In this subsection, we study the concentration of the matrices $\hat\bLambda_h$, $h\in[H]$ to their population counterparts. 
Recall from Algorithm \ref{alg:weighted FQI} that for each $h \in [H]$, the matrix $\hat\bLambda_h$ is generated using the function $\hatsigma_h(\cdot,\cdot)$ and the dataset $\cD_h = \{ (s_{k,h},a_{k,h},r_{k,h}, s'_{k,h}) \}_{k \in [K]}$. Since the function $\hatsigma_h(\cdot,\cdot)$ itself is generated by $\hatV_{h+1}(\cdot)$ and the dataset $\check \cD_h = \{ (\check s_{k,h}, \check a_{k,h}, \check r_{k,h}, \check s'_{k,h}) \}_{k \in [K]}$, we can equivalently view $\hat\bLambda_h$ as generated by $\hatV_{h+1}(\cdot)$ and the datasets $\check \cD_h$ and $\cD_h$. In the remaining of the subsection, we will omit the subscript and superscript when it is clear and simply write
\begin{align*}
    \hat\bLambda_h & = \sum_{k=1}^K \bphi(s_{k,h},a_{k,h}) \bphi^\top(s_{k,h},a_{k,h})/ \sigma^2(s_{k,h},a_{k,h}) + \lambda \Ib_d,
\end{align*}where $\sigma(\cdot,\cdot)$ is generated using the function $V(\cdot)$ and the dataset $\check \cD_h$ as described in Algorithm \ref{alg:weighted FQI}. We also denote \begin{align*}
    \sigma_V^2(\cdot,\cdot) := \max\left\{ \eta_h, \ \VV_h V(\cdot,\cdot) \right\} + \sigmanoise^2.
\end{align*} By Lemma \ref{lem:hatsigma concentration fixed V uniform}, we know that with high probability, $\sigma^2(\cdot,\cdot)$ will be a good estimator for $\sigma_V^2(\cdot,\cdot)$. This will be used to show the concentration of the matrix $\hat\bLambda_h$. We start from the next lemma.

\begin{lemma}\label{lem:barLambda concentration OPE fixed sigma}
For any $h\in[H]$, conditioning on $\sigma(\cdot,\cdot)\in\cT_h(L_1,L_2)$ being fixed, with conditional probability at least $1-\delta$,
\begin{align*}
    \left\| \frac{\hat\bLambda_h}{K} - \EE_{\bar\pi,h}\left[ \frac{\bphi(s,a)\bphi(s,a)^\top}{\sigma^2(s,a)} \right] \right\| \leq \frac{4\sqrt{2}}{(\eta_h+\sigmanoise^2) \sqrt{K}}\cdot \left( \log\frac{2d}{\delta}\right)^{1/2} + \frac{\lambda}{K} . 
\end{align*}
\end{lemma}

\begin{proof}[Proof of Lemma \ref{lem:barLambda concentration OPE fixed sigma}]
Since $\sigma(\cdot,\cdot)$ is a function of $V$ and the dataset $\check \cD_h$ which is independent of $\cD_h$, conditioning on $\sigma(\cdot,\cdot)$ won't change the distribution of $\cD_h$. In other words, $\bphi(s_{k,h},a_{k,h})/\sigma(s_{k,h},a_{k,h})$, $k\in[K]$ can be viewed as independent random vectors. Then by Lemma \ref{lemma:Sigma matrix concentration}, we have that, with conditional probability at least $1-\delta$,
\begin{align*}
    \Bignorm{\frac{\hat\bLambda_h}{K} - \EE_{\bar\pi,h}\left[\frac{\hat\bLambda_h}{K}  \right]} \leq \frac{4\sqrt{2}}{(\eta_h+\sigmanoise^2) \sqrt{K}}\cdot \sqrt{ \log\left(\frac{2d}{\delta}\right)},
\end{align*}and thus 
\begin{align*}
    \Bignorm{\frac{\hat\bLambda_h}{K} - \EE_{\bar\pi,h}\left[\frac{\bphi(s,a)\bphi(s,a)^\top}{\sigma^2(s,a)}  \right]} & \leq \Bignorm{\frac{\hat\bLambda_h}{K} - \EE_{\bar\pi,h}\left[\frac{\hat\bLambda_h}{K} \right]} + \Bignorm{\EE_{\bar\pi,h}\left[\frac{\bphi(s,a)\bphi(s,a)^\top}{\sigma^2(s,a)} \right] - \EE_{\bar\pi,h}\left[\frac{\hat\bLambda_h}{K}  \right]}
    \\ & \leq \frac{4\sqrt{2}}{(\eta_h+\sigmanoise^2) \sqrt{K}}\cdot\sqrt{ \log\left(\frac{2d}{\delta}\right)} + \frac{\lambda}{K} . 
\end{align*}

\end{proof}

Next, combine Lemma \ref{lem:barLambda concentration OPE fixed sigma} and the event that $\sigma^2(\cdot,\cdot)$ is a good estimator for $\sigma_V^2(\cdot,\cdot)$, we get the following lemma.

\begin{lemma}\label{lem:barLambda concentration fixed V OPE}
For any $h\in[H]$, condition on $V \in \cV_{h+1}(L)$ being fixed, with conditional probability at least $1-\delta$, 
\begin{align*}
        &\left\| \frac{\hat\bLambda_h}{K} - \EE_{\bar\pi,h}\left[ \frac{\bphi(s,a)\bphi(s,a)^\top}{\sigma_V^2(s,a)} \right] \right\| 
        \\ &\leq  \frac{4\sqrt{2}}{(\eta_h+\sigmanoise^2) \sqrt{K}}\cdot \left( \log\frac{4d}{\delta}\right)^{1/2} + \frac{\lambda}{K} + \frac{1}{(\eta_h+\sigmanoise^2)^2} \cdot \frac{C_{K,h,\delta} (H-h+1)^2\sqrt{d}}{\sqrt{K}}, 
\end{align*}where
\begin{align*}
    C_{K,h,\delta} & = 12\sqrt{2} \cdot \frac{1}{\sqrt{\kappa_h}} \cdot \left[ \frac{1}{2} \log\left( \frac{\lambda+K}{\lambda}\right) + \frac{1}{d} \log\frac{8}{\delta} \right]^{1/2}+ 12\lambda\cdot \frac{1}{\kappa_h}. 
\end{align*}
\end{lemma}

\begin{proof}[Proof of Lemma \ref{lem:barLambda concentration fixed V OPE}]
First note that condition on $\sigma(\cdot,\cdot)\in \cT_{h}(L_1,L_2)$ such that $\sup_{s,a} |\sigma^2(s,a)-\sigma_V^2(s,a)|\leq \rho$ for some $\rho\geq 0$, we have 
\begin{align*}
    &\left\| \EE_{\bar\pi,h}\left[\frac{\bphi(s,a)\bphi(s,a)^\top}{\sigma^2(s,a)}  \right] - \EE_{\bar\pi,h}\left[\frac{\bphi(s,a)\bphi(s,a)^\top}{\sigma_V^2(s,a)}  \right]  \right\| \\
    & \leq \EE_{\bar\pi,h} \left[ \left\| \bphi(s,a)\bphi(s,a)^\top\right\| \sup_{s,a}\left( \frac{1}{\sigma^2(s,a)} - \frac{1}{\sigma_V^2(s,a)}\right) \right] 
    \\ & \leq \frac{1}{(\eta_h+\sigmanoise^2)^2} \cdot \rho,
\end{align*} since $\sigma^2(s,a)$ and $\sigma_V^2(s,a)$ are lower bounded by $\eta_h+\sigmanoise^2$. Then by Lemma \ref{lem:barLambda concentration OPE fixed sigma}, we have that, conditioning on fixed $\sigma(\cdot,\cdot)$ s.t. $\sup_{s,a} |\sigma^2(s,a)-\sigma_V^2(s,a)|\leq \rho$, with conditional probability at least $1-\delta$,
\begin{align}\label{eq:barLambda concentration OPE fixed V eq1}
    \left\| \frac{\hat\bLambda_h}{K} - \EE_{\bar\pi,h}\left[ \frac{\bphi(s,a)\bphi(s,a)^\top}{\sigma_V^2(s,a)} \right] \right\| \leq \frac{4\sqrt{2}}{(\eta_h+\sigmanoise^2) \sqrt{K}}\cdot \left( \log\frac{2d}{\delta}\right)^{1/2} + \frac{\lambda}{K} + \frac{1}{(\eta_h+\sigmanoise^2)^2} \cdot \rho . 
\end{align}Since conditioning on $V(\cdot)$ won't change the distribution of $\check D_h$ under Assumption \ref{assump:data_generation trajectory}, by Lemma \ref{lem:hatsigma concentration fixed V uniform}, with probability at least $1-\delta$, it holds for all $(s,a)\in\cS\times\cA$ that 
\begin{align}\label{eq:barLambda concentration OPE fixed V eq2}
        \left| \sigma^2(s,a) - \sigma_V^2(s,a)\right| \leq \frac{C_{K,h,\delta} (H-h+1)^2\sqrt{d}}{\sqrt{K}},
\end{align}where
\begin{align*}
    & C_{K,h,\delta} = 12\sqrt{2} \cdot \frac{1}{\sqrt{\kappa_h}} \cdot \left[ \frac{1}{2} \log\left( \frac{\lambda+K}{\lambda}\right) + \frac{1}{d} \log\frac{4}{\delta} \right]^{1/2}+ 12\lambda\cdot \frac{1}{\kappa_h}.
\end{align*} Combine \eqref{eq:barLambda concentration OPE fixed V eq1} and \eqref{eq:barLambda concentration OPE fixed V eq2}, and we get that, condition on $V$, with probability at least $1-2\delta$, 
\begin{align*}
        &\left\| \frac{\hat\bLambda_h}{K} - \EE_{\bar\pi,h}\left[ \frac{\bphi(s,a)\bphi(s,a)^\top}{\sigma_V^2(s,a)} \right] \right\| 
        \\ &\leq  \frac{4\sqrt{2}}{(\eta_h+\sigmanoise^2) \sqrt{K}}\cdot \left( \log\frac{2d}{\delta}\right)^{1/2} + \frac{\lambda}{K} + \frac{1}{(\eta_h+\sigmanoise^2)^2} \cdot \frac{C_{K,h,\delta} (H-h+1)^2\sqrt{d}}{\sqrt{K}}. 
\end{align*}Replacing $\delta$ with $\delta/2$ finishes the proof. 
\end{proof}

Finally, combining Lemma \ref{lem:barLambda concentration fixed V OPE} and the event of uniform convergence, we can bound the distance between $\hat\bLambda_h$ and its population counterpart $\bLambda_h$. 

\begin{lemma}\label{lem:barLambda concentration V class OPE}
For any $h\in[H]$, condition on $V \in \cV_{h+1}(L) \cap \{V: \sup_s|V(s)-V_{h+1}^\pi(s)|\leq \rho\} $, with conditional probability at least $1-\delta$, we have 
\begin{align*}
        & \left\| \frac{\hat\bLambda_h}{K} - \bLambda_h \right\| 
       \\ &\leq  \frac{4\sqrt{2}}{(\eta_h+\sigmanoise^2) \sqrt{K}}\cdot \left( \log\frac{4d}{\delta}\right)^{1/2} + \frac{\lambda}{K} + \frac{1}{(\eta_h+\sigmanoise^2)^2} \cdot \left( \frac{C_{K,h,\delta} (H-h+1)^2\sqrt{d}}{\sqrt{K}} +4(H-h+1)\cdot\rho\right) , 
\end{align*}where
\begin{align*}
    C_{K,h,\delta} & = 12\sqrt{2}  \cdot \frac{1}{\sqrt{\kappa_h}} \cdot \left[ \frac{1}{2} \log\left( \frac{\lambda+K}{\lambda}\right) + \frac{1}{d} \log\frac{8}{\delta} \right]^{1/2}+ 12\lambda \cdot \frac{1}{\kappa_h}. 
\end{align*}
\end{lemma}

\begin{proof}[Proof of Lemma \ref{lem:barLambda concentration V class OPE}]
First note that by Lemma \ref{lem:barLambda concentration fixed V OPE}, with probability at least $1-\delta$, 
\begin{align*}
        &\left\| \frac{\hat\bLambda_h}{K} - \EE_{\bar\pi,h}\left[ \frac{\bphi(s,a)\bphi(s,a)^\top}{\sigma_V^2(s,a)} \right] \right\| 
        \\ &\leq  \frac{4\sqrt{2}}{(\eta_h+\sigmanoise^2) \sqrt{K}}\cdot \left( \log\frac{4d}{\delta}\right)^{1/2} + \frac{\lambda}{K} + \frac{1}{(\eta_h+\sigmanoise^2)^2} \cdot \frac{C_{K,h,\delta} (H-h+1)^2\sqrt{d}}{\sqrt{K}}.
\end{align*} On the other hand, by $\sup_s|V(s)-V_{h+1}^\pi(s)|\leq \rho$ and $|V(s)|$, $|V_{h+1}^\pi(s)|\leq H-h+1$, we have $\sup_s|\sigma_V^2(s)-\sigma_h^2(s)|\leq 4(H-h+1)\rho$. It implies that
\begin{align*}
    \left\| \EE_{\bar\pi,h}\left[ \frac{\bphi(s,a)\bphi(s,a)^\top}{\sigma_V^2(s,a)} \right] - \EE_{\bar\pi,h}\left[ \frac{\bphi(s,a)\bphi(s,a)^\top}{\sigma_h^2(s,a)} \right] \right\| 
    \leq \frac{1}{(\eta_h+\sigmanoise^2)^2}\cdot 4(H-h+1)\cdot \rho .
\end{align*}Then by triangular inequality, we conclude that 
\begin{align*}
    & \left\| \frac{\hat\bLambda_h}{K} - \EE_{\bar\pi,h}\left[\frac{\bphi(s,a)\bphi(s,a)^\top}{\sigma_h^2(s,a)}\right]\right\| 
    \\ &\leq  \frac{4\sqrt{2}}{(\eta_h+\sigmanoise^2) \sqrt{K}}\cdot \left( \log\frac{4d}{\delta}\right)^{1/2} + \frac{\lambda}{K} + \frac{1}{(\eta_h+\sigmanoise^2)^2} \cdot \left( \frac{C_{K,h,\delta} (H-h+1)^2\sqrt{d}}{\sqrt{K}} +4(H-h+1)\cdot\rho\right) .
\end{align*}Finally, recall the definition of $\bLambda_h$ given by \eqref{def: Lambda}.
\end{proof}

\subsection{Bound for the self-normalized martingales}\label{sec:self normalized OPE}

\begin{lemma}\label{lem:self normalized fix hatV}
For any $h \in [H]$, condition on $\hatV_{h+1}^\pi \in \cV_{h+1}(L)$ s.t. $\sup_s \left|\hatV_{h+1}^\pi(s) \right|\leq B$, with conditional probability at least $1-\delta$,
\begin{align*}
    & \left\|{\sum_{k=1}^K\bphi(\check s_{k,h},\check a_{k,h})\left[(\hat V_{h+1}^\pi)(\check s'_{k,h})-\PP_h(\hat V_{h+1}^\pi)(\check s_{k,h},\check a_{k,h})\right]}\right\|^2_{\hat\bSigma_h^{-1}}\leq  8B^2 \left[  \frac{d}{2} \log\left( \frac{\lambda+K}{\lambda}\right) + \log\frac{2}{\delta} \right] ,
        \\ & \left\|{\sum_{k=1}^K\bphi(\check s_{k,h},\check a_{k,h})\left[(\hat V_{h+1}^\pi)^2(\check s'_{k,h})-\PP_h(\hat V_{h+1}^\pi)^2(\check s_{k,h},\check a_{k,h})\right]}\right\|^2_{\hat\bSigma_h^{-1}}\leq  8B^4 \left[  \frac{d}{2} \log\left( \frac{\lambda+K}{\lambda}\right) + \log\frac{2}{\delta} \right].
\end{align*}
\end{lemma}
\begin{proof}[Proof of Lemma \ref{lem:self normalized fix hatV}]
Denote $\xb_k =\bphi(\check s_{k,h},\check a_{k,h})$, and $\eta_k = (\hat V_{h+1}^\pi)(\check s'_{k,h})-\PP_h(\hat V_{h+1}^\pi)(\check s_{k,h},\check a_{k,h}) $. 

Define the filtration $\{ \cF_k \}_{k=0}^K$ by $\cF_0 = \sigma(\check s_{1,h}, \check a_{1,h})$, $\cF_1=\sigma(\check s_{1,h}, \check a_{1,h},\check s'_{1,h}, \check s_{2,h}, \check a_{2,h})$ , $\cdots$, $\cF_k = \sigma(\check s_{1,h}, \check a_{1,h}, \check s'_{1,h}, \cdots, \check s_{k,h}, \check a_{k,h}, \check s'_{k,h}, \check s_{k+1,h}, \check a_{k+1,h})$ for $k = 1, \cdots, K-1$, and $\cF_K = \sigma( \cF_{K-1}, \check s'_{K,h})$. Then we see that $\xb_k$ is $\cF_{k-1}$-measurable, and $\eta_k$ is $\cF_k$-measurable. Furthermore, since $\EE[(\hatV_{h+1}^\pi)^2(\check s_{k,h}')\mid \cF_{k-1}] = \PP_h(\hatV_{h+1}^\pi)^2(\check s_{k,h}, \check a_{k,h})$, $\eta_k \mid \cF_{k-1}$ is zero-mean. Also, $|\eta_k|\leq2B$, which implies that $\eta_k \mid \cF_{k-1}$ is $2B$-subgaussian. Then by \ref{thm:self normalized hoeffding}, with probability at least $1-\delta/2$, 
\begin{align*}
    \left\|\sum_{k=1}^K\xb_k \eta_k\right\|^2_{\hat\bSigma_h^{-1}} \leq 8B^2 \log \left( \frac{ \det(\hat\bSigma_h)^{1/2} \det(\lambda I)^{-1/2} }{\delta/2}\right).
\end{align*}Recall that $\hat\bSigma_h = \sum_{k=1}^K \bphi(\check{s}_{k,h},\check{a}_{k,h})\bphi^\top(\check{s}_{k,h},\check{a}_{k,h}) + \lambda\Ib_d$ where $\left\| \bphi\right\| \leq 1$. It follows that 
\begin{align*}
    \det (\hat\bSigma_h) \leq (\lambda + K)^d.
\end{align*}We then conclude that 
\begin{align*}
    \left\|\sum_{k=1}^K\xb_k \eta_k\right\|^2_{\hat\bSigma_h^{-1}} \leq 8B^2 \left[  \frac{d}{2} \log\left( \frac{\lambda+K}{\lambda}\right) + \log\frac{2}{\delta} \right] . 
\end{align*}The second inequality is similar. Taking a union bound finishes the proof.
\end{proof}

\section{Auxiliary Lemmas}\label{sec:auxiliary lemmas}

\subsection{Concentration Inequalities}

\begin{lemma}[Matrix McDiarmid inequality, \citealt{tropp2012user}]\label{lem:matrix McDiarmid}
Let $\zb_k$, $k=1,\cdots, K$ be independent random vectors in $\RR^d$, and let $\Hb$ be a function that maps $K$ vectors to a $d\times d$ symmetric matrix. Assume there exists a sequence of fixed symmetric matrices $\{\Ab_k \}_{k \in[K]}$ such that
\begin{align*}
    \left( \Hb\left(\zb_1, \cdots, \zb_k, \cdots, \zb_K \right) - \Hb\left(\zb_1, \cdots, \zb'_k, \cdots, \zb_K \right)\right)^2 \preceq \Ab_k^2,
\end{align*}where $\zb_k, \zb'_k$ ranges over all possible values for each $k \in [K]$. Define $\sigma^2$ as
\begin{align*}
    \sigma^2 := \bignorm{\sum_k \Ab_k^2}.
\end{align*}Then, for any $t>0$, 
\begin{align*}
    \PP \left\{  \lambda_{\max} \left( \Hb(\zb) - \EE\Hb(\zb) \right)\geq t \right\} \leq d \cdot \exp\left(  \frac{-t^2}{8 \sigma^2}\right),
\end{align*}where $\zb = (\zb_1,\cdots,\zb_K)$.
\end{lemma}

\begin{lemma}[Freedman's inequality for martingales, \citealt{freedman1975tail}]\label{thm:freedman}
Consider a martingale difference sequence $\{e_k, \ k=1,2,3,\cdots\}$ with filtration $\cF_k \coloneqq \sigma\{e_1,\cdots, e_{k-1}\}$, for $k=1,2,\cdots$. Assume $e_k$ is uniformly bounded:
\begin{align*}
    |e_k| \leq R \quad \textnormal{almost surely for} \quad k=1,2,3,\cdots
\end{align*}Then for all $\epsilon\geq 0$ and $\sigma^2>0$, 
\begin{align*}
    \PP \left\{ \exists K>0:\ \left|\sum_{k=1}^K e_k\right|  \geq \epsilon, \ \sum_{k=1}^K\Var[e_k \mid \cF_k] \leq \sigma^2 \right\} \leq 2\exp\left( \frac{-\epsilon^2/2}{\sigma^2 + R\epsilon/3}\right).
\end{align*}
\end{lemma}

\subsection{Basic Matrix Inequalities}

\begin{lemma}\label{lem:matrix basic}
Assume $\Gb_1$ and $\Gb_2 \in \RR^{d\times d}$ are two positive semi-definite matrices. Then we have
\begin{align*}
    \left\| \Gb_1^{-1} \right\| & \leq \left\| \Gb_2^{-1} \right\| + \left\| \Gb_1^{-1} \right\|\cdot \left\| \Gb_2^{-1} \right\| \cdot \left\| \Gb_1 - \Gb_2 \right\|
\end{align*}and
\begin{align*}
    \left\| \ub\right\|_{\Gb_1^{-1}} \leq \left[1+ \sqrt{ \left(\left\| \Gb_2^{-1} \right\|\cdot\left\| \Gb_2 \right\| \right)^{1/2} \cdot\left\| \Gb_1^{-1} \right\|\cdot \left\| \Gb_1-\Gb_2 \right\| } \right] \cdot \left\|\ub \right\|_{\Gb_2^{-1}},
\end{align*}for all $\ub \in \RR^d$.
\end{lemma}

\begin{proof}[Proof of Lemma \ref{lem:matrix basic}]
The first inequality is by 
\begin{align*}
    \left\| \Gb_1^{-1} \right\| & \leq \left\| \Gb_2^{-1} \right\| + \left\| \Gb_2^{-1} - \Gb_1^{-1} \right\|
    \leq \left\| \Gb_2^{-1} \right\| + \left\| \Gb_2^{-1} \right\| \cdot \left\| \Gb_2 - \Gb_1 \right\|\cdot \left\| \Gb_1^{-1} \right\|.
\end{align*}To prove the second one, note that
\begin{align*}
        \left\| \ub \right\|_{\Gb_1^{-1}} & = \sqrt{\ub^\top \Gb_1^{-1}\ub}
        \\ & =\sqrt{\ub^\top(\Gb_1^{-1}-\Gb_2^{-1})\ub+\ub^\top \Gb_2^{-1}\ub}
        \\ & =\sqrt{\ub^\top \Gb_2^{-1/2} \left[ \Ib + (\Gb_2^{1/2}\Gb_1^{-1}\Gb_2^{1/2}-\Ib) \right]  \Gb_2^{-1/2}\ub } 
        \\ & \leq \left( 1 + \left\| \Gb_2^{1/2}\Gb_1^{-1}\Gb_2^{1/2}-\Ib \right\|^{1/2} \right) \cdot \left\|\ub\right\|_{\Gb_2^{-1}}, 
\end{align*} and the rest follows from 
\begin{align*}
    \left\| \Gb_2^{1/2}\Gb_1^{-1}\Gb_2^{1/2}-\Ib \right\| & = \left\| \Gb_2^{1/2}(\Gb_1^{-1}-\Gb_2^{-1})\Gb_2^{1/2} \right\|
    \\ & = \left\| \Gb_2^{1/2}\Gb_1^{-1}(\Gb_1-\Gb_2)\Gb_2^{-1}\Gb_2^{1/2} \right\|
    \\ & \leq \left(\left\| \Gb_2^{-1} \right\|\cdot\left\| \Gb_2 \right\| \right)^{1/2} \cdot\left\| \Gb_1^{-1} \right\|\cdot \left\| \Gb_1-\Gb_2 \right\| . 
\end{align*}
\end{proof}

\begin{lemma}\label{lemma:Sigma matrix concentration}
Let $\bvarphi:\cS\times\cA\to\RR^d$ be a bounded function such that $|\bvarphi(s,a)|\leq C$ for all $(s,a)\in\cS\times\cA$. For any $K>0$ and $\lambda>0$, define $\bar\Gb_K=\sum_{k=1}^K\bvarphi(s_k,a_k)\bvarphi(s_k,a_k)^\top +\lambda\Ib_d$ where $(s_k,a_k)$'s are i.i.d samples from some distribution $\nu$ over $\cS\times\cA$. Then with probability at least $1-\delta$, it holds that
\begin{align*}
    \left\| \frac{\bar\Gb_K}{K} - \EE_\nu\left[\frac{\bar\Gb_K}{K}  \right]\right\| \leq \frac{4\sqrt{2}C^2}{\sqrt{K}} \left( \log \frac{2d}{\delta}  \right)^{1/2}.
\end{align*}
\end{lemma}
\begin{proof}[Proof of Lemma \ref{lemma:Sigma matrix concentration}]
Denote $\xb_k = \bvarphi(s_k,a_k)$. Denote $\tilde\bSigma_h$ as the matrix obtained by replacing the $k$-th vector $\xb_k$ in $\hat\bSigma_h$ by $\tilde\xb_k$ and leaving the rest $K-1$ vectors unchanged. Then we have 
\begin{align*}
        \left(\frac{\hat\bSigma_h}{K} - \frac{\tilde\bSigma_h}{K}\right)^2 & = \left( \frac{\xb_k\xb_k^\top - \tilde\xb_k \tilde\xb_k^\top}{K} \right)^2
        \\ & \preceq \frac{1}{K^2} \left( 2\xb_k\xb_k^\top \xb_k\xb_k^\top +2 \tilde\xb_k\tilde\xb_k^\top\tilde\xb_k\tilde\xb_k^\top \right)
        \\ & \preceq \frac{1}{K^2} \left( 2C^4\Ib_d + 2C^4\Ib_d \right)
        \\ & = \frac{4C^4}{K^2}\cdot \Ib_d 
        \\ & := \bA_k^2,
\end{align*}where the first inequality uses the fact that $(\bA-\bB)^2 \preceq2\bA^2 + 2 \bB^2$ for all p.s.d. matrices $\bA$ and $\bB$, the second inequality is from $\norm{\bvarphi}\leq C$. Note that we have 
\begin{align*}
    \bignorm{\sum_{k} \bA_k^2}  = \frac{4C^4}{K}.
\end{align*}Then by Lemma \ref{lem:matrix McDiarmid}, we have: for all $t>0$, 
\begin{align*}
    \PP \left\{ \left\| \frac{\hat\bSigma_h}{K} - \EE\left[\frac{\hat\bSigma_h}{K}  \right]\right\| \geq t\right\} \leq 2d \cdot \exp \left( \frac{-t^2K}{32C^4} \right).
\end{align*}Equivalently,  with probability at least $1-\delta$,
\begin{align*}
    \left\| \frac{\hat\bSigma_h}{K} - \EE\left[\frac{\hat\bSigma_h}{K}  \right]\right\| \leq \frac{4\sqrt{2}C^2}{\sqrt{K}} \left( \log \frac{2d}{\delta}  \right)^{1/2}.
\end{align*}
This completes the proof.
\end{proof}

\begin{lemma}\label{lemma: quadratic form}
Let $\bvarphi:\cS\times\cA\to\RR^d$ be a bounded function such that $\|\bvarphi(s,a)\|_2\leq C$ for all $(s,a)\in\cS\times\cA$. For any $K>0$ and $\lambda>0$, define $\bar\Gb_K=\sum_{k=1}^K\bvarphi(s_k,a_k)\bvarphi(s_k,a_k)^\top +\lambda\Ib_d$ where $(s_k,a_k)$'s are i.i.d samples from some distribution $\nu$ over $\cS\times\cA$. Let $\Gb=\EE_\nu[\bvarphi(s,a)\bvarphi(s,a)^\top]$. Then for any $\delta\in(0,1)$, if $K$ satisfies that
\begin{align}\label{eq:K lower bound quadratic form}
    K\geq\max\left\{512C^4\|\Gb^{-1}\|^2\log\left(\frac{2d}{\delta}\right),4\lambda\|\Gb^{-1}\|\right\}.
\end{align}
Then with probability at least $1-\delta$, it holds simultaneously for all $\ub\in\RR^d$ that
\begin{align*}
    \|\ub\|_{\bar\Gb_K^{-1}}\leq \frac{2}{\sqrt{K}}\|\ub\|_{\Gb^{-1}}.
\end{align*}
\end{lemma}
\begin{proof}[Proof of Lemma \ref{lemma: quadratic form}]
Note that
\begin{align}\label{eq: quadratic form 1}
    \|\ub\|_{\bar\Gb_K^{-1}} &= \frac{1}{\sqrt{K}}\sqrt{\ub^\top \Gb^{-1}\ub + \ub^\top\left[\left(\frac{\bar\Gb_K}{K}\right)^{-1}-\Gb^{-1}\right]\ub}\notag\\
    &= \frac{1}{\sqrt{K}} \sqrt{\ub^\top\Gb^{-1}\ub+\ub^\top\Gb^{-1/2}\left[\Gb^{1/2}\left(\frac{\bar\Gb_K}{K}\right)^{-1}\Gb^{1/2}-\Ib_d\right]\Gb^{-1/2}\ub}\notag\\
    &\leq \frac{1}{\sqrt{K}}\left(1 + \left\|\Gb^{1/2}\left(\frac{\bar\Gb_K}{K}\right)^{-1}\Gb^{1/2}-\Ib_d\right\|^{1/2} \right)\|\ub\|_{\Gb^{-1}},
\end{align}
where the last inequality follows from Cauchy-Schwartz inequality.

It then reduces to bound $\left\|\Gb^{1/2}\left(\bar\Gb_K/K\right)^{-1}\Gb^{1/2}-\Ib_d\right\|$, which can be further bounded by
\begin{align}\label{eq: quadratic form 2}
    \left\|\Gb^{1/2}\left(\frac{\bar\Gb_K}{K}\right)^{-1}\Gb^{1/2}-\Ib_d\right\| &\leq \left\|\left[\Gb^{-1/2}\frac{\bar\Gb_K}{K}\Gb^{-1/2}\right]^{-1}\right\|\cdot \left\|\Ib_d-\Gb^{-1/2}\frac{\bar\Gb_K}{K}\Gb^{-1/2}\right\|.
\end{align}
By Lemma \ref{lemma:Sigma matrix concentration}, we have 
\begin{align*}
    \left\| \frac{\bar\Gb_K}{K} - \EE\left[\frac{\bar\Gb_K}{K}  \right]\right\| \leq \frac{4\sqrt{2}C^2}{\sqrt{K}} \left( \log \frac{2d}{\delta}  \right)^{1/2}
\end{align*}
with probability at least $1-\delta$, and thus
\begin{align}\label{eq: quadratic form 3}
    \left\|\Ib-\Gb^{-1/2}\frac{\bar\Gb_K}{K}\Gb^{-1/2} \right\| &\leq \left[\left\|\frac{\bar\Gb_K}{K}-\EE\left[\frac{\bar\Gb_K}{K}\right]\right\|+\left\|\EE\left[\frac{\bar\Gb_K}{K}\right] - \Gb\right\|\right]\cdot\|\Gb^{-1}\|\notag\\
    &\leq \frac{4\sqrt{2}C^2\|\Gb^{-1}\|}{\sqrt{K}}\sqrt{\log\frac{2d}{\delta}}+\frac{\lambda\|\Gb^{-1}\|}{K}\notag\\
    &\leq \frac{1}{2}
\end{align}
where the last inequality follows from the assumption \eqref{eq:K lower bound quadratic form}. Therefore,
\begin{align*}
    \lambda_{\min}\left(\Gb^{-1/2}\frac{\bar\Gb_K}{K}\Gb^{-1/2}\right) &\geq 1- \left\|\Ib-\Gb^{-1/2}\frac{\bar\Gb_K}{K}\Gb^{-1/2} \right\|\geq \frac{1}{2}
\end{align*}
with probability at least $1-\delta$. This further implies that
\begin{align}\label{eq: quadratic form 4}
    \left\|\left[\Gb^{-1/2}\frac{\bar\Gb_K}{K}\Gb^{-1/2}\right]^{-1}\right\|&=\lambda_{\min}\left(\Gb^{-1/2}\frac{\bar\Gb_K}{K}\Gb^{-1/2}\right)^{-1}\leq 2.
\end{align}
Combining \eqref{eq: quadratic form 2}, \eqref{eq: quadratic form 3} and \eqref{eq: quadratic form 4} yields that
\begin{align}\label{eq: quadratic form 5}
    \left\|\Gb^{1/2}\left(\frac{\bar\Gb_K}{K}\right)^{-1}\Gb^{1/2}-\Ib_d\right\|\leq 1
\end{align}
with probability at least $1-\delta$. Then plug \eqref{eq: quadratic form 5} back into \eqref{eq: quadratic form 1}, and we obtain that
\begin{align*}
    \|\ub\|_{\bar\Gb_K^{-1}}\leq \frac{2}{\sqrt{K}}\|\ub\|_{\Gb^{-1}}
\end{align*}
with probability at least $1-\delta$. Note that in the above argument we only need to bound $\left\|\Gb^{1/2}\left(\bar\Gb_K/K\right)^{-1}\Gb^{1/2}-\Ib_d\right\|$ which is independent of the choice of $\ub$, thus it holds for all $\ub\in\RR^d$ simultaneously. This completes the proof.
\end{proof}

\subsection{Inequalities for Sample Covariance Matrices}

Here we introduce some useful lemmas about the inverse Gram matrix. 

\begin{lemma}[Lemma D.1, \citealt{jin2020provably}]\label{lem:auxiliary inverse middle}
Let $\bLambda_t = \sum_{i=1}^t \xb_i \xb_i^\top + \lambda I$ where $\lambda >0$ and $\xb_i \in \RR^d$. Then
\begin{align*}
    \sum_{i=1}^t \xb_i^\top \bLambda_t^{-1} \xb_i \leq d. 
\end{align*}
\end{lemma}
\begin{proof}[Proof of Lemma \ref{lem:auxiliary inverse middle}]
Note that
\begin{align*}
    \sum_{i=1}^t \xb_i^\top \bLambda_t^{-1} \xb_i = \sum_{i=1}^t \tr(\xb_i^\top \bLambda_t^{-1} \xb_i) = \tr \left( \bLambda_t^{-1} \sum_{i=1}^t \xb_i\xb_i^\top \right).
\end{align*}Using the eigen-decomposition $\sum_{i=1}^t \xb_i \xb_i^\top = \Ub \textnormal{diag}(\lambda_1,\cdots, \lambda_d)\Ub^\top$, we have $\bLambda_t = \Ub \textnormal{diag}(\lambda_1+1,\cdots, \lambda_d+1)\Ub^\top$, and it follows that
\begin{align*}
    \tr \left( \bLambda_t^{-1} \sum_{i=1}^t \xb_i\xb_i^\top \right) = \sum_{j=1}^d \frac{\lambda_j}{\lambda_j+\lambda} \leq  d.
\end{align*}
\end{proof}

\begin{lemma}\label{lem:auxiliary inverse matrix distance}
For any $h\in[H]$ and $L_1,L_2>0$, let $\sigma_1$, $\sigma_2 \in \cT_h(L_1,L_2)$ such that $\sup_{s,a}|\sigma_1(s,a) - \sigma_2(s,a)| \leq \epsilon$.
Define
\begin{align*}
    &\bLambda_1 := \sum_{k=1}^K \bphi(s_{k,h},a_{k,h}) \bphi(s_{k,h},a_{k,h})^\top/ \sigma_1(s_{k,h}, a_{k,h})^2 + \lambda \Ib_d \ , \\
    & \bLambda_2 := \sum_{k=1}^K \bphi(s_{k,h},a_{k,h}) \bphi(s_{k,h},a_{k,h})^\top/ \sigma_2(s_{k,h}, a_{k,h})^2 + \lambda \Ib_d \ .
\end{align*}Then under Assumption \ref{assump:linear_MDP}, it holds that
\begin{align*}
    \norm{\bLambda_1 - \bLambda_2} & \leq \frac{2K \sqrt{(H-h+1)^2 + \sigmanoise^2}\cdot \epsilon}{(\eta_h+\sigmanoise^2)^2},
\end{align*}
and
\begin{align*}
    \norm{\bLambda_1^{-1} - \bLambda_2^{-1}} \leq \frac{2K \sqrt{(H-h+1)^2 + \sigmanoise^2}\cdot \epsilon}{\lambda^2 (\eta_h+\sigmanoise^2)^2}. 
\end{align*}
\end{lemma}

\begin{proof}[Proof of Lemma \ref{lem:auxiliary inverse matrix distance}]
We have 
\begin{align*}
    {\bLambda_1 - \bLambda_2} & = \sum_{k=1}^K \bphi(s_{k,h},a_{k,h}) \bphi^\top(s_{k,h},a_{k,h}) \left(\frac{1}{\sigma_1^2(s_{k,h},a_{k,h})} - \frac{1}{\sigma_2^2(s_{k,h},a_{k,h})}\right)
\end{align*}and thus 
\begin{align*}
        \norm{\bLambda_1 - \bLambda_2} & \leq \sum_{k=1}^K \norm{\bphi(s_{k,h},a_{k,h}) \bphi^\top(s_{k,h},a_{k,h})}\cdot \left|\frac{1}{\sigma_1^2(s_{k,h},a_{k,h})} - \frac{1}{\sigma_2^2(s_{k,h},a_{k,h})}\right|
        \\ & \leq \sum_{k=1}^K \left|\frac{1}{\sigma_1^2(s_{k,h},a_{k,h})} - \frac{1}{\sigma_2^2(s_{k,h},a_{k,h})}\right| 
        \\ & = \sum_{k=1}^K \left| \frac{|\sigma_1(s_{k,h},a_{k,h})+\sigma_2(s_{k,h},a_{k,h})|\cdot|\sigma_1(s_{k,h},a_{k,h})-\sigma_2(s_{k,h},a_{k,h})| }{\sigma_1^2(s_{k,h},a_{k,h}) \sigma_2^2(s_{k,h},a_{k,h})}\right|
        \\ & \leq K \cdot \frac{2\sqrt{(H-h+1)^2 + \sigmanoise^2}  }{(\eta_h+\sigmanoise^2)^2}\cdot \epsilon
\end{align*}where the first inequality is from the assumption that $\norm{\bphi(s,a)} \leq 1$ for all $(s,a)\in\cS\times\cA$ and the second inequality is by $\sigma^2(\cdot)\in[\eta_h+\sigmanoise^2, (H-h+1)^2+\sigmanoise^2]$. It then follows that 
\begin{align*}
        \norm{\bLambda_1^{-1} - \bLambda_2^{-1}} & = \norm{\bLambda_1^{-1} \rbr{\bLambda_1 - \bLambda_2}\bLambda_2^{-1}} 
        \\ & \leq \norm{\bLambda_1^{-1}} \cdot \norm{\bLambda_1 - \bLambda_2} \cdot \norm{\bLambda_2^{-1}}
        \\ & \leq \frac{2K  \sqrt{(H-h+1)^2 + \sigmanoise^2}\cdot \epsilon}{\lambda^2 (\eta_h+\sigmanoise^2)^2}, 
\end{align*} where in the last inequality we use $\norm{\bLambda_1^{-1}}, \ \norm{\bLambda_1^{-2}} \leq 1/\lambda$. 
\end{proof}

\begin{lemma}\label{lem:auxiliary inverse matrix distance square}
For any $h\in[H]$ and $L_1,L_2>0$, let $\sigma_1$, $\sigma_2 \in \cT_h(L_1,L_2)$ such that $\sup_{s,a}|\sigma_1^2(s,a) - \sigma_2^2(s,a)| \leq \epsilon$. Then it holds that
\begin{align*}
    \norm{\bLambda_1 - \bLambda_2} & \leq \frac{K}{ \left( \eta+\sigmanoise^2\right)^2} \cdot \epsilon,\ 
     \norm{\bLambda_1^{-1} - \bLambda_2^{-1}}  \leq \frac{K}{\lambda^2 \left(\eta+\sigmanoise^2\right)^2} \cdot \epsilon. 
\end{align*}
\end{lemma}
\begin{proof}[Proof of Lemma \ref{lem:auxiliary inverse matrix distance square}]
Note that 
\begin{align*}
        \norm{\bLambda_1 - \bLambda_2} & \leq \sum_{k=1}^K \norm{\bphi(s_{k,h},a_{k,h}) \bphi^\top(s_{k,h},a_{k,h})}\cdot \left|\frac{1}{\sigma_1^2(s_{k,h},a_{k,h})} - \frac{1}{\sigma_2^2(s_{k,h},a_{k,h})}\right|
        \\ & \leq \sum_{k=1}^K \left|\frac{\sigma_2^2(s_{k,h},a_{k,h}) - \sigma_1^2(s_{k,h},a_{k,h})}{\sigma_1^2(s_{k,h},a_{k,h})\cdot \sigma_2^2(s_{k,h},a_{k,h})} \right|,         
\end{align*}and the rest follows from the proof of Lemma \ref{lem:auxiliary inverse matrix distance}.
\end{proof}

\subsection{Bounds for self-normalized vector-valued martingales}\label{sec:auxiliary self normalized martingale}
Here we introduce some concentration inequalities that can be applied to bound the self-normalized martingales. 

\begin{theorem}[Hoeffding inequality for self-normalized martingales, \citealt{abbasi2011improved}]\label{thm:self normalized hoeffding} Let $\{\eta_t\}_{t=1}^\infty$ be a real-valued stochastic process. Let $\{\cF_t \}_{t=0}^\infty$ be a filtration, such that $\eta_t$ is $\cF_t$-measurable. Assume $\eta_t \mid \cF_{t-1}$ is zero-mean and $R$-subgaussian for some $R > 0$, i.e.,
\begin{align*}
    \forall \lambda \in \RR, \quad \EE\left[ e^{\lambda \eta_t \mid \cF_{t-1}} \right] \leq e^{\lambda^2 R^2 /2}.
\end{align*}
Let $\{\xb_t\}_{t=1}^\infty$ be an $\RR^d$-valued stochastic process where $\xb_t$ is $\cF_{t-1}$-measurable. 
Assume $\bLambda_0$ is a $d\times d$ positive definite matrix, and define $\bLambda_t = \bLambda_0 + \sum_{s=1}^t \xb_s \xb_s^\top$. 
Then, for any $\delta >0$, with probability at least $1-\delta$, for all $t > 0$, 
\begin{align*}
     \Bignorm{\sum_{s=1}^t \xb_s \eta_s }^2_{\bLambda_t^{-1}} \leq 2 R^2 \log \left( \frac{ \det(\bLambda_t)^{1/2} \det(\bLambda_0)^{-1/2} }{\delta}\right). 
\end{align*}
\end{theorem}

\begin{theorem}[Bernstein inequality for self-normalized martingales, \citealt{zhou2020nearly}]\label{thm:self normalized bernstein}
Let $\{\eta_t\}_{t=1}^\infty$ be a real-valued stochastic process. Let $\{\cF_t \}_{t=0}^\infty$ be a filtration, such that $\eta_t$ is $\cF_t$-measurable. Assume $\eta_t$ also satisfies 
\begin{align*}
    |\eta_t|\leq R, \ \EE[\eta_t \mid \cF_{t-1}]=0, \ \EE[\eta_t^2 \mid \cF_{t-1}] \leq \sigma^2.
\end{align*}Let $\{\xb_t\}_{t=1}^\infty$ be an $\RR^d$-valued stochastic process where $\xb_t$ is $\cF_{t-1}$-measurable and $\left\| \xb_t\right\|\leq L$. 
Let $\bLambda_t = \lambda \Ib_d + \sum_{s=1}^t \xb_s \xb_s^\top$. 
Then, for any $\delta >0$, with probability at least $1-\delta$, for all $t > 0$, 
\begin{align*}
     \Bignorm{\sum_{s=1}^t \xb_s \eta_s }_{\bLambda_t^{-1}} \leq 8\sigma\sqrt{d\log\left( 1+\frac{tL^2}{\lambda d} \right) \cdot \log\left( \frac{4t^2}{\delta} \right)} + 4R\log \left( \frac{4t^2}{\delta} \right) . 
\end{align*}
\end{theorem}

\subsection{Auxiliary Results for Self-normalized Martingales}
Assume the function $\sigma_1(\cdot,\cdot)$ is computed using the function $V_1(\cdot)$ in the same way $\hatsigma_h$ is computed using $\hatV_{h+1}^\pi$ as in Algorithm \ref{alg:weighted FQI}. In this way, we can view $\sigma_1$ as a function parameterized by $V_1$. And similar for $\sigma_2$ and $V_2$. 
\begin{lemma}\label{lem:function distance sigma}
Assume $V_1$ and $V_2 \in \cV_{h+1}(L)$ and satisfy $\sup_s |V_1(s)-V_2(s)| \leq \epsilon$. Then
\begin{align*}
    \sup_{s,a} |\sigma_1(s,a)-\sigma_2(s,a)| & \leq 2\sqrt{\frac{K(H-h+1)}{\lambda}}\cdot\sqrt{\epsilon},
    \\ \sup_{s,a} |\sigma_1^2(s,a)-\sigma_2^2(s,a)| & \leq \frac{4K(H-h+1)}{\lambda} \epsilon . 
\end{align*}
\end{lemma}
\begin{proof}[Proof of Lemma \ref{lem:function distance sigma}]
By the proof of Lemma \ref{lem:covering sigma}, we have 
\begin{align*}
    \sup_{s,a} |\sigma_1(s,a)-\sigma_2(s,a)| \leq \sup_{s,a} \sqrt{\left|\sigma_1^2(s,a) - \sigma_2^2(s,a) \right|} \leq  \sqrt{ \norm{\bbeta_1 - \bbeta_2} + 2(H-h+1)\cdot \norm{\btheta_1-\btheta_2} }.
\end{align*}Note that
\begin{align*}
        \left\| \btheta_1 - \btheta_2 \right\| &= \left\| (\hat\bSigma_h)^{-1}\sum_{k=1}^K \bphi(\check s_{k,h},\check a_{k,h}) (V_1-V_2)(\check s'_{k,h})  \right\| \leq \frac{K}{\lambda} \epsilon,
\end{align*}where we use $\left\| (\hat\bSigma_h)^{-1}\right\| \leq 1/\lambda$ and $\left\| \bphi(s,a)\right\|\leq 1$ for all $(s,a)$. Similarly, we can show 
\begin{align*}
        \left\| \bbeta_1 - \bbeta_2 \right\| & \leq \left\| (\hat\bSigma_h)^{-1} \sum_{k=1}^K \bphi(\check{s}_{k,h}, \check{a}_{k,h}) (V_1(\check{s}_{k,h}')^2 - V_2(\check{s}_{k,h}')^2)\right\|
        \\ & \leq \left\| (\hat\bSigma_h)^{-1} \right\| \cdot K \cdot 2(H-h+1)\epsilon \leq \frac{2K(H-h+1)}{\lambda}\epsilon.
\end{align*}Altogether, we have
\begin{align*}
        \sup_{s,a} |\sigma_1(s,a)-\sigma_2(s,a)| \leq 2\sqrt{\frac{K(H-h+1)}{\lambda}}\cdot\sqrt{\epsilon} ,
\end{align*}and 
\begin{align*}
        \sup_{s,a} |\sigma_1^2(s,a)-\sigma_2^2(s,a)| & \leq \frac{4K(H-h+1)}{\lambda} \epsilon . 
\end{align*}
\end{proof}

\subsection{Covering numbers of the function classes}\label{sec:covering number}
Here we compute the covering numbers of the function classes $\cV_h$ and $\cT_h$.

\begin{lemma}[Covering number of the Euclidean Ball]\label{lem:covering Euclidean ball} For any $\epsilon >0$, the $\epsilon$-covering number of the ball of radius $r$ under the Euclidean norm satisfies $\cN_\epsilon \leq (1+2r/ \epsilon)^d$. 
\end{lemma}
A proof of this classical result can be found, for example, in the work by \citet{vershynin2010introduction}. Now we give the covering number of the function class $\cV_h(L)$ for all $h\in[H]$ and $L>0$.

\begin{lemma}\label{lem:covering V}
For any $h\in[H]$ and any $L>0$, let $\cV_h(L)$ be as defined in \eqref{def: v function class}. Let $\cN_\epsilon$ denote the $\epsilon$-covering number of $\cV_h(L)$ with respect to the distance $\textnormal{dist}(V_1, V_2) = \sup_{s}|V_1(s)-V_2(s)|$. Then under Assumption \ref{assump:linear_MDP}, it holds that 
\begin{align*}
    \cN_\epsilon\leq \rbr{1+\frac{2L}{\epsilon}}^d. 
\end{align*}
\end{lemma}

\begin{proof}[Proof of Lemma \ref{lem:covering V}]
For any $V_1$, $V_2 \in \cV_h(L)$ parametrized by $\wb_1$ and $\wb_2$ respectively, we have
\begin{align*}
    \textnormal{dist}(V_1,V_2)  = \sup_s\left|\langle \bphi_h^\pi(s),\wb_1-\wb_2\rangle\right|
    \leq \norm{\wb_1-\wb_2}_2 \cdot \sup_s \left\|\bphi_h^\pi(s)\right\|_2
    \leq \norm{\wb_1-\wb_2}_2,
\end{align*}where the first inequality is by Cauchy-Schwarz inequality and the second inequality uses the assumption that $\norm{\bphi(s,a)}\leq 1$. 

Let $\cC_{\wb}(\epsilon)$ be an $\epsilon-$cover of the Euclidean ball $\{\wb\in\RR^d|\ \|\wb\|_2\leq L\}$. Then for any $V_1\in\cV_h(L)$, there exists a $V_2\in\cV_h(L)$ parametrized by $\wb_2\in\cC_{\wb}(\epsilon)$ such that $\text{dist}(V_1,V_2)\leq \epsilon$. Then we see that
\begin{align*}
    \cN_\epsilon \leq |\cC_{\wb}(\epsilon)|\leq \left(1+\frac{2L}{\epsilon}\right)^d,
\end{align*}
where the second inequality follows from Lemma \ref{lem:covering Euclidean ball}.
\end{proof}

\begin{lemma}\label{lem:covering sigma}
For any $h\in[H]$ and $L_1,L_2>0$, let $\cT_h(L_1,L_2)$ be as defined in \eqref{def: sigma function class}. Let $\cN_\epsilon$ denote the $\epsilon$-covering number of $\cT_h(L_1,L_2)$ with respect to the distance $\textnormal{dist}(\sigma_1, \sigma_2) = \sup_{s,a}|\sigma_1(s,a)-\sigma_2(s,a)|$. Then under Assumption \ref{assump:linear_MDP}, it holds that
\begin{align*}
    \cN_\epsilon \leq \left(1+\frac{4L_1}{\epsilon^2}\right)^d \cdot \left(1+\frac{8(H-h+1)L_2}{\epsilon^2}\right)^d . 
\end{align*}
\end{lemma}

\begin{proof}[Proof of Lemma \ref{lem:covering sigma}]
For any $\sigma_1$, $\sigma_2\in \cT$ which are parameterized by $(\bbeta_1,\btheta_1)$ and $(\bbeta_2,\btheta_2)$ respectively, we have
\begin{align*}
        &\textnormal{dist}(\sigma_1, \sigma_2)\\
         &= \sup_{s,a} \left| \sigma_1(s,a) - \sigma_2(s,a) \right|
        \\  &\leq \sup_{s,a} \sqrt{\left|\sigma_1^2(s,a) - \sigma_2^2(s,a) \right|}
        \\ &\leq \sup_{s,a} \sqrt{ \Big|\langle \bphi(s,a), \bbeta_1\rangle- \langle \bphi(s,a), \bbeta_2\rangle \Big| + \left| [\langle \bphi(s,a), \btheta_1 \rangle_{[0,H-h+1]}]^2 - [\langle \bphi(s,a), \btheta_2 \rangle_{[0,H-h+1]}]^2 \right| }
        \\ &\leq \sup_{s,a} \sqrt{ \Big|\langle \bphi(s,a), \bbeta_1\rangle - \langle \bphi(s,a), \bbeta_2\rangle \Big| + 2 (H-h+1)\cdot \left| \langle \bphi(s,a), \btheta_1 \rangle - \langle \bphi(s,a), \btheta_2 \rangle\right| }
        \\ &\leq \sqrt{ \norm{\bbeta_1 - \bbeta_2} + 2(H-h+1)\cdot \norm{\btheta_1-\btheta_2} }.
\end{align*}where the first inequality uses the fact that $|a-b|\leq\sqrt{|a^2-b^2|}$ for any $a,b\geq 0$, the second and the third inequalities follows from the fact that $\max\{\eta_h,\ \cdot\}$ and the clipping $\{\cdot\}_{[0,(H-h+1)^2]}$,  $\{\cdot\}_{[0,H-h+1]}$ are all contraction maps, and the last inequality is by Cauchy-Schwarz inequality and the assumption that $\norm{\bphi(s,a)}\leq 1$.

In order to have $\textnormal{dist}(\sigma_1,\sigma_2)\leq \epsilon$, it suffices to have $\norm{\bbeta_1-\bbeta_2}\leq \epsilon^2/2$ and $2(H-h+1)\norm{\btheta_1 - \btheta_2}\leq \epsilon^2/2$. By Lemma \ref{lem:covering Euclidean ball}, in order to $\epsilon^2/2$-cover $\{\bbeta: \norm{\bbeta}<L_1\}$ and $\epsilon^2/(4(H-h+1))$-cover $\{\btheta: \ \norm{\btheta}\leq L_2 \}$ we need 
\begin{align*}
    {\cN_\beta} \leq \left( 1 + \frac{4 L_1}{\epsilon^2}\right)^d \ ,
     {\cN_\theta}  \leq \left( 1 + \frac{8(H-h+1) L_2}{\epsilon^2} \right)^d \ .
\end{align*}Altogether, to $\epsilon$-cover $\cT$, we have 
\begin{align*}
    \cN_\epsilon \leq {\cN_\beta} \cdot {\cN_\theta} \leq \left( 1 + \frac{4L_1}{\epsilon^2}\right)^d \cdot \left( 1 + \frac{8(H-h+1) L_2}{\epsilon^2} \right)^d. 
\end{align*}
\end{proof}

\subsection{Bounds for the Regression Estimators}
\begin{lemma}\label{lem:bound beta w}
Assume $ \sup_s \left|\hatV_{h+1}^\pi(s) \right|\leq B$ for some $B \geq 0$. Then $\hattheta_h$, $\hatbeta_h$ and $\hatw_h^\pi$ in Algorithm \ref{alg:weighted FQI} satisfy the following:
\begin{align*}
        \norm{\hattheta_h}  \leq B \sqrt{\frac{Kd}{\lambda}},
        \quad \norm{\hatbeta_h}  \leq B^2 \sqrt{\frac{Kd}{\lambda}},
        \quad \norm{\hatw_h^\pi}  \leq \frac{B+1}{\sqrt{\eta+\sigmanoise^2}} \sqrt{\frac{Kd}{\lambda}} . 
\end{align*}
\end{lemma}

\begin{proof}[Proof of Lemma \ref{lem:bound beta w}]
For any vector $\vb \in \RR^d$, we have
\begin{align*}
        |\vb^\top \hattheta_h| &=  \left| \vb^\top (\hat{\bSigma}_h)^{-1} \sum_{k=1}^K \bphi(\check{s}_{k,h}, \check{a}_{k,h}) \hatV_{h+1}^\pi(\check{s}_{k,h}') \right|
        \\ &\leq \sum_{k=1}^K |\vb^\top (\hat{\bSigma}_h)^{-1} \bphi(\check{s}_{k,h}, \check{a}_{k,h})| \cdot \sup_s|\hatV_{h+1}^\pi(s)|
        \\ & \leq B \cdot \sqrt{\left[\sum_{k=1}^K \vb^\top (\hat{\bSigma}_h)^{-1}\vb \right]\cdot \left[\sum_{k=1}^K \bphi(\check{s}_{k,h}, \check{a}_{k,h})^\top (\hat{\bSigma}_h)^{-1} \bphi(\check{s}_{k,h}, \check{a}_{k,h}) \right] }
        \\ & \leq B \norm{\vb}_2\sqrt{\frac{K}{\lambda}}\cdot\sqrt{d},
\end{align*}where the second inequality is by Cauchy-Schwarz inequality, and the last inequality uses $\norm{(\hat{\bSigma}_h)^{-1}} \leq 1/\lambda$ and Lemma \ref{lem:auxiliary inverse middle}. It follows that $\norm{\hattheta_h}  \leq B \sqrt{\frac{Kd}{\lambda}}$. Similarly, we have $\norm{\hatbeta_h} \leq B^2 \sqrt{\frac{Kd}{\lambda}}$ since $ \sup_s |\hatV_{h+1}^\pi(s) |^2\leq B^2$. To bound $\norm{\hat\wb_h^\pi}$, note that 
\begin{align*}
        |\vb^\top \hatw_h^\pi| & = \left| \vb^\top \hat\bLambda_h^{-1} \sum_{k=1}^K \bphi(s_{k,h},a_{k,h}) Y_{k,h} / \hat\sigma_{k,h}^2 \right| 
        \\ &\leq \frac{B+1}{\sqrt{\eta_h+\sigmanoise^2}} \cdot \sum_{k=1}^K \left| \vb^\top \hat\bLambda_h^{-1}\frac{\bphi(s_{k,h},a_{k,h})}{\hat\sigma_{k,h}} \right| 
        \\ & \leq \frac{B+1}{\sqrt{\eta_h+\sigmanoise^2}} \cdot \sqrt{\left[\sum_{k=1}^K \vb^\top (\hat{\bLambda}_h)^{-1}\vb \right]\cdot \left[\sum_{k=1}^K \frac{\bphi(s_{k,h},a_{k,h})}{\hat\sigma_{k,h}}^\top (\hat{\bLambda}_h)^{-1} \frac{\bphi(s_{k,h},a_{k,h})}{\hat\sigma_{k,h}} \right] }
        \\ & \leq \frac{B+1}{\sqrt{\eta_h+\sigmanoise^2}} \cdot\norm{\vb}_2\sqrt{\frac{K}{\lambda}}\cdot\sqrt{d}
        \\ & = \left(\frac{B+1}{\sqrt{\eta_h+\sigmanoise^2}} \sqrt{\frac{Kd}{\lambda}} \right)\norm{\vb}_2,
\end{align*}where the first inequality comes from 
\begin{align*}
    \left| \frac{Y_{k,h}}{\hat\sigma_{k,h}}\right| & =  \left| \frac{r_{k,h}+\hatV_{h+1}^\pi(s'_{k,h})}{\hat\sigma_{k,h}}\right| \leq (B+1)\cdot\frac{1}{\sqrt{\eta_h + \sigmanoise^2}}, 
\end{align*}and note that by assumption $|r_{k,h}|\leq1$ $a.s.$, and by the clipping in the algorithm, $\hat\sigma_{k,h} \geq \sqrt{\eta_h + \sigmanoise^2}$. 

\end{proof}

\end{document}